\documentclass{article}

\usepackage{etoolbox}
\newtoggle{longversion}
\settoggle{longversion}{true} 
\iftoggle{longversion}{\newcommand{\extversion}{the Appendix}}
{\newcommand{\extversion}{the extended version of this paper \cite{chen2018understanding-arxiv}}}

\iftoggle{longversion}{
\usepackage[preprint, nonatbib]{neurips_2018}
}
{
\usepackage[final]{neurips_2018}
}

\usepackage{microtype}
\usepackage{graphicx}
\usepackage{amsmath,amssymb}
\usepackage{mathabx} 
\usepackage{booktabs} 
\usepackage{xspace}
\usepackage{enumerate}

\usepackage{hyperref}
\usepackage{xcolor}

\usepackage{algorithm}
\usepackage[noend]{algorithmic}

\newcommand{\BlackBox}{\rule{1.5ex}{1.5ex}}  
\newenvironment{proof}{\par\noindent{\bf Proof\ }}{\hfill\BlackBox\\[2mm]}
\newtheorem{example}{Example} 
\newtheorem{theorem}{Theorem}
\newtheorem{lemma}[theorem]{Lemma} 
\newtheorem{proposition}[theorem]{Proposition}

\newcommand{\yuxin}[1]{\ifthenelse{\boolean{showcomments}}{\textcolor{blue}{YC: #1}}{}}
\newcommand{\adish}[1]{\ifthenelse{\boolean{showcomments}}{\textcolor{red}{AS: #1}}{}}
\newcommand{\oisin}[1]{\ifthenelse{\boolean{showcomments}}{\textcolor{green}{OMA: #1}}{}}

\newcommand{\commentout}[1]{}

\newcommand{\denselist}{
\itemsep -4pt\topsep-8pt\partopsep-8pt\itemindent-10pt
}

\newcommand{\Hypotheses}{\mathcal{H}}
\newcommand{\hypotheses}{H}

\newcommand{\hstar}{\hypothesis^*}
\newcommand{\hinit}{{\hypothesis_0}}

\newcommand{\oracle}{\textsf{Oracle}\xspace}
\newcommand{\teachalg}{\textsf{Teacher}\xspace}

\newcommand{\Examples}{\mathcal{Z}}
\newcommand{\Instances}{\mathcal{X}}
\newcommand{\Clabels}{\mathcal{Y}}
\newcommand{\commentfmt}[1]{\textit{\textcolor{gray}{#1}}}
\newcommand{\examples}{Z}

\renewcommand{\example}{{z}}
\newcommand{\instance}{{x}}
\newcommand{\clabel}{{y}}

\newcommand{\TD}{\textsf{TD}}

\newcommand{\dist}{\ensuremath{\text{dist}}}

\newcommand{\edgedist}{\ensuremath{\text{dist}_e}}

\newcommand{\prefset}{\mathcal{S}}
\newcommand{\ordering}{\sigma}
\newcommand{\uordering}{\sigma_{\textsf{unif}}}

\newcommand{\orderingof}[2]{\ordering({#1};{#2})}

\newcommand{\randomstate}{\phi}

\newcommand{\adarteacher}{\textsf{Ada-R-Teacher}\xspace}

\newcommand{\adar}{\textsf{Ada-R}\xspace}

\newcommand{\nonadar}{\textsf{Non-R}\xspace}
\newcommand{\nonadal}{\textsf{Non-L}\xspace}

\newcommand{\setcover}{\textsf{SC}\xspace}
\newcommand{\random}{\textsf{Rand}\xspace}

\newcommand{\tworec}{\textsc{2-Rec}\xspace}
\newcommand{\lattice}{\textsc{Lattice}\xspace}

\newcommand{\optcost}{\ensuremath{\text{cost}^{*}}\xspace}

\newcommand{\greedycost}{\ensuremath{\text{cost}^{\text{g}}}\xspace}

\newcommand{\futurecostopt}{\ensuremath{D^*}\xspace}
\newcommand{\futurecostapprox}{\ensuremath{\tilde{D}\xspace}}

\newcommand{\lbl}{y}

\DeclareMathOperator{\cost}{cost}

\newcommand{\paren} [1] {\ensuremath{ \left( {#1} \right) }}

\def \argmin {\mathop{\rm arg\,min}}

\newcommand{\reals}{\ensuremath{\mathbb{R}}}

\newcommand{\hypothesis}[0]{\ensuremath{h}}

\newcommand{\bigO}[1]{\ensuremath{O\paren{#1}}}

\newcommand{\bigOmega}[1]{\ensuremath{\Omega\paren{#1}}}

\newtheorem{exampledef}{Example}
\numberwithin{equation}{section}

\newcommand{\figref}[1]{Fig.~\ref{#1}}

\newcommand{\secref}[1]{\S\ref{#1}}

\newcommand{\thmref}[1]{Theorem~\ref{#1}}

\newcommand{\lemref}[1]{Lemma~\ref{#1}}
\newcommand{\algref}[1]{Algorithm~\ref{#1}}

\usepackage{subcaption}
\usepackage{wrapfig}

\newcommand{\given}{\mid}

\usepackage[utf8]{inputenc} 
\usepackage[T1]{fontenc}    
\usepackage{hyperref}       
\usepackage{url}            
\usepackage{booktabs}       
\usepackage{amsfonts}       
\usepackage{nicefrac}       
\usepackage{microtype}      

\title{Understanding the Role of Adaptivity in Machine Teaching: The Case of Version Space Learners}

\author{
  Yuxin Chen\textsuperscript{\textdagger} \quad \textbf{Adish Singla\textsuperscript{\textdaggerdbl}} \quad Oisin Mac Aodha\textsuperscript{\textdagger}\\
  \textbf{Pietro Perona\textsuperscript{\textdagger}} \quad \textbf{Yisong Yue\textsuperscript{\textdagger}}\\
  \\
  \textsuperscript{\textdagger}Caltech, \texttt{\{chenyux, macaodha, perona, yyue\}@caltech.edu}, \\
  \textsuperscript{\textdaggerdbl}MPI-SWS, \texttt{adishs@mpi-sws.org}\\
}

\begin{document}

\maketitle


\begin{abstract}
\looseness -1 In real-world applications of education, an effective teacher adaptively chooses the next example to teach based on the learner's current state. However, most existing work in \emph{algorithmic machine teaching} focuses on the batch setting, where adaptivity plays no role. In this paper, we study the case of teaching consistent, version space learners in an interactive setting.
At any time step, the teacher provides an example, the learner performs an update, and the teacher observes the learner's new state. 
We highlight that adaptivity does not speed up the teaching process when considering existing models of version space learners, such as the ``worst-case'' model (the learner picks the next hypothesis randomly from the version space) and the ``preference-based'' model (the learner picks hypothesis according to some global preference). 
Inspired by human teaching, we propose a new model where the learner picks hypotheses according to some local preference defined by the current hypothesis. We show that our model exhibits several desirable properties, e.g., adaptivity plays a key role, and the learner's transitions over hypotheses are smooth/interpretable. We develop adaptive teaching algorithms, and demonstrate our results via simulation and user studies.

\end{abstract}



\section{Introduction}\label{sec:intro}
Algorithmic machine teaching studies the interaction between a teacher and a student/learner where the teacher's objective is to find an optimal training sequence to steer the learner towards a desired goal \cite{DBLP:journals/corr/ZhuSingla18}. Recently, there has been a surge of interest in machine teaching as several different communities have found connections to this problem setting: (i) machine teaching provides a rigorous formalism for a number of real-world applications including personalized educational systems~\cite{zhu2015machine}, adversarial attacks~\cite{DBLP:conf/aaai/MeiZ15}, imitation learning~\cite{cakmak2012algorithmic,haug2018teachingrisk}, and program synthesis~\cite{DBLP:journals/acta/JhaS17}; (ii) the complexity of teaching (``Teaching-dimension'') has strong connections with the information complexity of learning (``VC-dimension'') \cite{doliwa2014recursive}; and (iii) the optimal teaching sequence has properties captured by new models of interactive machine learning, such as curriculum learning \cite{DBLP:conf/icml/BengioLCW09} and self-paced learning \cite{DBLP:journals/isci/MengZJ17}.

In the above-mentioned applications, adaptivity clearly plays an important role. For instance, in automated tutoring, adaptivity enables personalization of the content based on the student's current knowledge~\cite{DBLP:conf/icassp/TekinBS15,weld2012personalized,hunziker2018teachingmultiple}. In this paper, we explore the \emph{adaptivity gain} in algorithmic machine teaching, i.e., how much speedup a teacher can achieve via adaptively selecting the next example based on the learner's current state? While this question has been well-studied in the context of active learning and sequential decision making \cite{DBLP:conf/ciac/HellersteinKL15}, the role of adaptivity is much less understood in algorithmic machine teaching. 
A deeper understanding would, in turn, enable us to develop better teaching algorithms and more realistic learner models to exploit the adaptivity gain.

We consider the well-studied case of teaching a consistent, version space learner.
A learner in this model class maintains a version space (i.e., a subset of hypotheses that are consistent with the examples received from a teacher) and outputs a hypothesis from this version space. Here, a hypothesis can be viewed as a function that assigns a label to any unlabeled example. Existing work has studied this class of learner model to establish theoretical connections between the information complexity of teaching vs. learning \cite{goldman1995complexity,zilles2011models,gao2017preference}.  
Our main objective is to understand, when and by how much, a teacher can benefit by adapting the next example based on the learner's current hypothesis. We compare two types of teachers: (i) an \emph{adaptive teacher} that observes the learner's hypothesis at every time step, and (ii) a  \emph{non-adaptive teacher} that only knows the initial hypothesis of the learner and does not receive any feedback during teaching. The non-adaptive teacher operates in a batch setting where the complete sequence of examples can be constructed before teaching begins.
\begin{wrapfigure}{R}{0.35\textwidth}
  \centering
  \includegraphics[trim={10pt 70pt 10pt 70pt},clip,width=0.35\textwidth]{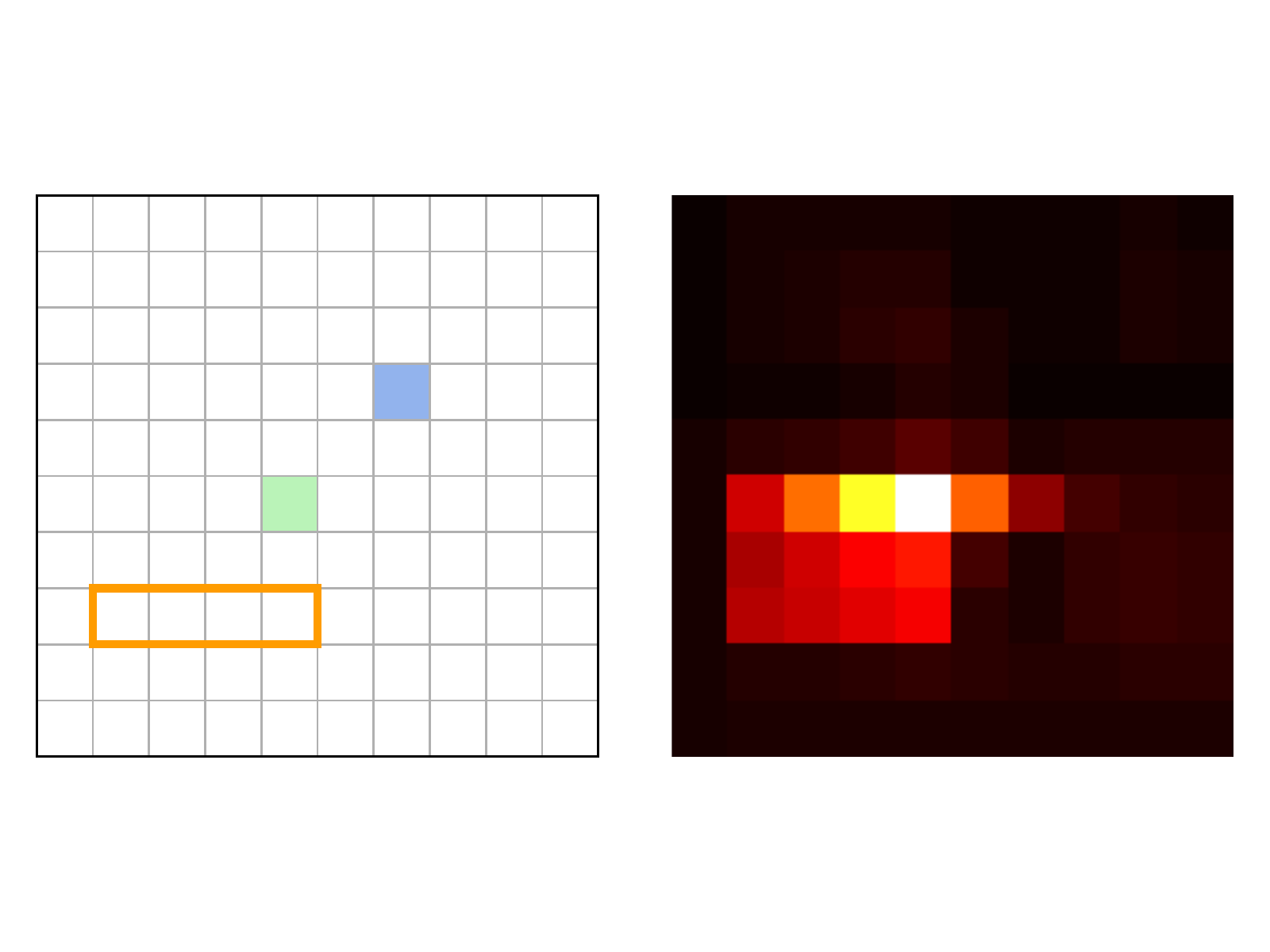}
  \caption{Local update preference. Users were asked to update the position of the orange rectangle so that green cells were inside and blue ones outside. The heatmap on the right displays the updated positions.}
  \label{fig:intro.smoothness}
\end{wrapfigure}
Inspired by real-world teaching scenarios and as a generalization of the global ``preference-based'' model  \cite{gao2017preference}, we propose a new model where the learner's choice of next hypothesis $\hypothesis'$  depends on some \emph{local} preferences defined by the current hypothesis $\hypothesis$. 
For instance, the local preference could encode that the learner prefers to make smooth transitions by picking a consistent hypothesis  $\hypothesis'$ which is ``close'' to $\hypothesis$.  
Local preferences, as seen in \figref{fig:intro.smoothness}, are an important aspect of many machine learning algorithms (e.g., incremental or online learning algorithms \cite{ross2008incremental,shalev2012online}) in order to increase robustness and reliability. 
We present results in the context of two different hypotheses classes, and show through simulation and user studies that adaptivity can play a crucial role when teaching learners with local preferences.
\section{Related Work} \label{sec:related}
\noindent\textbf{Models of version space learners}
Within the model class of version space learners, there are different variants of learner models depending upon their anticipated behavior, and these models lead to different notions of teaching complexity. For instance, (i) the ``worst-case'' model \cite{goldman1995complexity} essentially assumes nothing and the learner's behavior is completely unpredictable, (ii) the ``cooperative''  model \cite{zilles2011models} assumes a smart learner who anticipates that she is being taught, and (iii) the ``preference-based'' model \cite{gao2017preference} assumes that she has a global preference over the hypotheses. Recently, some teaching complexity results have been extended beyond version space learners, such as Bayesian learners \cite{zhu2013machine}, probabilistic/randomized learners \cite{singla2014near,balbach2011teaching}, learners implementing an optimization algorithm \cite{liu2016teaching}, and for iterative learning algorithms based on gradient updates \cite{liu2017iterative}. Here, we focus on the case of version space learners, leaving the extension to other types of learners for future work.


%
%
%
%
%

\noindent\textbf{Batch vs. sequential teaching}
Most existing work on algorithmic machine teaching has focused on the batch setting, where the teacher constructs a set of examples and provides it to the learner at the beginning of teaching \cite{goldman1995complexity,zilles2011models,gao2017preference,chen18explain}. 
There has been some work on sequential teaching models that are more suitable for understanding the role of adaptivity.  Recently, \cite{liu2017iterative} studied the problem of iteratively teaching a gradient learner by providing a sequence of carefully constructed examples. However, since the learner's update rule is completely deterministic, a non-adaptive teacher with knowledge of the learner's initial hypothesis $h^0$ would behave exactly the same as an adaptive teacher (i.e., the adaptivity gain is zero). \cite{balbach2011teaching} studied randomized version-space learners with limited memory, and demonstrated the power of adaptivity for a specific class of hypotheses. Sequential teaching has also been studied in the context of crowdsourcing applications by \cite{johns2015becoming} and \cite{singla2013actively}, empirically demonstrating the improved performance of adaptive vs. non-adaptive teachers. However, these approaches do not provide any theoretical understanding of the adaptivity gain as done in our work.
%

\noindent\textbf{Incremental learning and teaching} Our learner model with local preferences is quite natural in real-world applications.  A large class of iterative machine learning algorithms are based on the idea of incremental updates which in turn is important for the robustness and generalization of learning \cite{ross2008incremental,shalev2012online}. From the perspective of a human learner, the notion of incremental learning aligns well with the concept of the ``Zone of Proximal Development (ZPD)'' in the educational research and psychology literature \cite{vygotsky1987zone}. The ZPD suggests that teaching is most effective when focusing on a task \emph{slightly} beyond the current abilities of the student as the human learning process is inherently incremental. 
Different variants of learner model studied in the cognitive science literature \cite{levine1975cognitive,bonawitz2014win,rafferty2016faster} have an aspect of incremental learning. For instance, the ``win stay lose shift’’ model \cite{bonawitz2014win} is a special case of the local preference model that we propose in our work.
Based on the idea of incremental learning, \cite{balbach2005teaching} studied the case of teaching a variant of the version space learner when restricted to incremental learning and is closest to our model with local preferences. 
However, there are two key differences in their model compared to ours: (i) they allow learners to select inconsistent hypotheses (i.e., outside the version space), (ii) the restricted movement in their model is a hard constraint which in turns means that teaching is not always feasible -- given a problem instance it is NP-Hard to decide if a given target hypothesis is teachable or not.

\section{The Teaching Model}\label{sec:model}
We now describe the teaching domain, present a generic model of the learner and the teacher, and then state the teacher's objective.

\subsection{The Teaching Domain}
Let $\Instances$ denote a ground set of unlabeled examples, and $\Clabels$ denote the set of possible labels that could be assigned to elements of $\Instances$. We denote by $\Hypotheses$ a finite class of hypotheses, each element $\hypothesis\in \Hypotheses$ is a function $\hypothesis: \Instances \rightarrow \Clabels$. In this paper, we will only consider boolean functions and hence $\Clabels = \{0,1\}$. In our model, $\Instances$, $\Hypotheses$, and $\Clabels$ are known to both the teacher and the learner.
There is a \emph{target hypothesis} $\hstar\in \Hypotheses$ that is known to the teacher, but not the learner.  Let $\Examples \subseteq \Instances \times \Clabels$ be the ground set of labeled examples. Each element $\example = (\instance_\example,\clabel_\example) \in \Examples$ represents a labeled example where the label is given by the target hypothesis $\hstar$, i.e., $\clabel_\example = \hstar(\instance_\example)$.
Here, we define the notion of \emph{version space} needed to formalize our model of the learner. Given a set of labeled examples $\examples \subseteq \Examples$, the version space induced by $\examples$ is the subset of hypotheses $\Hypotheses(\examples) \in \Hypotheses$ that are consistent with the labels of all the examples, i.e., $\Hypotheses(\examples):= \{\hypothesis \in \Hypotheses: \forall \example = (\instance_\example, \clabel_\example) \in \examples, h(\instance_\example) = \clabel_\example\}$.

\subsection{Model of the Learner}\label{sec:model:learner}
We now introduce a generic model of the learner by formalizing our assumptions about how she adapts her hypothesis based on the labeled examples received from the teacher. A key ingredient of this model is the \emph{preference function} of the learner over the hypotheses as described below. As we show in the next section, by providing specific instances of this preference function, our generic model reduces to existing models of version space learners, such as the ``worst-case'' model \cite{goldman1995complexity} and the global ``preference-based" model \cite{gao2017preference}.


Intuitively, the preference function encodes the learner's transition preferences. 
Consider that the learner's current hypothesis is $h$, and there are two hypotheses $h'$, $h''$ that they could possibly pick as the next hypothesis.
We want to encode whether the learner has any preference in choosing $h'$ or  $h''$. Formally, we define the preference function as $\sigma: \Hypotheses \times \Hypotheses \rightarrow \reals_+$. Given current hypothesis $h$ and any two hypothesis $\hypothesis', \hypothesis''$, we say that $\hypothesis'$ is preferred to $\hypothesis''$ from $\hypothesis$, iff $\sigma(\hypothesis' ; \hypothesis) < \sigma(\hypothesis''; \hypothesis)$. If $\sigma(\hypothesis'; \hypothesis) = \sigma(\hypothesis''; \hypothesis)$, then the learner could pick either one of these two.


The learner starts with an initial hypothesis $\hinit\in \Hypotheses$ before receiving any labeled examples from the teacher. 
Then, the interaction between the teacher and the learner proceeds in discrete time steps. At any time step $t$, let us denote the labeled examples received by the learner up to (but not including) time step $t$ via a set $\examples_{t}$, the learner's version space as $\Hypotheses_t=\Hypotheses(\examples_{t})$,  and the current hypothesis as $\hypothesis_t$.
At time step $t$, we model the learning dynamics as follows:
(i) the learner receives a new example $\example_t$; 
and (ii) the learner updates the version space $\Hypotheses_{t+1}$, and picks the next hypothesis based on the current hypothesis $\hypothesis_t$, version space $\Hypotheses_{t+1}$, and the preference function $\ordering$:
  \begin{align}
    \hypothesis_{t+1} \in \{\hypothesis\in \Hypotheses_{t+1} : \orderingof{\hypothesis}{\hypothesis_t} = \min_{\hypothesis'\in \Hypotheses_{t+1}} \orderingof{\hypothesis'}{\hypothesis_t}\}.
    \label{eq.learners-jump}
  \end{align}



\subsection{Model of the Teacher and the Objective}
\looseness -1 The teacher's goal is to steer the learner towards the target hypothesis $\hstar$ by providing a sequence of labeled examples. At time step $t$, the teacher selects a labeled example $\example_t \in \Examples$ and the learner transitions from the current $\hypothesis_t$ to the next hypothesis $\hypothesis_{t+1}$ as per the model described above. Teaching finishes at time step $t$ if the learner's hypothesis $\hypothesis_{t}=\hypothesis^*$. Our objective is to design teaching algorithms that can achieve this goal in a minimal number of time steps. We study the \emph{worst-case} number of steps needed as is common when measuring the information complexity of teaching~\cite{goldman1995complexity,zilles2011models,gao2017preference}.

We assume that the teacher knows the learner's initial hypothesis $\hypothesis_0$ as well as the preference function $\orderingof{\cdot}{\cdot}$. 
In order to quantify the gain from adaptivity, we compare two types of teachers: (i) an \emph{adaptive teacher} who observes the learner's hypothesis $\hypothesis_t$ before providing the next labeled example $\example_{t}$ at any time step $t$; and (ii) a \emph{non-adaptive teacher} who only knows the initial hypothesis of the learner and does not receive any feedback from the learner during the teaching process.
Given these two types of teachers, we want to measure the \emph{adaptivity gain} by quantifying the difference in teaching complexity of the \emph{optimal adaptive} teacher compared to the \emph{optimal non-adaptive} teacher. 
\section{The Role of Adaptivity}\label{sec:adaptivity}
In this section, we study different variants of the learner's preference function, and formally state the adaptivity gain with two concrete problem instances.

\subsection{State-independent Preferences}
We first consider a class of preference models where the learner's preference about the next hypothesis does not depend on her current hypothesis. 
The simplest state-independent preference is captured by the ``worst-case'' model \cite{goldman1995complexity}, where the learner's preference over all hypotheses is uniform, i.e., $\forall \hypothesis, \hypothesis'$, $\orderingof{\hypothesis'}{\hypothesis} = c$, where $c$ is some constant.

\looseness -1 A more generic state-independent preference model is captured by non-uniform, global preferences. 
More concretely, for any $\hypothesis' \in \Hypotheses$, we have $\sigma(\hypothesis'; \hypothesis) = c_{\hypothesis'} \ \forall \hypothesis \in \Hypotheses$, a constant dependent only on $\hypothesis'$. 
This is similar to the notion of the global ``preference-based" version space learner introduced by \cite{gao2017preference}.


\begin{proposition}\label{prop:ada-state-ind}
  For the state-independent preference, adaptivity plays no role, i.e.,  the sample complexities of the optimal adaptive teacher and the optimal non-adaptive teacher are the same.
\end{proposition}
In fact, for the uniform preference model, the teaching complexity of the adaptive teacher is the same as the \emph{teaching dimension} of the hypothesis class with respect to teaching $\hypothesis^*$, given by
\begin{align}\label{eq:td}
  \TD(\hypothesis^*, \Hypotheses) := \min_{\examples} |\examples|, \text{~s.t.~} \Hypotheses(\examples) = \{\hypothesis^*\}.
\end{align}
For the global preference model, similar to the notion of \emph{preference-based teaching dimension} \cite{gao2017preference}, the teaching complexity of the adaptive teacher is given by
\begin{align}\label{eq:pbtd}
                     \min_{\examples} |\examples|, \text{~s.t.~} \forall \hypothesis \in \Hypotheses(\examples) \setminus \{\hypothesis^*\}, \orderingof{\hypothesis}{\cdot} > \orderingof{\hstar}{\cdot}.
\end{align}

\subsection{State-dependent Preferences}\label{sec:statepref}
In real-world teaching scenarios, human learners incrementally build up their knowledge of the world, and their preference of the next hypothesis naturally depends on their current state. To better understand the behavior of an adaptive teacher under a state-dependent preference model, we investigate the following two concrete examples:
\begin{exampledef}[\tworec]\label{example:rect}
  $\Hypotheses$ consists of up to two disjoint rectangles\footnote{For simplicity of discussion, we assume that for the \tworec hypothesis that contains two rectangles, the edges of the two rectangles do not overlap.} on a grid and $\Instances$ represents the grid cells (cf. \figref{fig:intro.smoothness} and \figref{fig:exp:recexample}). Consider an example $\example = (\instance_\example,\clabel_\example) \in \Examples$: $\clabel_\example=1$  (positive) if the grid cell $\instance_\example$ lies inside the target hypothesis, and $0$ (negative) elsewhere.
\end{exampledef}
The \tworec hypothesis class consists of two subclasses, namely $\Hypotheses^1$: all hypotheses with one rectangle, and $\Hypotheses^2$: those with exactly two (disjoint) rectangles. The \tworec class is inspired by teaching a union of disjoint objects. Here, objects correspond to rectangles and any $\hypothesis \in \Hypotheses$ represents one or two rectangles. Furthermore,  each hypothesis $\hypothesis$ is associated with a complexity measure given by the number of objects in the hypothesis.  \cite{pmlr-v76-gao17a} recently studied the problem of teaching a union of disjoint geometric objects, and \cite{balbach2008measuring} studied the problem of teaching a union of monomials. Their results show that, in general, teaching a target hypothesis of lower complexity from higher complexity hypotheses is the most challenging task.

For the \tworec class, we assume the following local preferences: (i) in general, the learner prefers to transition to a hypothesis with the same complexity as the current one (i.e., $\Hypotheses^1 \rightarrow \Hypotheses^1$ or $\Hypotheses^2 \rightarrow \Hypotheses^2$), (ii) when transitioning within the same subclass, the learner prefers small edits, e.g., by moving the smallest number of edges possible when changing their hypothesis, and (iii) the learner could switch to a subclass of lower complexity (i.e., $\Hypotheses^2 \rightarrow \Hypotheses^1$) in specific cases. 
We provide a detailed description of the preference function in \extversion.

\begin{exampledef}[\lattice]\label{example:lattice}
  $\Hypotheses$ and $\Instances$ both correspond to nodes in a $2$-dimensional integer lattice of length $n$. For a node $v$ in the grid, we have an associated $\hypothesis_v \in \Hypotheses$ and $\instance_v \in \Instances$. 
Consider an example $\example_v = (\instance_{\example_v},\clabel_{\example_v}) \in \Examples$: $\clabel_{\example_v}=0$ (negative) if the target hypothesis corresponds to the same node $v$, and $1$ (positive) elsewhere.  We consider the problem of teaching with positive-only examples.
  %
\end{exampledef}
\lattice class is inspired by teaching in a physical world from positive-only (or negative-only) reinforcements, for instance, teaching a robot to navigate to a target state by signaling that the current location is not the target. The problem of learning and teaching with positive-only examples is an important question with applications to learning languages and reinforcement learning tasks \cite{gold1967language,lange1996incremental}.
For the \lattice class, we assume that the learner prefers to move to a close-by hypothesis measured via $L1$ (Manhattan) distance, and when hypotheses have equal distances we assume that the learner prefers hypotheses with larger coordinates.


\begin{theorem}\label{thm:adaptivity}
  For teaching the \tworec class, the ratio between the cost of the optimal non-adaptive teacher and the optimal adaptive teacher is $\bigOmega{|\hinit|/\log|\hinit|}$, where $|\hinit|$ denotes the number of positive examples induced by the learner's initial hypothesis $\hinit$; for teaching the \lattice class, the difference between the cost of the optimal non-adaptive teacher and the optimal adaptive teacher is $\bigOmega{n}$.
\end{theorem}
In the above theorem, we show that for both problems, under natural behavior of an incremental learner, adaptivity plays a key role.
The proof of \thmref{thm:adaptivity} is provided in \extversion. Specifically, we show the teaching sequences for an adaptive teacher which matches the above bounds for the \tworec and \lattice classes. We also provide lower bounds for any non-adaptive algorithms for these two classes.
Here, we highlight two necessary conditions under which adaptivity can possibly help: (i) preferences are local and (ii) there are ties among the learner's preference over hypotheses. The learner's current hypothesis, combined with the local preference structure, gives the teacher a handle to steer the learner in a controlled way. 

\section{Adaptive Teaching Algorithms} \label{sec:algorithm}
In this section, we first characterize the optimal teaching algorithm, and then propose a non-myopic adaptive teaching framework.
\subsection{The Optimality Condition} \label{sec:alg:optimal}
Assume that the learner's current hypothesis is $\hypothesis$, and the current version space is $\hypotheses \subseteq \Hypotheses$. Let $\futurecostopt(\hypothesis, \hypotheses, \hstar)$ denote the minimal number of examples required in the worst-case to teach $\hstar$.
We identify the following optimality condition for an adaptive teacher:
\begin{proposition}
  A teacher achieves the minimal teaching cost, if and only if for all states $(\hypothesis, \hypotheses)$ of the learner, it picks an example such that
  \begin{align*}
    \example^* \in \argmin_{\example}\left( 1 + \max_{\hypothesis'\in \mathbf{C}(\hypothesis,\hypotheses,\sigma, \example)}\futurecostopt\left(\hypothesis', \hypotheses \cap \Hypotheses(\{\example\}), \hstar \right) \right)
  \end{align*}
  where $\mathbf{C}(\hypothesis,\hypotheses,\sigma, \example)$ denotes the set of candidate hypotheses in the next round as defined in \eqref{eq.learners-jump}, and for all $(\hypothesis, \hypotheses)$, it holds that
  \begin{align*}
    \futurecostopt(\hypothesis, \hypotheses, \hstar) = \min_{\example}\left( 1 + \max_{\hypothesis'\in \mathbf{C}(\hypothesis,\hypotheses,\sigma, \example)}\futurecostopt\left(\hypothesis', \hypotheses \cap \Hypotheses(\{\example\}), \hstar\right) \right)
  \end{align*}
\end{proposition}
In general, computing the optimal cost $\futurecostopt$ for non-trivial preference functions, including the uniform/global preference, requires solving a linear equation system of size $|\Hypotheses|\cdot 2^{|\Hypotheses|}$.
\paragraph{State-independent preference} 
When the learner's preference is uniform, $\futurecostopt_u(\hypothesis, \hypotheses, \hstar)=\TD(\hstar, \hypotheses)$ (Eq.~\ref{eq:td}) denotes the set cover number of the version space, which is NP-Hard to compute. A myopic heuristic which gives best approximation guarantees for a polynomial time algorithm (with cost that is within a logarithmic factor of the optimal cost \cite{goldman1995complexity}) is given by $\futurecostapprox_u(\hypothesis, \hypotheses, \hstar)=|\hypotheses|$.
For the global preference, the optimal cost 
$\futurecostopt_g(\hypothesis, \hypotheses, \hstar)$ 
is given by Eq.~\eqref{eq:pbtd}.
i.e., the set cover number of all hypotheses in the version space that are more or equally preferred over $\hstar$. Similarly, one can also follow the greedy heuristic, i.e., $\futurecostapprox_g(\hypothesis, \hypotheses, \hstar)= |\{\hypothesis' \in \hypotheses: \orderingof{\hypothesis'}{\cdot} \leq \orderingof{\hstar}{\cdot}\}|$ to achieve a logarithmic factor approximation.
\paragraph{General preference} Inspired by the two myopic heuristics above, we propose the following heuristic for  general preference models:
\begin{align}
  \label{eq:objrank}
  \futurecostapprox(\hypothesis, \hypotheses, \hstar)= |\{\hypothesis' \in \hypotheses: \orderingof{\hypothesis'}{\hypothesis} \leq \orderingof{\hstar}{\hypothesis}\}|
\end{align}

In words, $\futurecostapprox$ denotes the index of the target hypothesis $\hstar$ in the preference vector associated with $\hypothesis$ in the version space $\hypotheses$.
Notice that for the uniform (resp. global) preference model, the function $\futurecostapprox$ reduces to $\futurecostapprox_u$ (resp. $\futurecostapprox_g$).
In the following theorem, we provide a sufficient condition for the myopic adaptive algorithm that greedily minimizes Eq.~\eqref{eq:objrank}
to attain provable guarantees:
\begin{theorem}\label{thm:greedy_suff}
  Let $\hinit\in \Hypotheses$ be the learner's initial hypothesis, and $\hstar \in \Hypotheses$ be the target hypothesis. For any $\hypotheses \subseteq \Hypotheses$, let $\bar{\hypotheses}(\{\example\}) = \{\hypothesis' \in \hypotheses: \hypothesis'(\instance_\example) \neq \lbl_\example\}$ be the set of hypotheses in $\hypotheses$ which are inconsistent with the teaching example $\example \in \Examples$. 
  If for all learner's states $(h, \hypotheses)$, the preference and the structure of the teaching examples satisfy:
  \begin{enumerate}\denselist
  \item $\forall \hypothesis_i, \hypothesis_j \in \hypotheses$, $\orderingof{\hypothesis_i}{\hypothesis} \leq \orderingof{\hypothesis_j}{\hypothesis} \leq \orderingof{\hstar}{\hypothesis} \implies \orderingof{\hypothesis_j}{\hypothesis_i} \leq \orderingof{\hstar}{\hypothesis_i}$\label{thm:suffcond:1}
  \item $\forall \hypotheses' \subseteq \bar{\hypotheses}(\{\example\})$, there exists $\example'\in \Examples$, s.t., $\bar{\hypotheses}(\{\example'\}) = \hypotheses'$, \label{thm:suffcond:2}
  \end{enumerate}
  then, the cost of the myopic algorithm that greedily minimizes\footnote{In the case of ties, we assume that the teacher prefers examples that make learner stay at the same hypothesis.} \eqref{eq:objrank} is within a factor of $2(\log{\futurecostapprox(\hypothesis_0, \Hypotheses, \hstar)}+1)$ approximation of the cost of the optimal adaptive algorithm.
\end{theorem}

\looseness -1 
We defer the proof of the theorem to \extversion. Note that both the uniform preference model and the global preference model satisfy Condition~\ref{thm:suffcond:1}. Intuitively, the first condition states that there does not exist any hypothesis between $\hypothesis$ and $\hstar$ that provides a ``short-cut'' to the target. Condition~\ref{thm:suffcond:2} implies that we can always find teaching examples that ensure smooth updates of the version space.  
For instance, a feasible setting that fits Condition ~\ref{thm:suffcond:2} is where we assume that the teacher can synthesize an example to remove any subset of hypotheses of size at most $k$, where $k$ is some constant.



\begin{figure}[t]
  \begin{minipage}[t]{0.45\linewidth}
    \begin{algorithm}[H]
      \caption{Non-myopic adaptive teaching}\label{alg:non-myopic}
      \begin{algorithmic}
        \STATE {\bf input:}
        $\Hypotheses$, $\ordering$, initial $\hypothesis_0$, target $\hstar$.
        \STATE Initialize $t \leftarrow 0$,
        $\Hypotheses_0 \leftarrow \Hypotheses$
        \WHILE{$\hypothesis_t \neq \hypothesis^*$}
        \STATE $\hstar_t \leftarrow \oracle(\hypothesis_t, \Hypotheses_t, \hstar)$
        \STATE $\example_{t+1} \leftarrow \teachalg(\ordering, \hypothesis_t, \Hypotheses_t, \hstar_t)$
        \STATE Learner makes an update
        \STATE $t \leftarrow t+1$
        \ENDWHILE
      \end{algorithmic}
    \end{algorithm}
  \end{minipage}
  \hfill
  \begin{minipage}[t]{0.53\linewidth}
    \begin{figure}[H]
      \centering
      \begin{subfigure}[b]{0.25\textwidth}
        \includegraphics[width=1.0\linewidth]{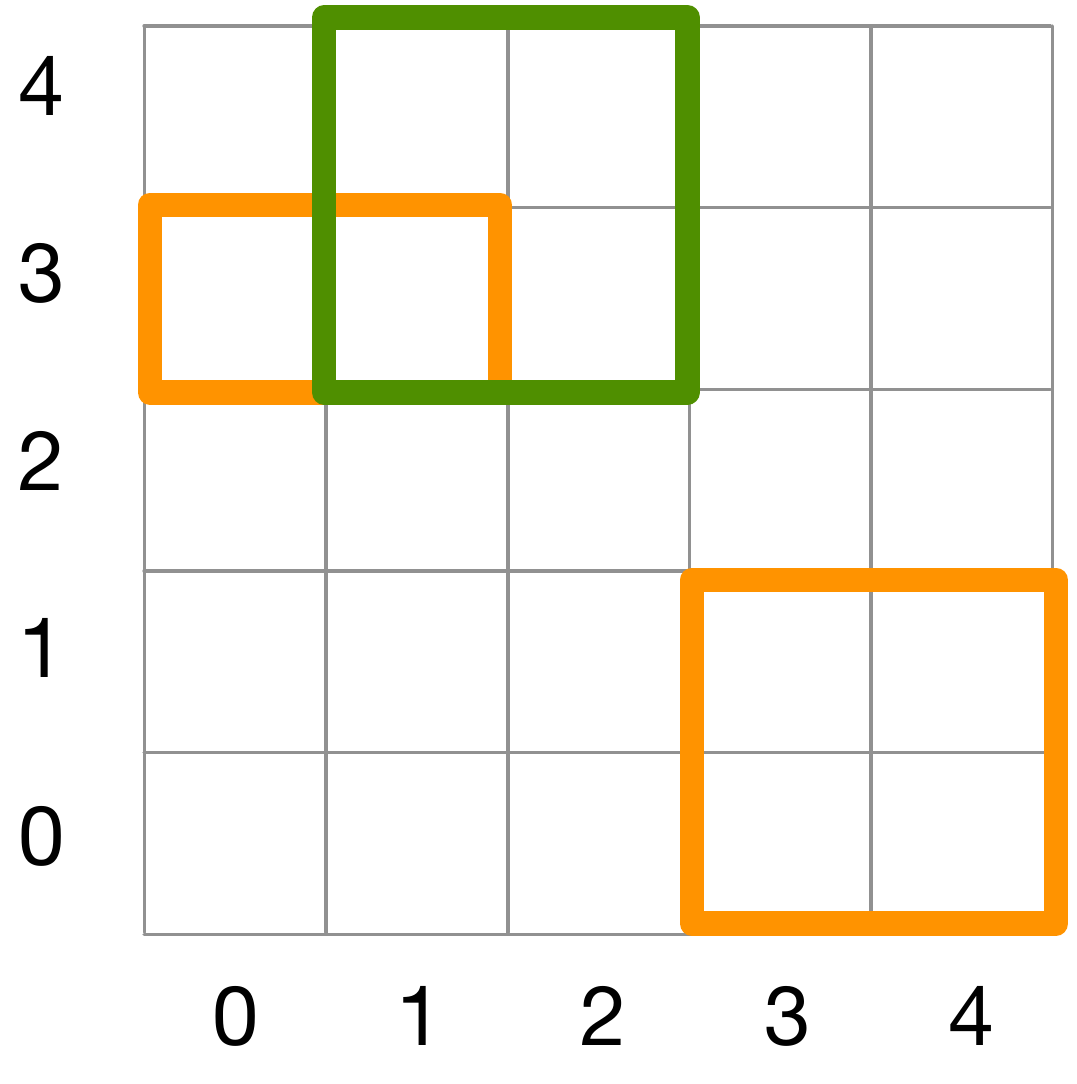}
      \end{subfigure}
      \quad
      \begin{subfigure}[b]{0.25\textwidth}
        \includegraphics[trim={8pt 12pt 5pt 5pt},width=1.0\textwidth]{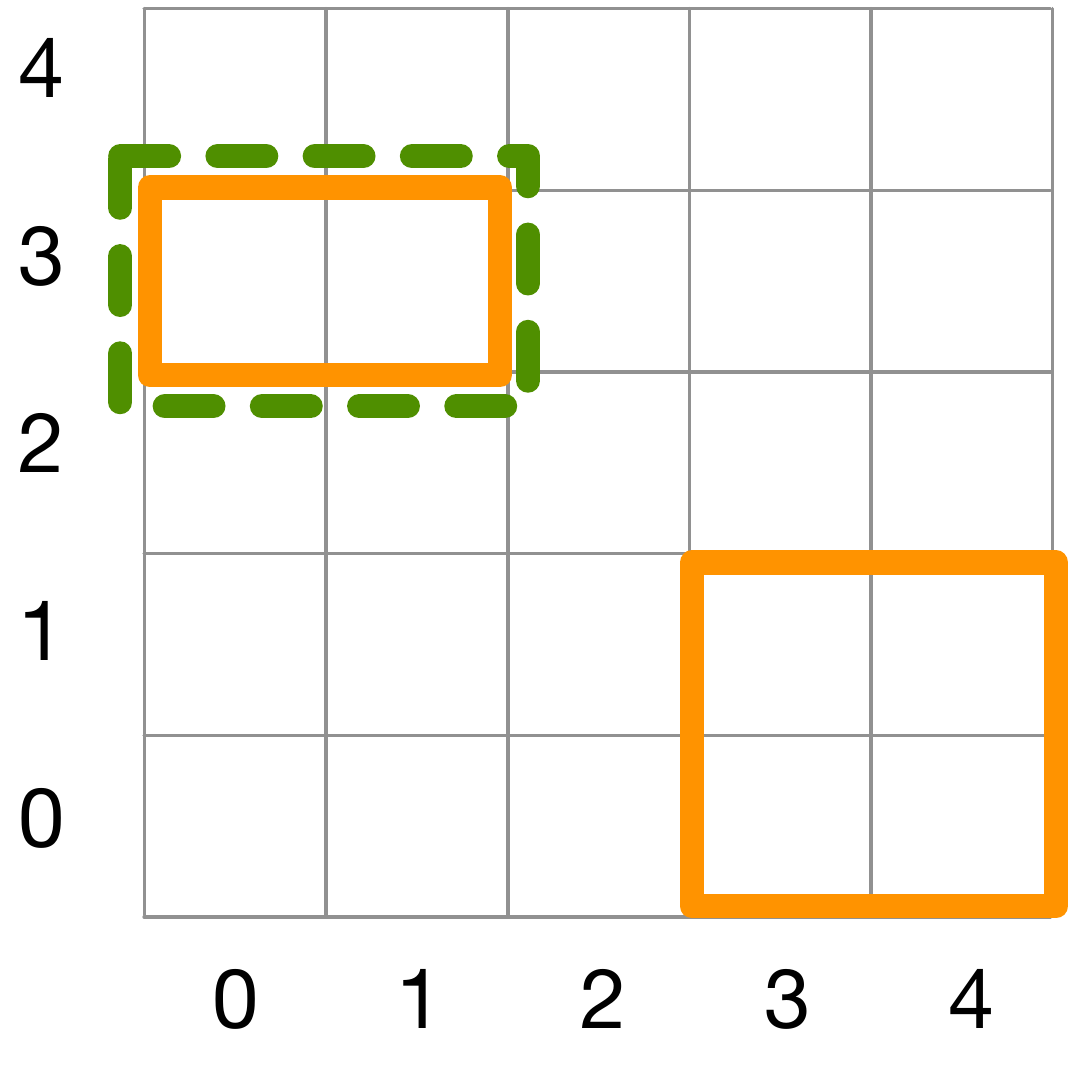}
      \end{subfigure}
      \quad
      \begin{subfigure}[b]{0.25\textwidth}
        \includegraphics[trim={8pt 12pt 5pt 5pt},width=1.0\textwidth]{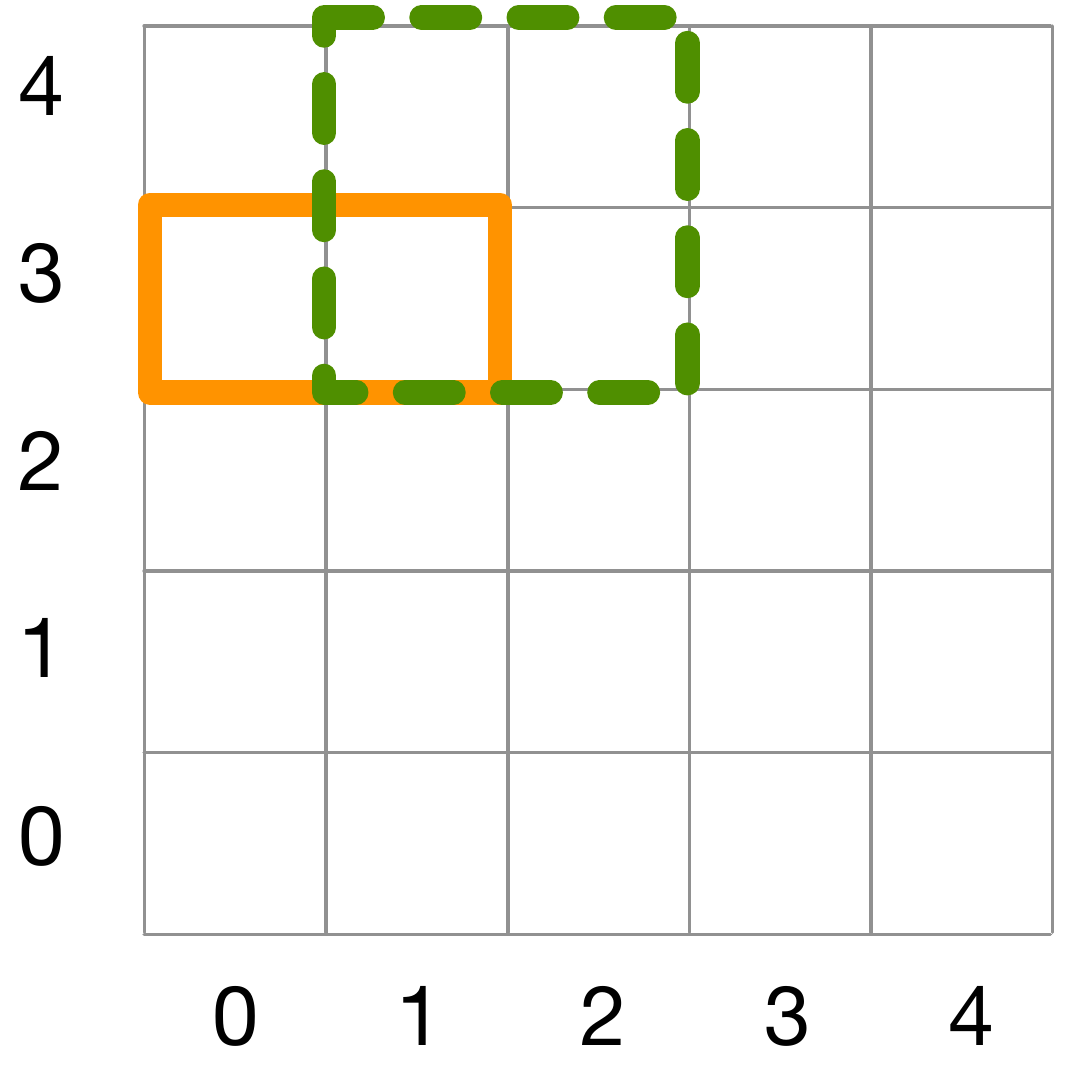}
      \end{subfigure}
      \caption{\small An illustrative example for \tworec. $\hypothesis_t$, $\hstar$, and $\hstar_t$ are represented by the orange rectangles, solid green rectangle and dashed green rectangles, respectively. (Left) The teaching task. (Middle) Sub-task 1. (Right) Sub-task 2.}
      \label{fig:tworec:subtasks}
    \end{figure}
  \end{minipage}
\end{figure}

\subsection{Non-Myopic Teaching Algorithms}\label{sec:application-specific}

When the conditions provided in \thmref{thm:greedy_suff} do not hold, the myopic heuristic \eqref{eq:objrank} could perform poorly. 
An important observation from \thmref{thm:greedy_suff} is that, when $\futurecostapprox(\hypothesis, \hypotheses, \hstar)$ is small, i.e., $\hstar$ is close to the learner's current hypothesis in terms of preference ordering, we need less stringent constraints on the preference function.
This motivates adaptively devising intermediate target hypotheses to ground the teaching task into multiple, separate sub-tasks.
Such divide-and-conquer approaches have proven useful for many practical problems, e.g., constructing a hierarchical decomposition for reinforcement learning tasks \cite{
  Hengst2010}.
In the context of machine teaching, we assume that there is an oracle, $\oracle(\hypothesis, \hypotheses, \hstar)$ that maps the learner's state $(\hypothesis, \hypotheses)$ and the target hypothesis $\hstar$ to an intermediate target hypothesis, which defines the current sub-task.

We outline the non-myopic adaptive teaching framework in \algref{alg:non-myopic}.
Here, the subroutine \teachalg aims to provide teaching examples that bring the learner closer to the intermediate target hypothesis. 
As an example, let us consider the \tworec hypothesis class. In particular, we consider the challenging case where the target hypothesis $\hstar \in \Hypotheses^1$ represents a single rectangle $r^\star$, and the learner's initial hypothesis $\hypothesis_0 \in \Hypotheses^2$ has two rectangles $(r_{1}, r_{2})$. Imagine that the first rectangle $r_{1}$ is overlapping with $r^\star$, and the second rectangle $r_{2}$ is disjoint from $r^\star$. To teach the hypothesis $\hstar$, the first sub-task (as provided by the oracle) is to eliminate the rectangle $r_{2}$ by providing negative examples so that the learner's hypothesis represents a single rectangle $r_{1}$. Then, the next sub-task (as provided by the oracle) is to teach $\hstar$ from $r_{1}$. We illustrate the sub-tasks in \figref{fig:tworec:subtasks}, and provide the full details of the adaptive teaching algorithm (i.e., \adar as used in our experiments) in \extversion.
\section{Experiments}\label{sec:exp}

In this section, we empirically evaluate our teaching algorithms on the \tworec hypothesis class via simulated learners.
\subsection{Experimental Setup}
For the \tworec hypothesis class (cf. \figref{fig:exp:recexample} and Example~\ref{example:rect}), we consider a grid with size varying from $5\times 5$ to $20 \times 20$. The ground set of unlabeled teaching examples $\Instances$ consists of all grid cells.  In our simulations, we consider all four possible teaching scenarios, $\Hypotheses^{1\rightarrow 1}$, $\Hypotheses^{1\rightarrow 2}$, $\Hypotheses^{2\rightarrow 1}$, $\Hypotheses^{2\rightarrow 2}$, where $i,j$ in $\Hypotheses^{i\rightarrow j}$ specify the subclasses of the learner's initial hypothesis $\hinit$ and the target hypothesis $\hstar$. In each simulated teaching session, we sample a random pair of hypotheses $(\hinit, \hstar)$ from the corresponding subclasses. 

\paragraph{Teaching algorithms}
We consider three different teaching algorithms as described below. The first algorithm, \setcover, is a greedy set cover algorithm, where the teacher greedily minimizes $\futurecostapprox_u=|\hypotheses|$ (see \secref{sec:alg:optimal}). In words, the teacher acts according to the uniform preference model, and greedily picks the teaching example that eliminates the most inconsistent hypotheses in the  version space. The second algorithm, denoted by \nonadar for the class \tworec, represents the non-adaptive teaching algorithm that matches the non-adaptive lower bounds provided in \thmref{thm:adaptivity}, with implementation details provided in \extversion. 
Note that both \setcover and \nonadar are non-adaptive. The third algorithm, \adar, represents the non-myopic adaptive teaching algorithm instantiated from \algref{alg:non-myopic}. The details of the subroutines \oracle and \teachalg  for \adar are provided in \extversion.
We note that all teaching algorithms have the same stopping criterion: the teacher stops when the learner reaches the target hypothesis, that is, $\hypothesis_{t} = \hstar$.

\begin{figure*}[!t]
  \centering
  \begin{subfigure}[b]{0.17\textwidth}
    \includegraphics[width=1.0\linewidth]{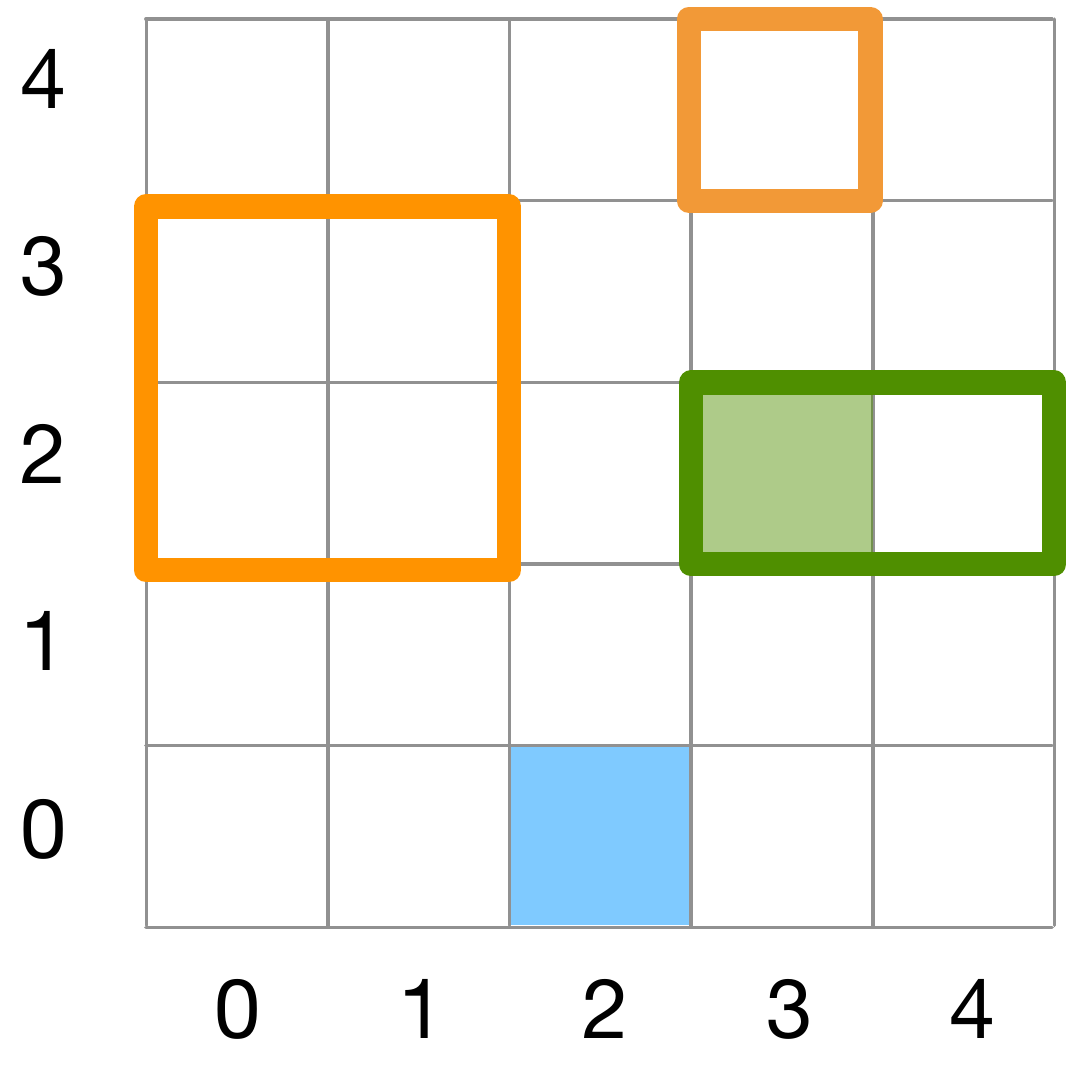}
    \caption{\tworec class}
    \label{fig:exp:recexample}
  \end{subfigure}
  \quad
  \begin{subfigure}[b]{0.25\textwidth}
    \includegraphics[trim={8pt 12pt 5pt 5pt},height=2.6cm,width=1.0\textwidth]{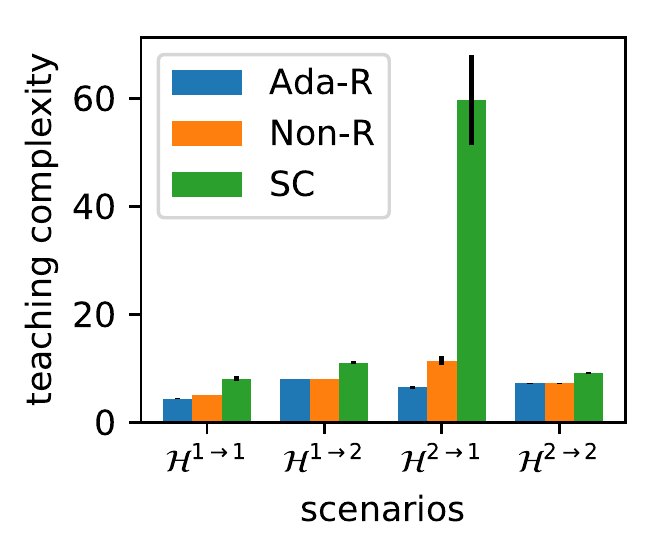}
    \caption{\tworec, size $15\times 15$}
    \label{fig:exp:cplx-vs-scenarios-2rec}
  \end{subfigure}
  ~
  \begin{subfigure}[b]{0.23\textwidth}
    \includegraphics[trim={8pt 12pt 5pt 5pt},height=2.6cm,width=1.0\textwidth]{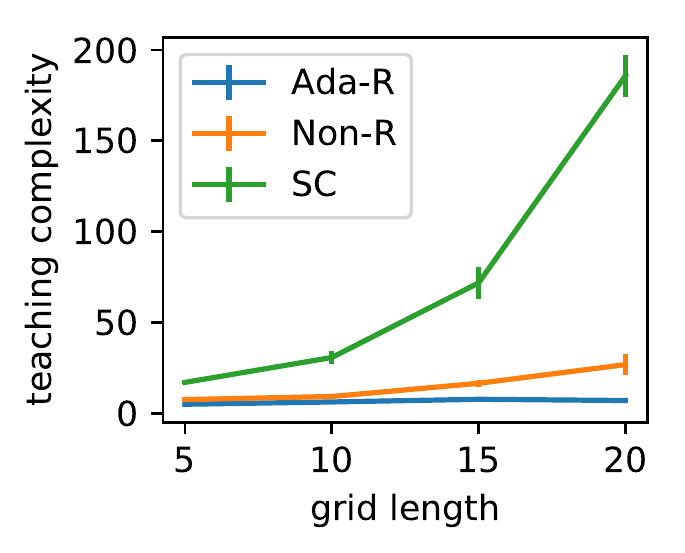}
    \caption{\tworec, $\Hypotheses^{2\rightarrow 1}$}
    \label{fig:exp:cplx-vs-gsz-2rec}
  \end{subfigure}
  ~
  \begin{subfigure}[b]{0.24\textwidth}
    \includegraphics[trim={8pt 5pt 5pt 5pt},height=2.6cm,width=1.0\textwidth]{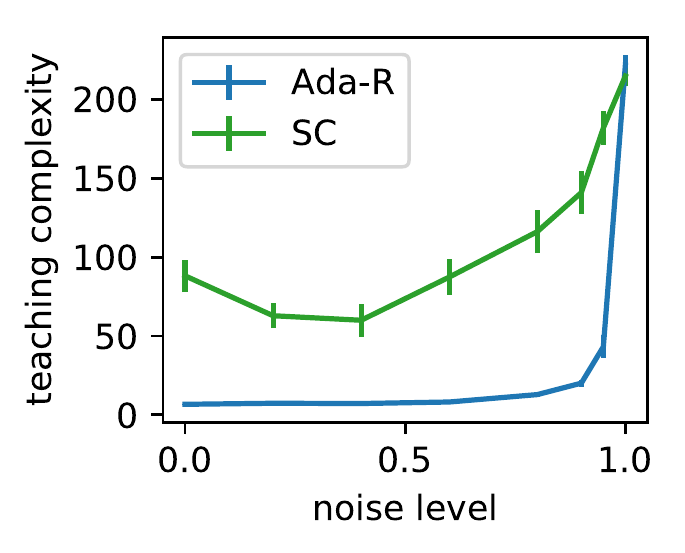}
    \caption{\tworec, robustness}
    \label{fig:exp:cplx-vs-p-2rec}
  \end{subfigure}\\
  \caption{ \looseness -1 Illustration and simulation results for \tworec. (a) illustrates the \tworec hypothesis class. The initial hypothesis $\hinit \in \Hypotheses^2$ is represented by the orange rectangles, and the target hypothesis $\hstar \in \Hypotheses^1$ is represented by the green rectangle. The green and blue cells represent a positive and a negative teaching example, respectively. Simulation results are shown in (b)-(d).}
  \label{fig:exp:performance}
\end{figure*}

\subsection{Results}\label{sec:simulation:results}
We measure the performance of the teaching algorithms by their teaching complexity, and all results are averaged over 50 trials with random samples of $(\hinit, \hstar)$.

\paragraph{Noise-free setting}
Here, we consider the ``noise-free'' setting, i.e., the learner acts according to the state-dependent preference models as described in \secref{sec:statepref}.
In \figref{fig:exp:cplx-vs-scenarios-2rec}, we show the results for \tworec class with a fixed grid size $15\times 15$ for all four teaching scenarios.
As we can see from \figref{fig:exp:cplx-vs-scenarios-2rec}, \adar has a consistent advantage over the non-adaptive baselines across all four scenarios. As expected, teaching $\Hypotheses^{1\rightarrow 1}, \Hypotheses^{1\rightarrow 2}$, and $\Hypotheses^{2\rightarrow 2}$ is easier, and the non-adaptive algorithms (\setcover and \nonadar) perform well. In contrast, when teaching $\Hypotheses^{2\rightarrow 1}$, we see a significant gain from \adar over the non-adaptive baselines. In the worst case, \setcover has to explore \emph{all} the negative examples to teach $\hstar$, whereas \nonadar needs to consider all negative examples within the learner's initial hypothesis $\hinit$ to make the learner jump from the subclass $\Hypotheses^2$ to $\Hypotheses^1$.
In \figref{fig:exp:cplx-vs-gsz-2rec}, we observe that the adaptivity gain increases drastically as we increase the grid size. This matches our analysis of the logarithmic adaptivity gain in \thmref{thm:adaptivity} for \tworec.


\paragraph{Robustness in a noisy setting}

In real-world teaching tasks, the learner's preference  may deviate from the preference $\ordering$ of an ``ideal'' learner that the teacher is modeling. In this experiment, we consider a more realistic scenario, where we simulate the noisy learners by randomly perturbing the preference of the ``ideal'' learner at each time step. With probability $1-\varepsilon$ the learner follows $\ordering$, and with probability $\varepsilon$, the learner switches to a random hypothesis in the version space.
In \figref{fig:exp:cplx-vs-p-2rec}, we show the results for the \tworec hypothesis class with different noise levels $\varepsilon \in [0,1]$.
We observe that even for highly noisy learners e.g., $\varepsilon = 0.9$, our algorithm \adar performs much better than \setcover. 
\footnote{ In general, the teaching sequence constructed by the non-adaptive algorithms \nonadar (resp. \nonadal) would not be sufficient to reach the target under the noisy setting. Hence, we did not include the results of these non-adaptive algorithms in the robustness plots. Note that one can tweak \nonadar (resp. \nonadal) by concatenating the teaching sequence with teaching examples generated by \setcover; however, in general, in a worst-case sense, any non-adaptive algorithm in the noisy setting will not perform better than SC.\label{footnote:reason:no:sc}}$^,$\footnote{The performance of \setcover is non-monotone w.r.t. the noise-level. This is attributed to the stopping criteria of the algorithm as the increase in the noise level increases the chance for the learner to randomly jump to $h^*$.}

\section{User Study}\label{sec:usrstudy}




\looseness -1 Here we describe experiments performed with human participants from Mechanical Turk using the \tworec hypothesis class. We created a web interface in order to (i) elicit the preference over hypotheses of human participants, and to (ii) evaluate our adaptive algorithm when teaching human learners.

\paragraph{Eliciting human preferences}
We consider a two-step process for the elicitation experiments. At the beginning of the session (first step), participants were shown a grid of green, blue, or white cells and asked to draw a hypothesis from the \tworec class represented by one or two rectangles. Participants could only draw ``valid'' hypothesis which is consistent with the observed labels (i.e., the hypothesis should contain all the green cells and exclude all the blue cells), cf. \figref{fig:exp:recexample}. The color of the revealed cells is defined by an underlying target hypothesis $\hstar$. In the second step, the interface updated the configuration of cells (either by adding or deleting green/blue cells) and participants were asked to redraw their rectangle(s) (or move the edges of the previously drawn rectangle(s)) which ensures that the updated hypothesis is consistent.


We consider 
$5$ types of sessions, depending on the class of $\hstar$ and configurations presented to a participant in the first and the second step. These 
configurations are listed in \figref{tab:usrstudy:trans_across_cls}. For instance, the session type in the third row $(\Hypotheses^2,(1/2),2)$ means the following: the labels were generated based on a hypothesis $\hstar \in \Hypotheses^2$; in the first step, both subclasses $\Hypotheses^1$ and $\Hypotheses^2$ had consistent hypotheses; and in the second step, only the subclass $\Hypotheses^2$ had consistent hypotheses.



\looseness -1 We tested $215$ participants, where each individual performed $10$ trials on a grid of size $12\times 12$. For each trial, we randomly selected one of the five types of sessions as discussed above.
In \figref{tab:usrstudy:trans_across_cls}, we see that participants tend to favor staying in the same hypothesis subclass when possible. Within the same subclass, they have a preference towards updates that are close to their initial hypothesis, cf. \figref{fig:elicitation_1}.\footnote{Given that a participant is allowed to move edges when updating the hypothesis, our interface could bias the participants' choice of the next hypothesis towards a preference structure that favors local edits as assumed by our algorithm. As future work, one could consider an alternative interface which enforces participants to draw the rectangle(s) from scratch at every step.}

\begin{figure}
  \centering
  \begin{subfigure}[b]{0.45\textwidth}
    \raisebox{1.5cm}{
      \resizebox{\linewidth}{!}{
        \renewcommand{\arraystretch}{1.25}
        \begin{tabular}{c c c|c c c c}
          \multicolumn{3}{c}{Session Type} & \multicolumn{4}{c}{User Transition}\\
          $\hstar$ & 1st & 2nd 
                         & $\Hypotheses^{1} \rightarrow \Hypotheses^{1}$ 
                         & $\Hypotheses^{1} \rightarrow \Hypotheses^{2}$ 
                         & $\Hypotheses^{2} \rightarrow \Hypotheses^{1}$ 
                         & $\Hypotheses^{2} \rightarrow \Hypotheses^{2}$\\\hline
          $\Hypotheses^1$ & ${(1/2)}$ & ${(1/2)}$ & 0.54 &   0.13 &  0.06 &  0.27\\
          $\Hypotheses^2$ & ${(1/2)}$ & ${(1/2)}$ & 0.30 &   0.31 &  0.01 &  0.38\\
          $\Hypotheses^2$ & ${(1/2)}$ & ${2}$    & 0.00 &   0.61 &  0.00 &  0.39  \\
          $\Hypotheses^2$ &  ${2}$ & ${2}$ & 0.00 &   0.00 &  0.00 &  1.00 \\
          $\Hypotheses^{2}$ &  ${2}$ & ${(1/2)}$ & 0.00&0.00   &  0.11 &  0.89 \\
        \end{tabular}
      }
    }
    \caption{Transitions across subclasses $\Hypotheses^{i} \rightarrow \Hypotheses^{j}$}
    \label{tab:usrstudy:trans_across_cls}
  \end{subfigure}\quad
  \begin{subfigure}[b]{0.255\textwidth}
    \includegraphics[width=\linewidth]{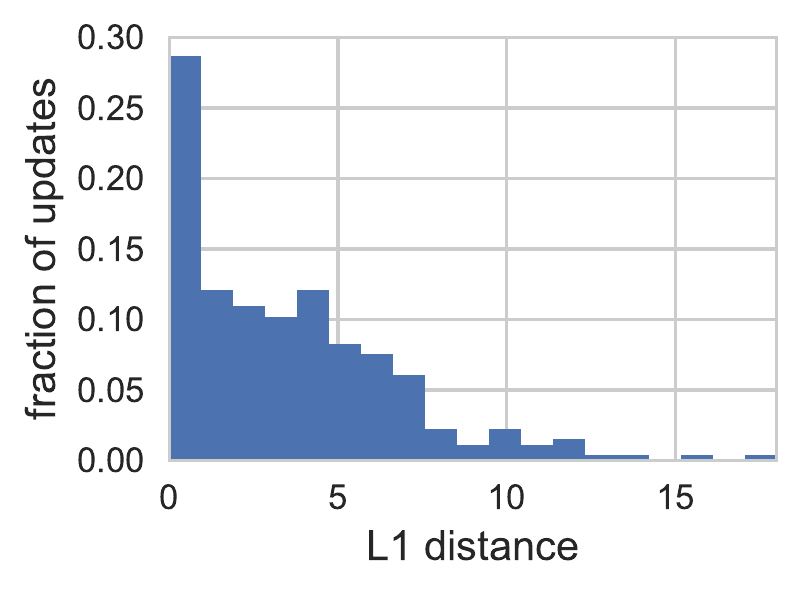}
    \caption{Transitions within $\Hypotheses^{1}$}
    \label{fig:elicitation_1}
  \end{subfigure}
  \hfill
  \begin{subfigure}[b]{0.255\textwidth}
    \includegraphics[width=\linewidth]{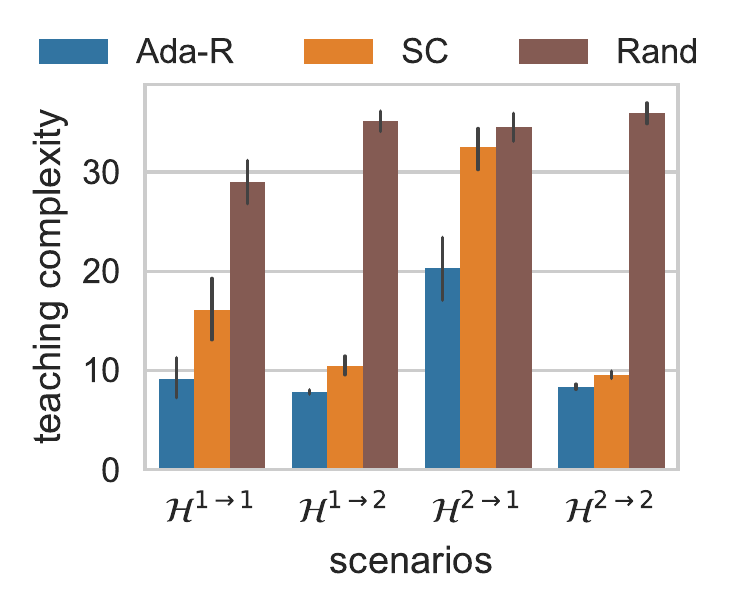}
    \caption{Teaching results}
    \label{fig:usrstudy:teaching}
  \end{subfigure}
  \caption{(a)-(b) represent results for eliciting human preferences for different session types as explained in the text below and (c) shows results for teaching human learners.
    (a) Participants prefer staying within the same hypothesis subclass when possible, displayed as the fraction of time they switched subclasses for different session types. 
    (b) Considering the transitions within subclass $\Hypotheses^1$, participants favor staying at their current hypothesis if it remains valid, along with preferring smaller updates, computed as the $L1$ distance between the initial and updated rectangle.
    (c) Adaptive teaching algorithm \adar is significantly better than \setcover and \textsf{Rand}.
  }
  \label{fig:usrstudy:trans_within_cls}
\end{figure}

\paragraph{Teaching human learners}
Next we evaluate our teaching algorithms on human learners. As in the simulations, we consider four teaching scenarios  $\Hypotheses^{1\rightarrow 1}$, $\Hypotheses^{1\rightarrow 2}$, $\Hypotheses^{2\rightarrow 1}$, and $\Hypotheses^{2\rightarrow 2}$. At the beginning of the teaching session, a participant was shown a blank $8\times 8$ grid with either one or two initial rectangles, corresponding to $\hypothesis^0$. At every iteration, the participants were provided with a new teaching example (i.e., a new green or blue cell is revealed), and were asked to update the current hypothesis. 


We evaluate three algorithms, namely \adar, \setcover, and \random, where \random denotes a teaching strategy that picks examples at random. The non-adaptive algorithm \nonadar was not included in the user study for the same reasons as explained in Footnote \ref{footnote:reason:no:sc}. 
We enlisted $200$ participants to evaluate teaching algorithms and this was repeated five times for each participant. For each trial, we randomly selected one of the three teaching algorithms and one of the four teaching scenarios. Then, we recorded the number of examples required to learn the target hypothesis. Teaching was terminated when $60\%$ of the cells were revealed. If the learner did not reach the target hypothesis by this time we set the number of teaching examples to this upper limit. We illustrate a teaching session in \extversion. 

\figref{fig:usrstudy:teaching} illustrates the superiority of the adaptive teacher \adar, while \random performs the worst. In both cases where the target hypothesis is in $\Hypotheses^2$, the \setcover teacher performs nearly as well as the adaptive teacher, as at most $12$ teaching examples are required to fully characterize the location of both rectangles. However, we observe a large gain from the adaptive teacher for the scenario $\Hypotheses^{2\rightarrow 1}$.





\section{Conclusions}
We explored the role of adaptivity in algorithmic machine teaching and showed that the adaptivity gain is zero when considering well-studied learner models (e.g., ``worst-case'' and ``preference-based'') for the case of version space learners. This is in stark contrast to real-life scenarios where adaptivity is an important ingredient for effective teaching. 
We highlighted the importance of local preferences (i.e., dependent on the current hypothesis) when the learner transitions to the next hypothesis. We presented hypotheses classes where such local preferences arise naturally, given that machines and humans have a tendency to learn incrementally. Furthermore, we characterized the structure of optimal adaptive teaching algorithms, designed near-optimal general purpose and application-specific adaptive algorithms, and validated these algorithms via simulation and user studies.



\paragraph{Acknowledgments} This work was supported in part by Northrop Grumman, Bloomberg, AWS Research Credits, Google as part of the Visipedia project, and a Swiss NSF Early Mobility Postdoctoral Fellowship.



\bibliography{references}

\iftoggle{longversion}{
\appendix
\section{Structure of the Appendices}\label{app:structure}

We now present the supplemental results for this paper. In \secref{app:userstudy}, we show supplemental experimental results demonstrating the smooth transitions between hypotheses in the \tworec class for human learners. In \secref{app:prefoverview}, we illustrate the preference function of the \tworec hypothesis class. We then provide the full proof of
\thmref{thm:adaptivity} in \secref{app:proof:adaptivity}, and the proof of \thmref{thm:greedy_suff} in \secref{app:proof:greedy:suff}. Finally, in \secref{app:tworec}, we present the formal definitions of the preference function and oracle function for teaching the \tworec hypothesis class, along with the details of the teaching algorithms, namely \adar and \nonadar, which we ran in our experiments.



\clearpage
\section{Supplemental Results from User Study} \label{app:userstudy}
\noindent\textbf{Example Teaching Traces}
\figref{fig:usrstudy:teaching_trace} shows an example teaching session with the adaptive teacher, visualizing the teaching of two target rectangles from one initial rectangle i.e., $\Hypotheses^{1\rightarrow 2}$
At each time step, a new square is revealed by the teacher and the learner updates her hypothesis accordingly (depicted here as an orange rectangle).

\begin{figure*}[h]
  \centering
  \includegraphics[width=1.0\textwidth]{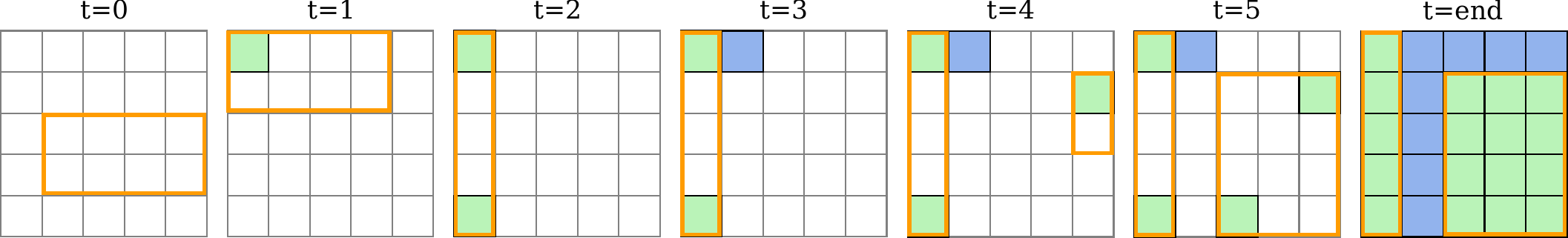}
  \label{fig:user_teach}
  \caption{Adaptive teaching session. Here we see an example teaching session with the adaptive teacher \adar for a grid of size $5 \times 5$. The learner's hypothesis (in orange) is updated over time upon observing squares that are revealed by the teacher. In the final image we see the target hypothesis. For space reasons we omit some of the intermediate time steps.}
  \label{fig:usrstudy:teaching_trace}
\end{figure*}

\noindent\textbf{Local Preference of Participants}
For eliciting the human update preferences depicted in \figref{fig:intro.smoothness}, participants were shown a $10 \times 10$ grid and given an initial hypothesis and a subset of revealed squares. They were then instructed to update the position of the orange rectangle so that it contained green squares, with no blue squares inside it. They were free to draw up to two rectangles in total (including the initial hypothesis) and could move rectangles by clicking the center or grabbing the corners and dragging them to move the edges. They could also click to delete a rectangle and redraw it anywhere on the grid. The rectangles were fixed to only live on the grid lines.
The task was completed when they submitted a valid configuration of the rectangles.
\begin{figure*}[h]
  \centering
  \begin{subfigure}[b]{0.32\textwidth}
    \includegraphics[trim={0pt 70pt 0pt 30pt},width=1.0\textwidth]{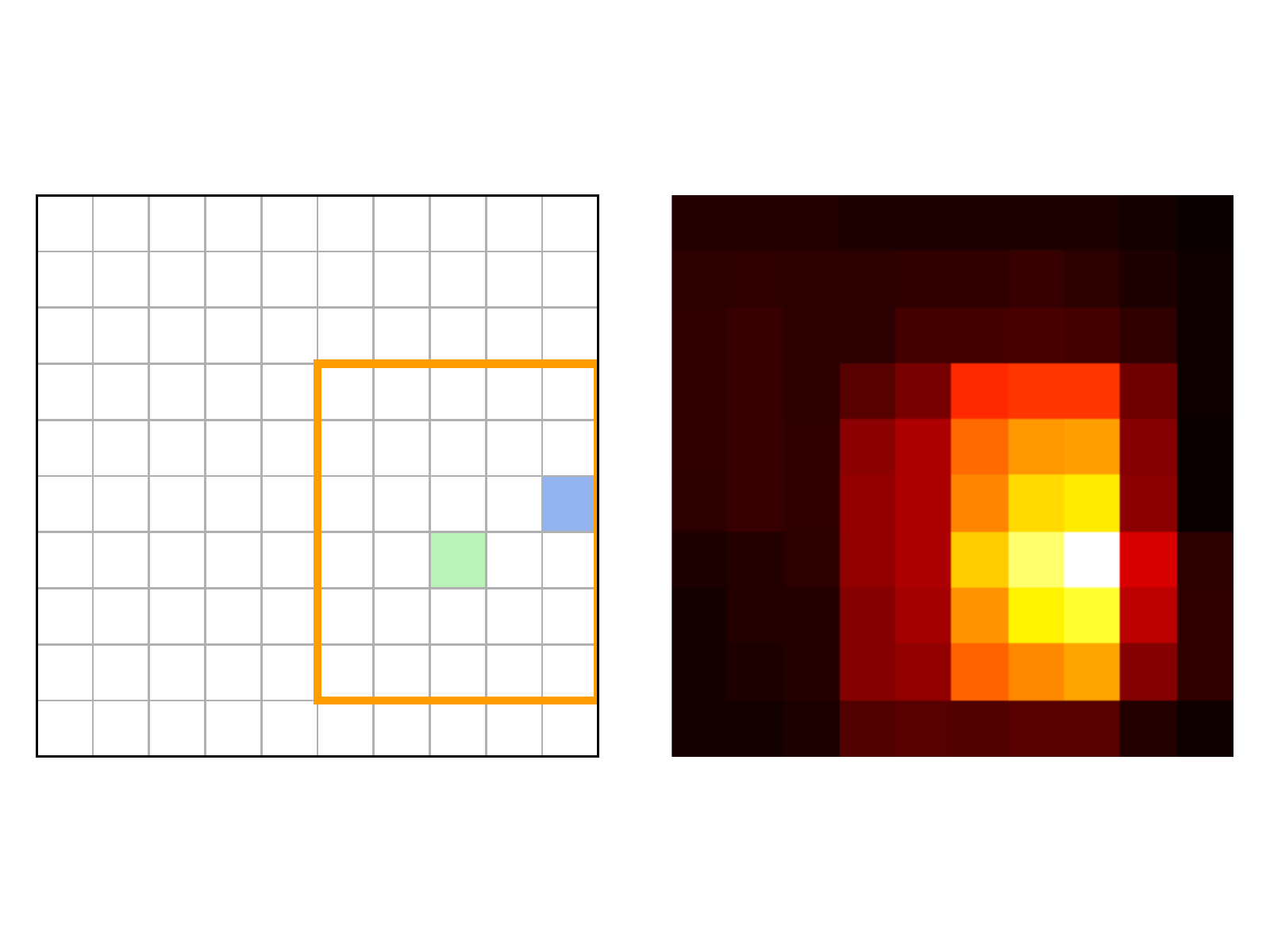}
    \caption{Example configuration 1}
    \label{fig:supp.smoothness1}
  \end{subfigure}
  \begin{subfigure}[b]{0.32\textwidth}
    \includegraphics[trim={0pt 70pt 0pt 30pt},width=1.0\textwidth]{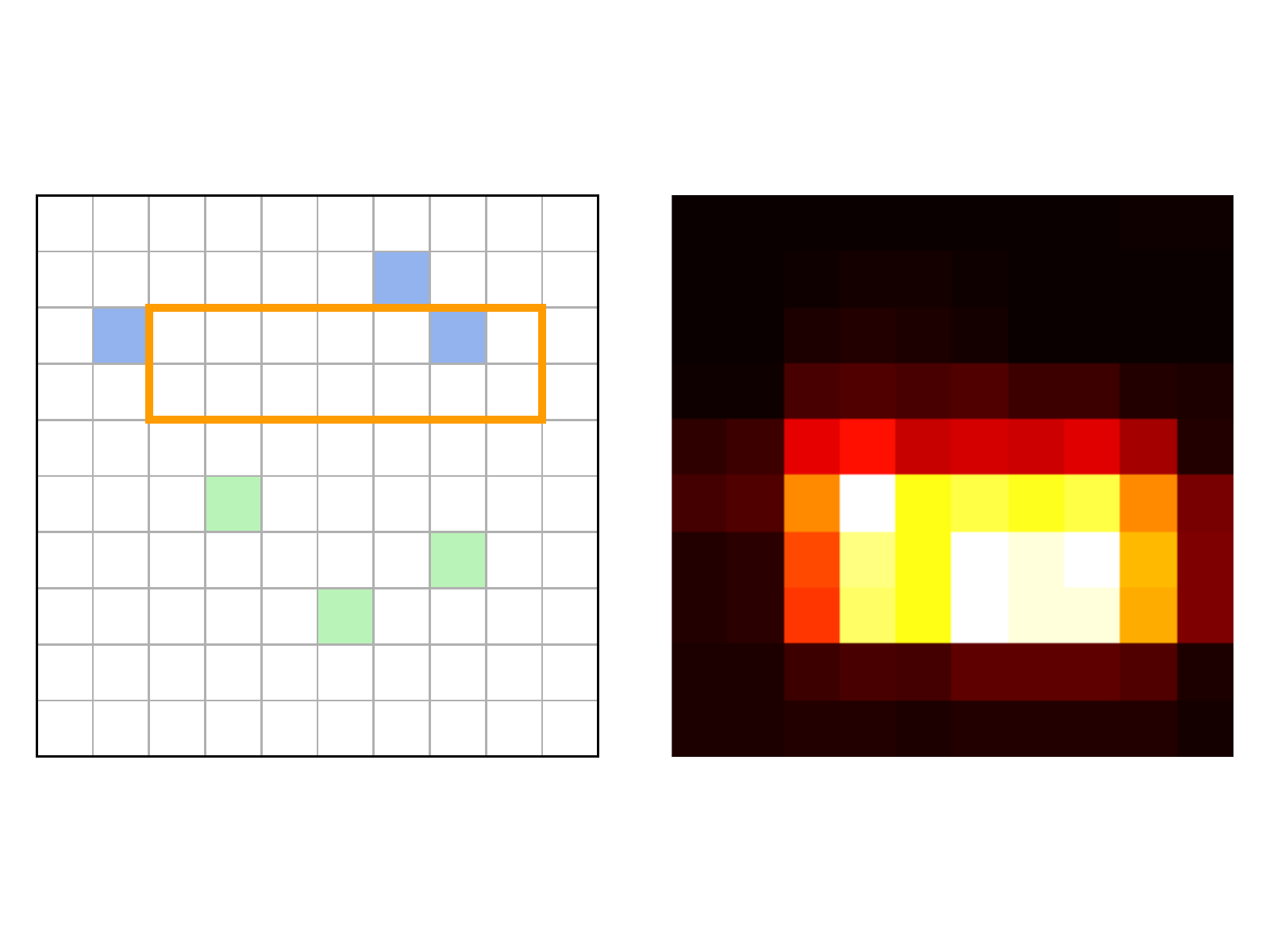}
    \caption{Example configuration 2}
    \label{fig:supp.smoothness2}
  \end{subfigure}
  \begin{subfigure}[b]{0.32\textwidth}
    \includegraphics[trim={0pt 70pt 0pt 30pt},width=1.0\textwidth]{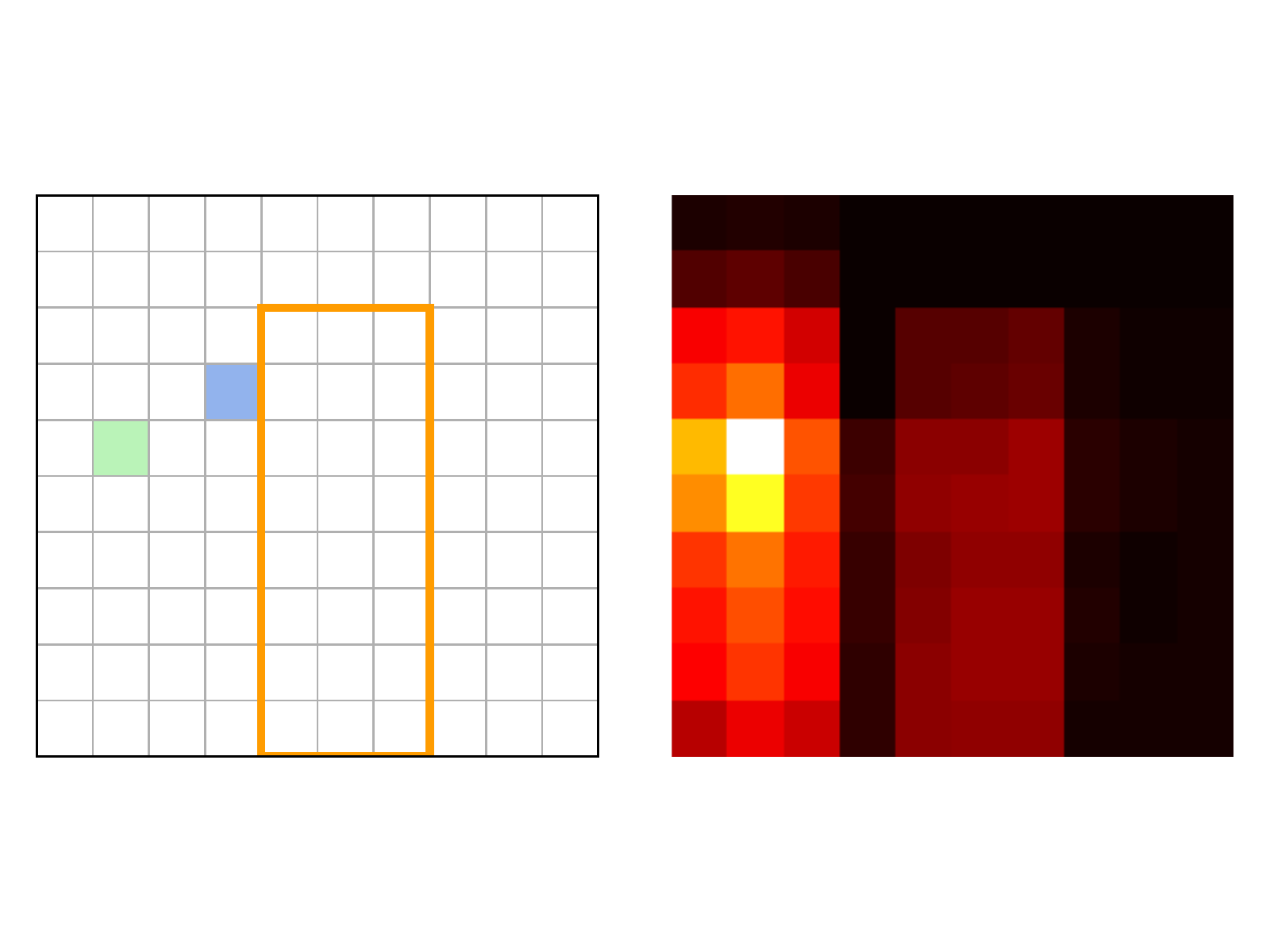}
    \caption{Example configuration 3}
    \label{fig:supp.smoothness3}
  \end{subfigure}
  \caption{Participants prefer smooth updates. For each pair, the left depicts the initial configuration of the grid shown to participants and the right is the heat map of their updated valid rectangles. Again, green squares had to lie within the rectange(s) and blue squares had to be outside.}
  \label{fig:supp.smoothness}
\end{figure*}

$100$ participants on Mechanical Turk were each shown the same set of initial configurations.
We show the heatmap of their updated rectangles for three of the ten configurations in \figref{fig:supp.smoothness}.
Each position on the heatmap records the number of updated rectangles that overlap with that location, where brighter colors indicate more rectangles.
We clearly see a preference for local updates as compared to the set of all possible, valid, updates.


\clearpage
\section{Preferences for \tworec: Illustrative Examples} \label{app:prefoverview}
In the following we provide illustrative examples for the most preferred hypothesis for the \tworec hypothesis class under different configurations of the current hypothesis $h_t$ and the next version space $\Hypotheses_{t+1}$.

In the figures below (\figref{fig:supp.preference.tworec.h1.h1} -- \figref{fig:illustration:h2}), the hypotheses $\hypothesis_t$ (left illustration) and $\hypothesis_{t+1}$ (right illustration) are represented by the orange rectangles. The green and blue cells represent positive and negative teaching examples, respectively.

\begin{itemize}
\item $h_t \in \Hypotheses^1$, $\Hypotheses_{t+1} \cap \Hypotheses^1 \neq \emptyset$. See \figref{fig:supp.preference.tworec.h1.h1}.
  \begin{figure*}[!h]
    \centering
    \begin{subfigure}[b]{0.18\textwidth}
      \includegraphics[width=1.0\textwidth]{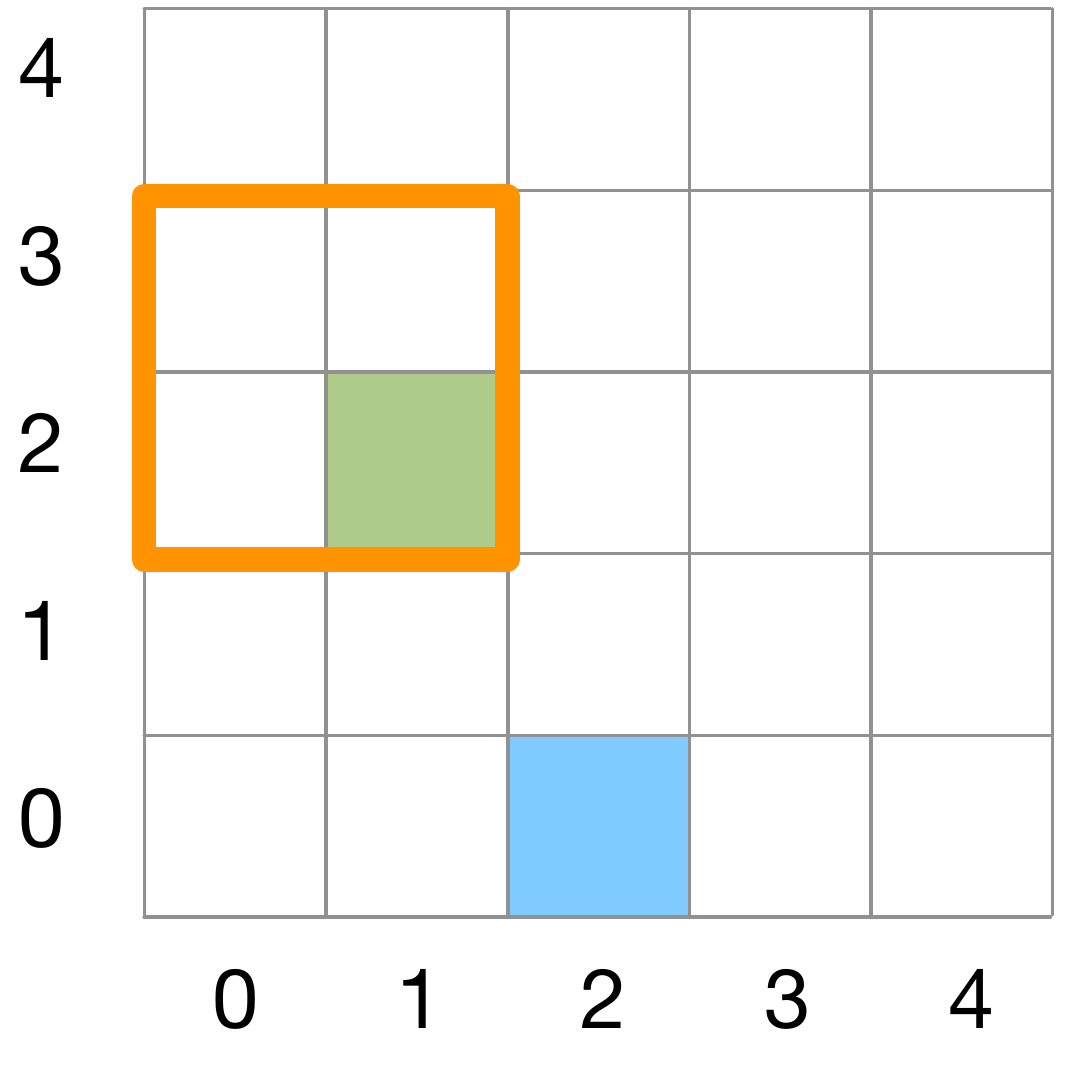}
      \label{fig:supp.pref.tworec.h1.h1.a}
    \end{subfigure}
    \quad
    \begin{subfigure}[b]{0.18\textwidth}
      \includegraphics[width=1.0\textwidth]{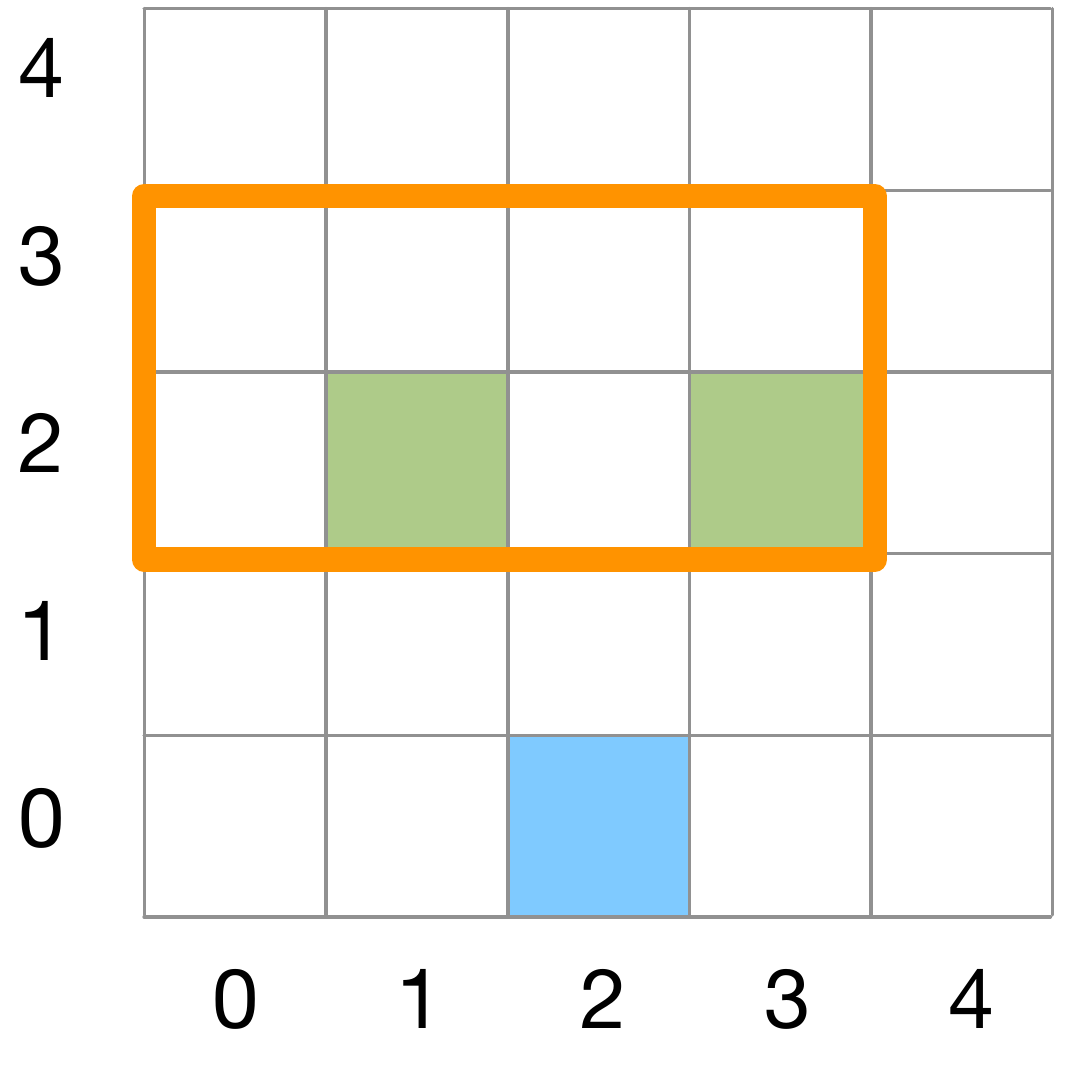}
      \label{fig:supp.pref.tworec.h1.h1.b}
    \end{subfigure}
    \caption{$h_t \in \Hypotheses^1$, $\Hypotheses_{t+1} \cap \Hypotheses^1 \neq \emptyset$: Update to one of the closest hypothesis in $\Hypotheses_{t+1} \cap \Hypotheses^1$ (here, distance is defined by the number of edge movements).}
    \label{fig:supp.preference.tworec.h1.h1}
  \end{figure*}
  
\item $h_t \in \Hypotheses^1$, $\Hypotheses_{t+1} \cap \Hypotheses^1 = \emptyset$. See \figref{fig:supp.preference.tworec.h1.h2.singleton} and \figref{fig:supp.preference.tworec.h1.h2.redraw}.
  \begin{figure*}[!h]
    \centering
    \begin{subfigure}{.45\textwidth}
      \centering
      \includegraphics[width=.4\textwidth]{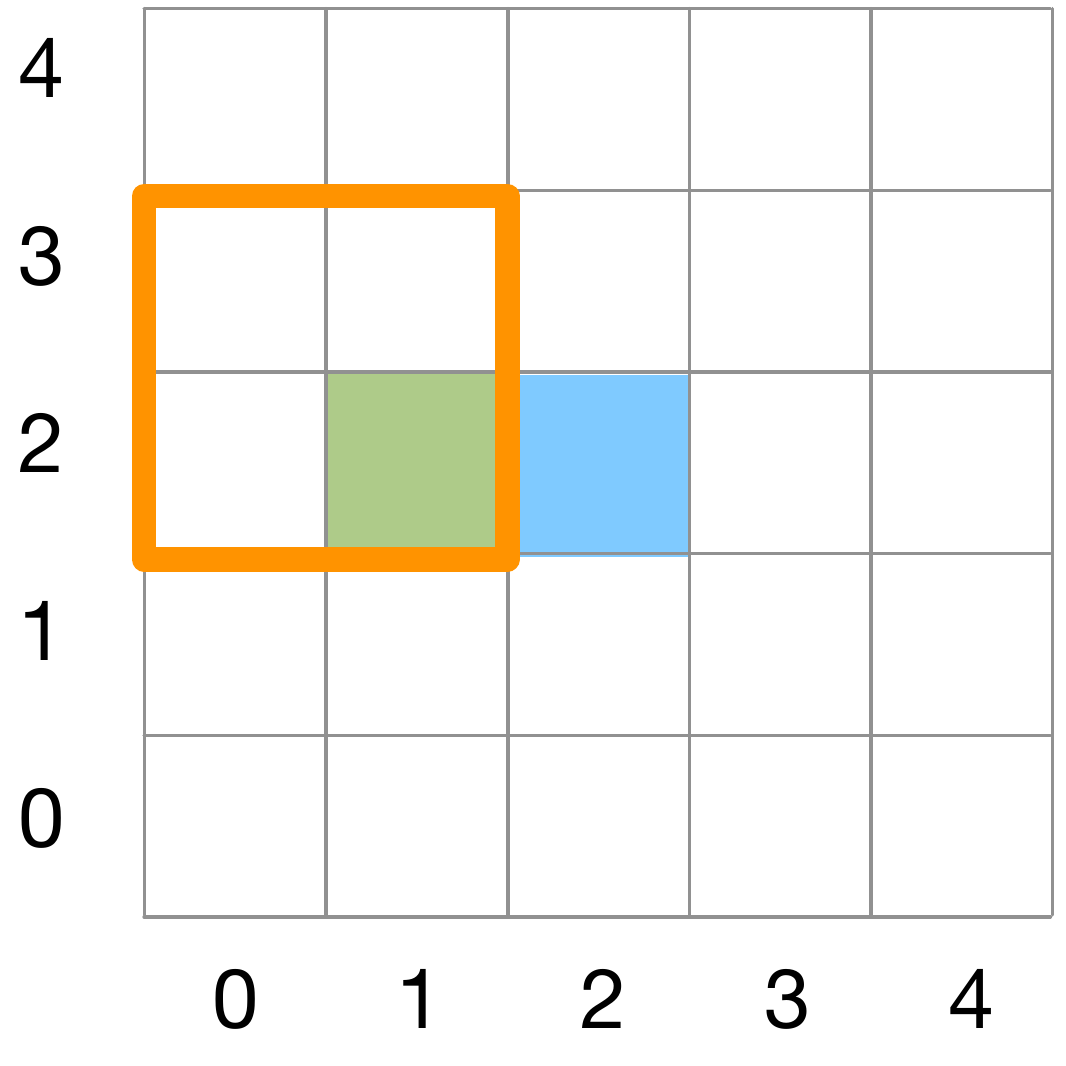}
      \quad
      \includegraphics[width=.4\textwidth]{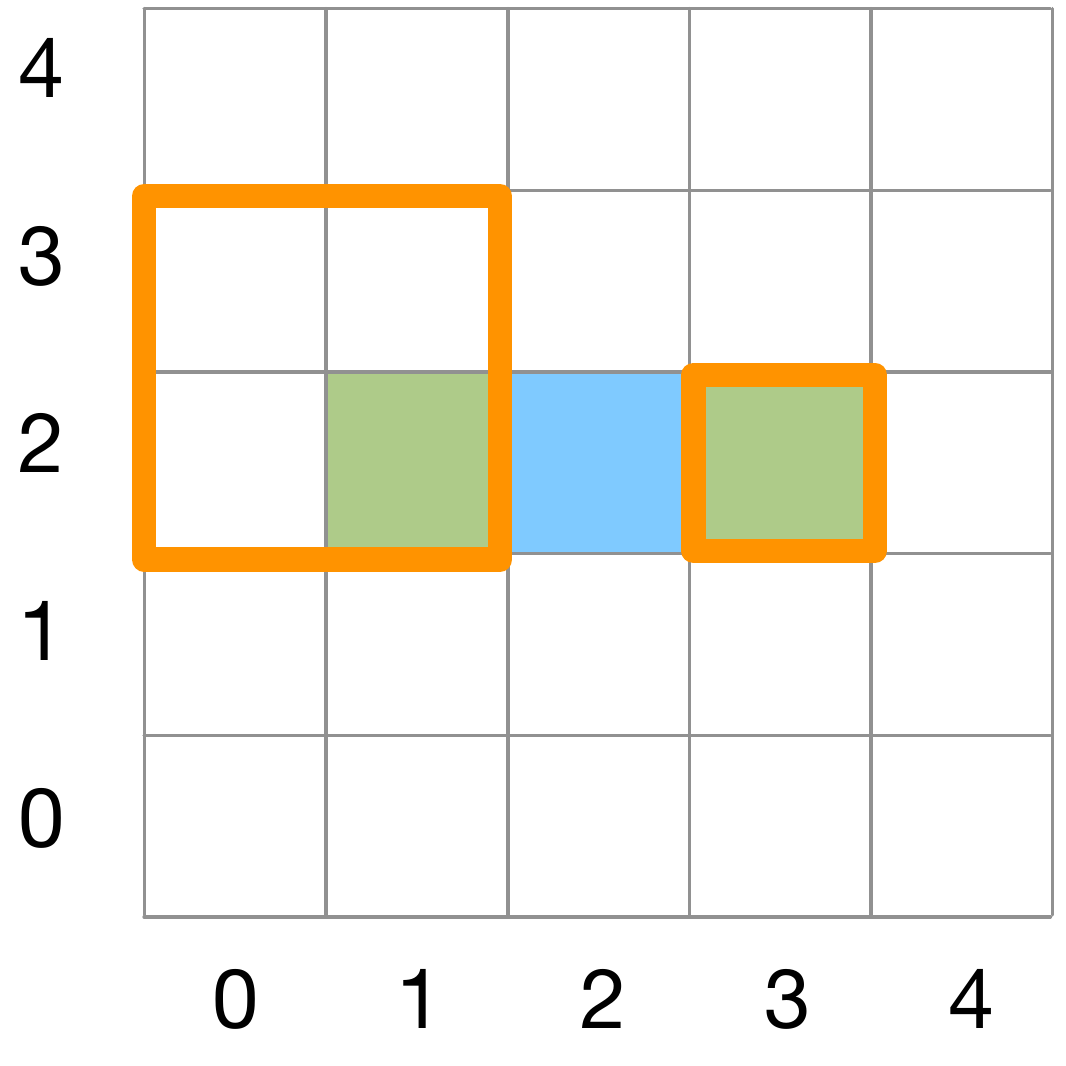}
      \caption{}
      \label{fig:supp.preference.tworec.h1.h2.singleton}
    \end{subfigure}
    \qquad
    \begin{subfigure}{.45\textwidth}
      \centering
      \includegraphics[width=.4\textwidth]{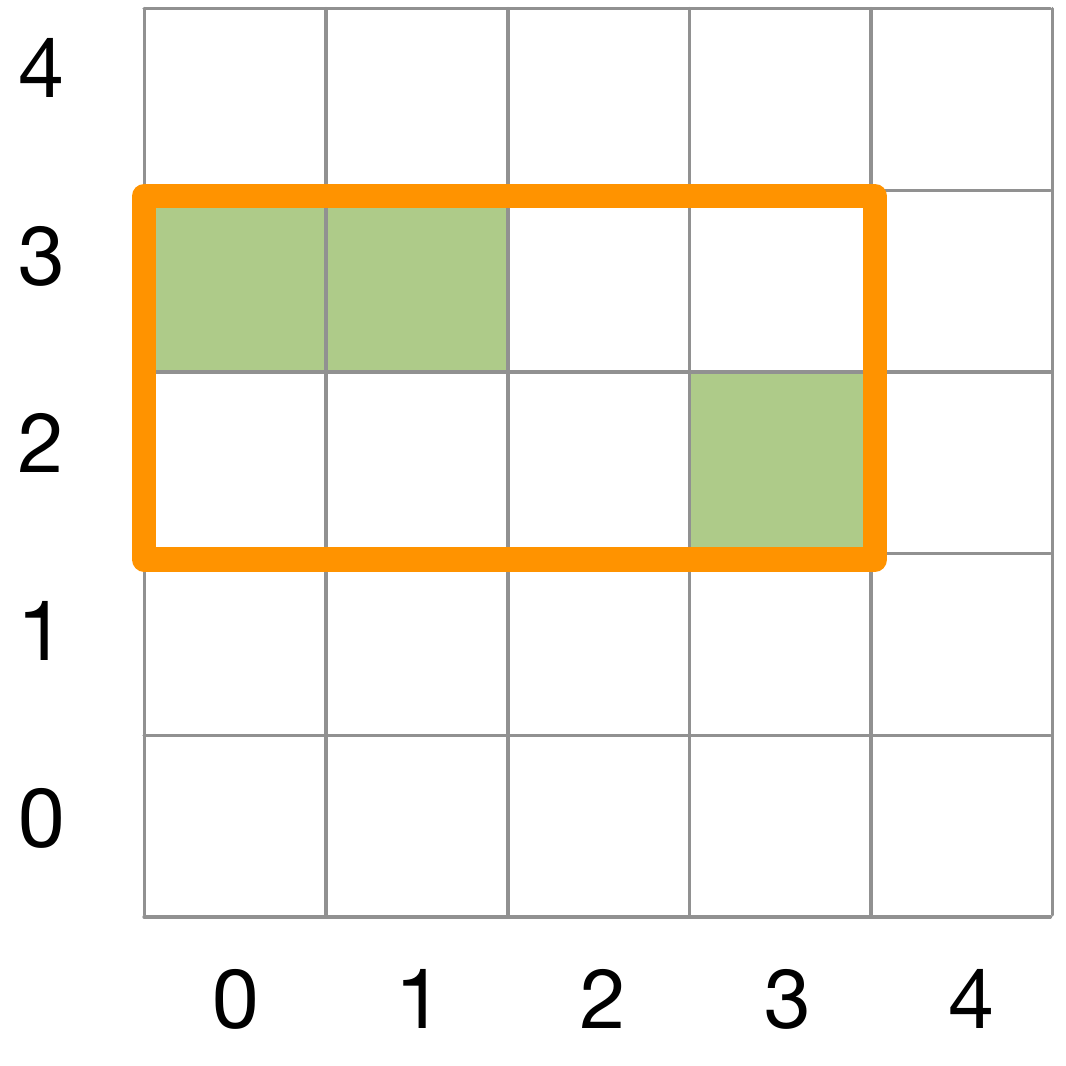}
      \quad
      \includegraphics[width=.4\textwidth]{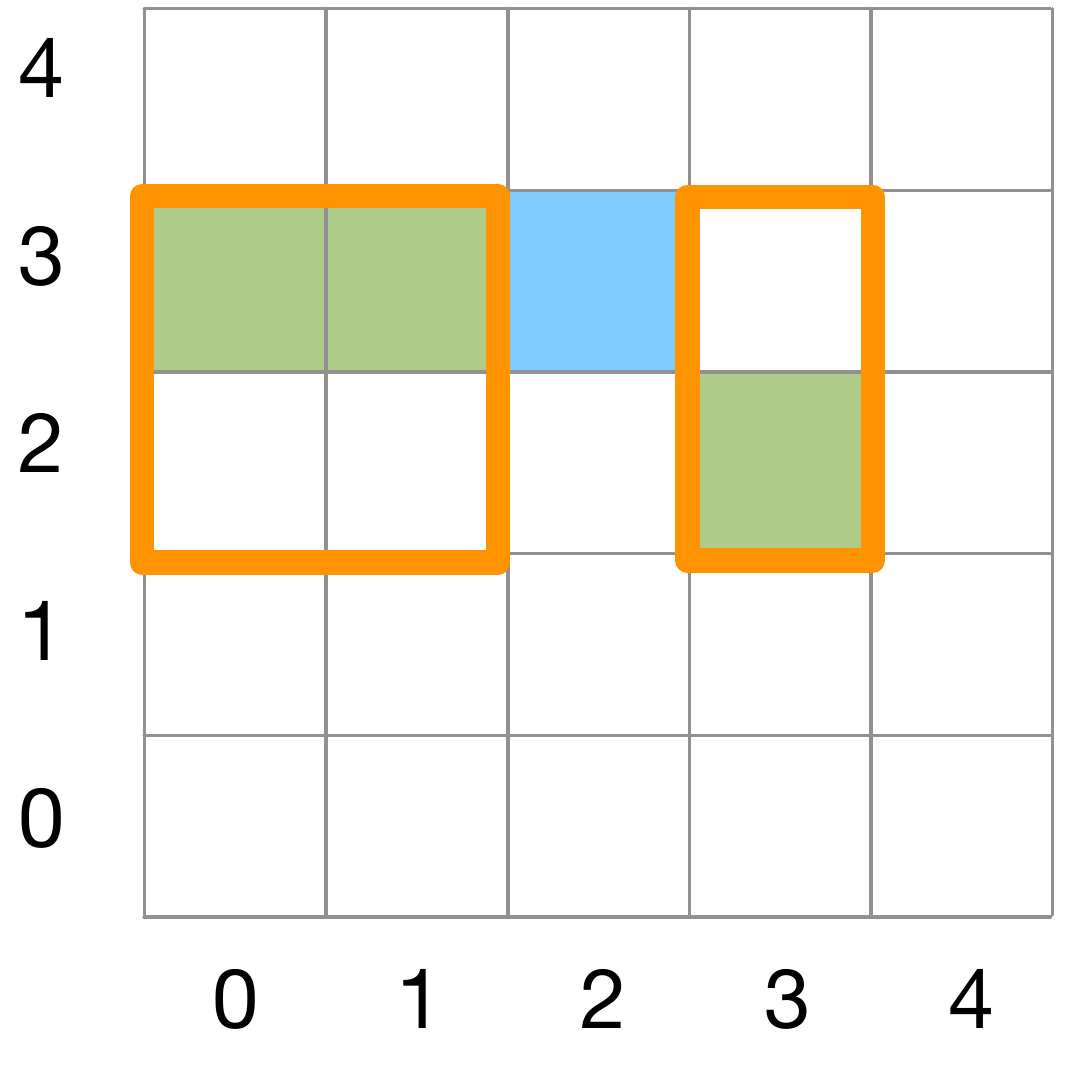}
      \caption{}
      \label{fig:supp.preference.tworec.h1.h2.redraw}
    \end{subfigure}
    \caption{$h_t \in \Hypotheses^1$, $\Hypotheses_{t+1} \cap \Hypotheses^1 = \emptyset$: (a) Add one singleton rectangle. (b) Redraw two rectangles.
    }
  \end{figure*}
\item $h_t \in \Hypotheses^2$, $\Hypotheses_{t+1} \cap \Hypotheses^2 = \emptyset$. See \figref{fig:supp.preference.tworec.h2.h1.exhaust.singleton} and \figref{fig:supp.preference.tworec.h2.h1.exhaust.merge}.
  \begin{figure*}[!h]
    \centering
    \begin{subfigure}{.45\textwidth}
      \centering
      \includegraphics[width=.4\textwidth]{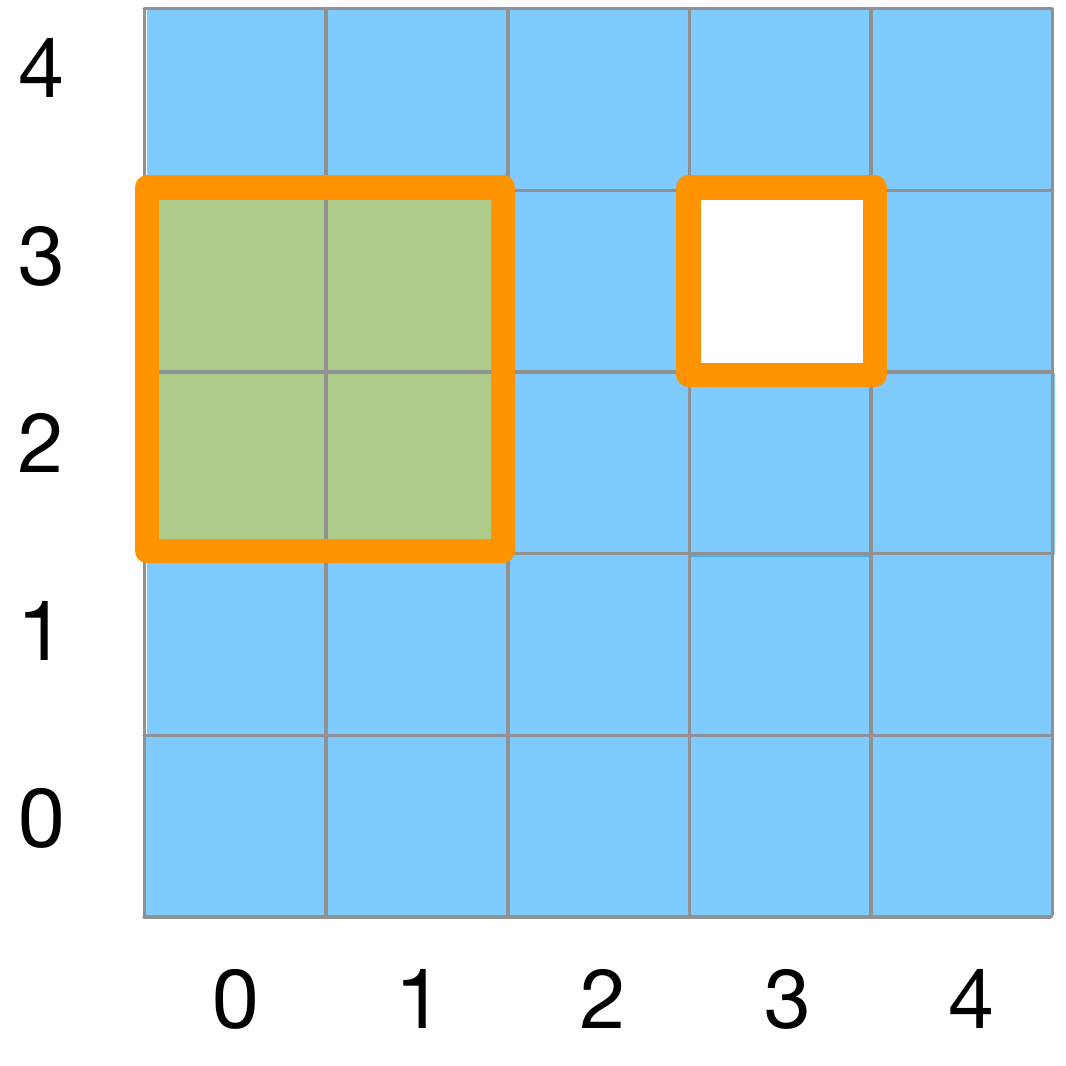}
      \quad
      \includegraphics[width=.4\textwidth]{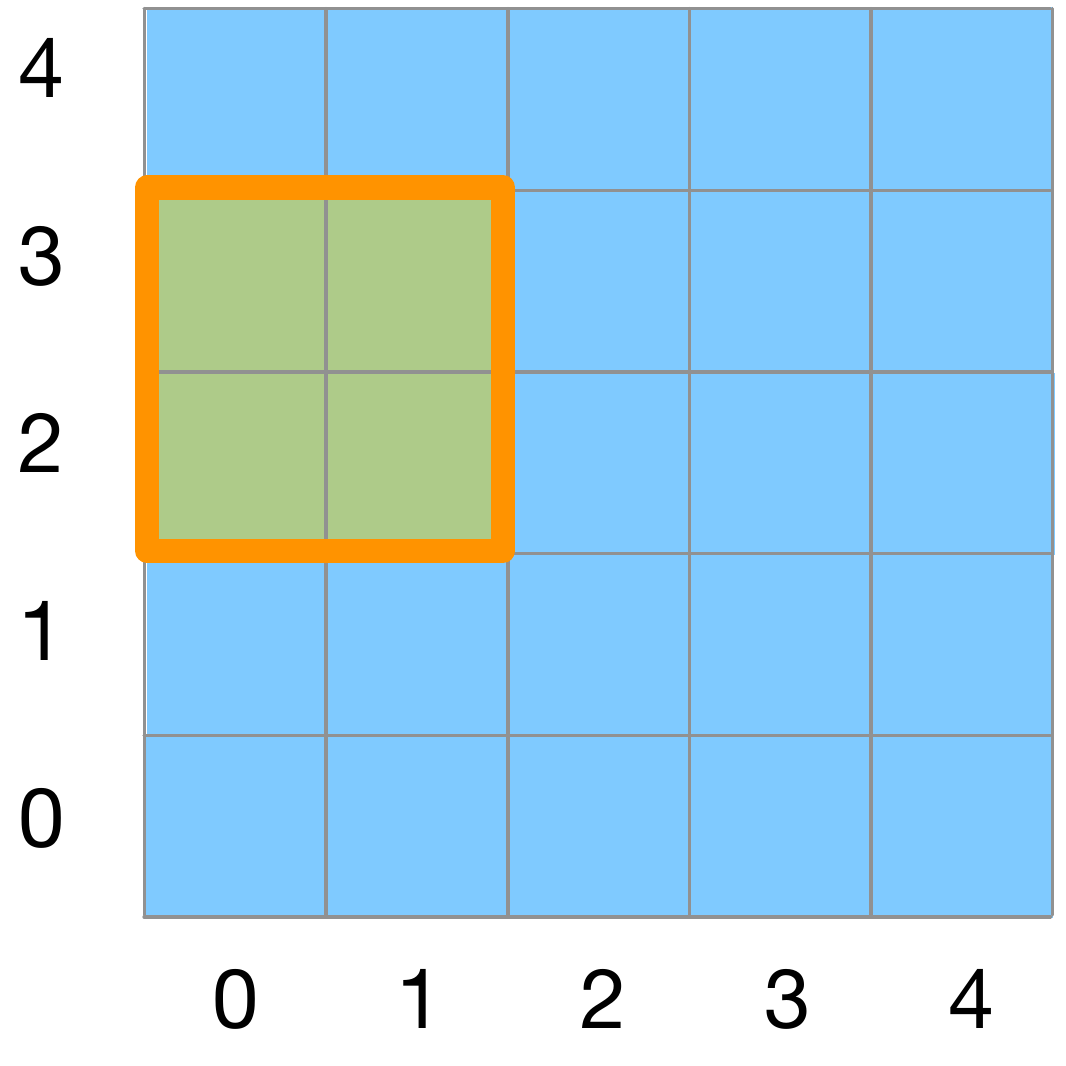}
      \caption{}
      \label{fig:supp.preference.tworec.h2.h1.exhaust.singleton}
    \end{subfigure}
    \qquad
    \begin{subfigure}{.45\textwidth}
      \centering
      \includegraphics[width=.4\textwidth]{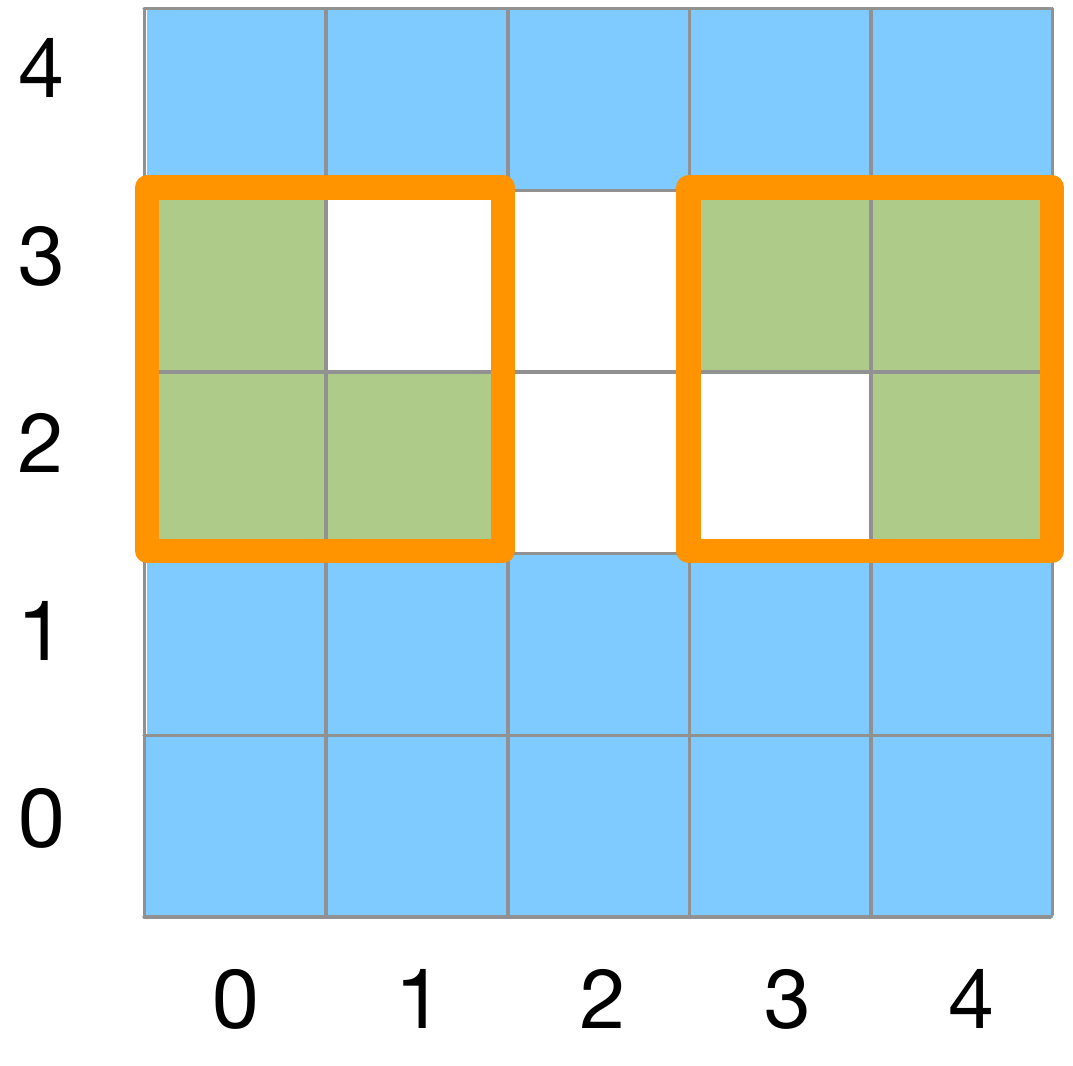}
      \quad
      \includegraphics[width=.4\textwidth]{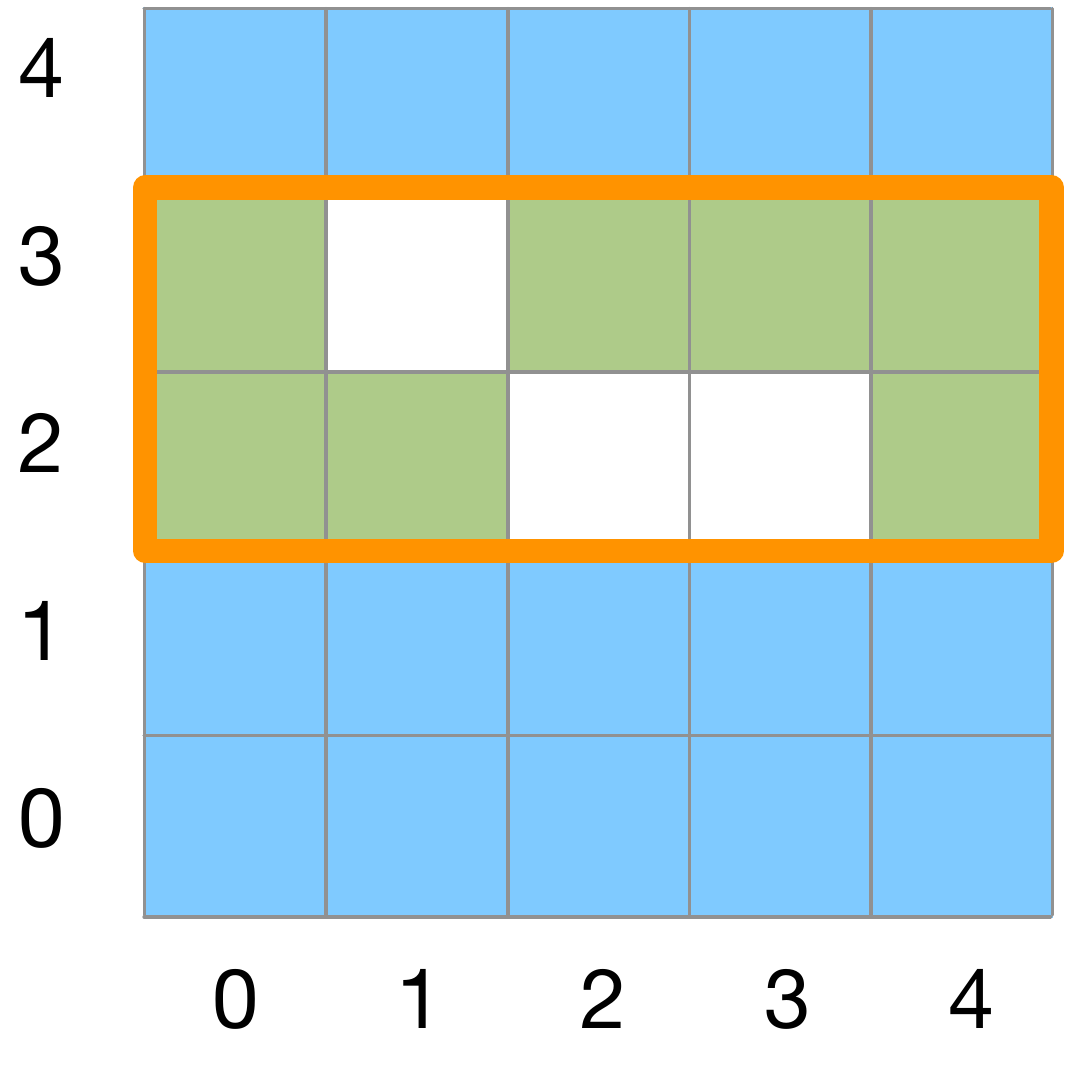}
      \caption{}
      \label{fig:supp.preference.tworec.h2.h1.exhaust.merge}
    \end{subfigure}
    \caption{$h_t \in \Hypotheses^2$, $\Hypotheses_{t+1} \cap \Hypotheses^2 = \emptyset$: (a) Eliminate the singleton rectangle. (b) Merge two rectangles to a single rectangle.  
    }
  \end{figure*}
  \clearpage
\item $h_t \in \Hypotheses^2$, $\Hypotheses_{t+1} \cap \Hypotheses^2 \neq \emptyset$. See \figref{fig:supp.preference.tworec.h2.h2}, \figref{fig:supp.preference.tworec.h2.h1.shortcut.singleton} and \figref{fig:supp.preference.tworec.h2.h1.shortcut.merge}.
  \begin{figure*}[!h]
    \centering
    \begin{subfigure}{.45\textwidth}
      \centering
      \includegraphics[width=.4\textwidth]{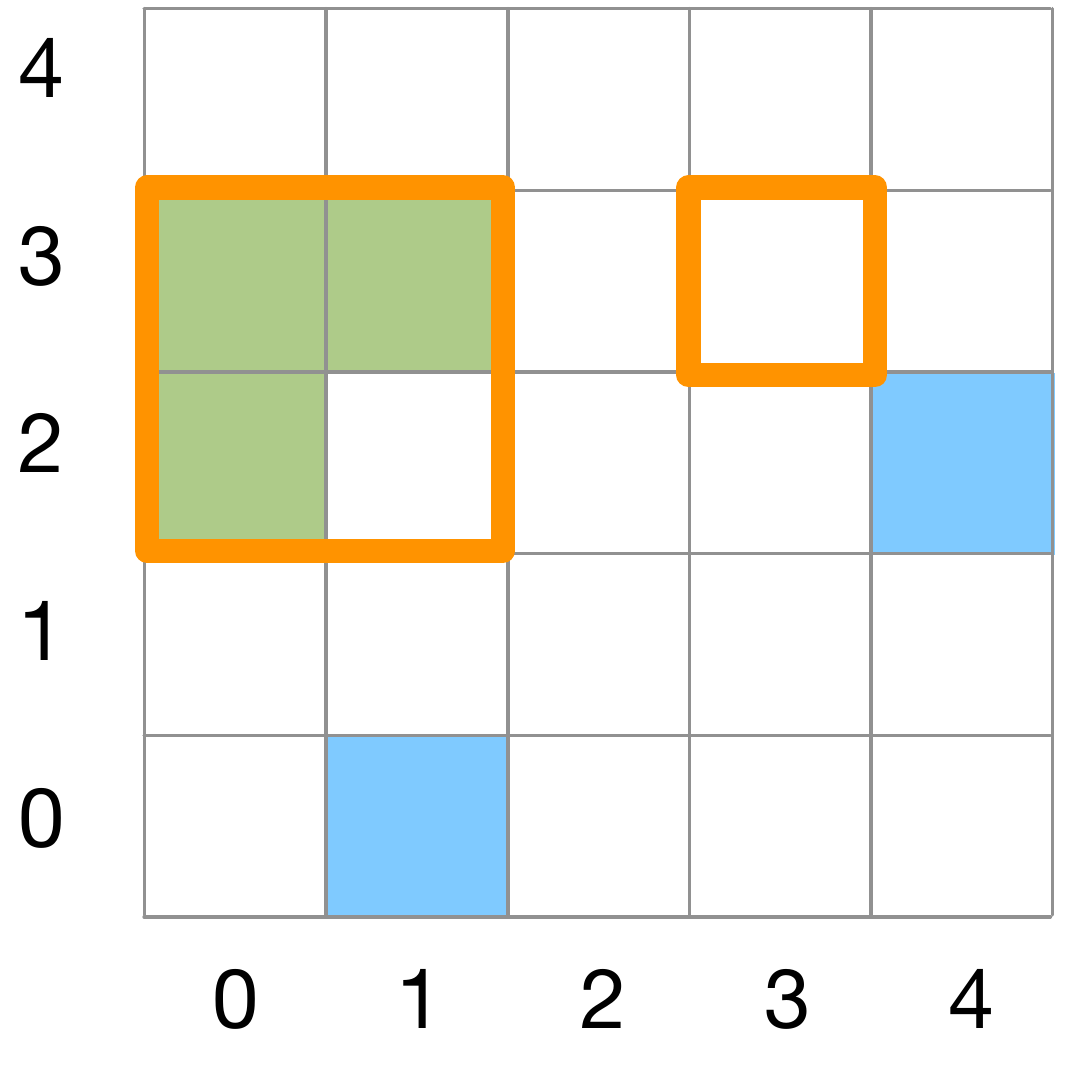}
      \quad
      \includegraphics[width=.4\textwidth]{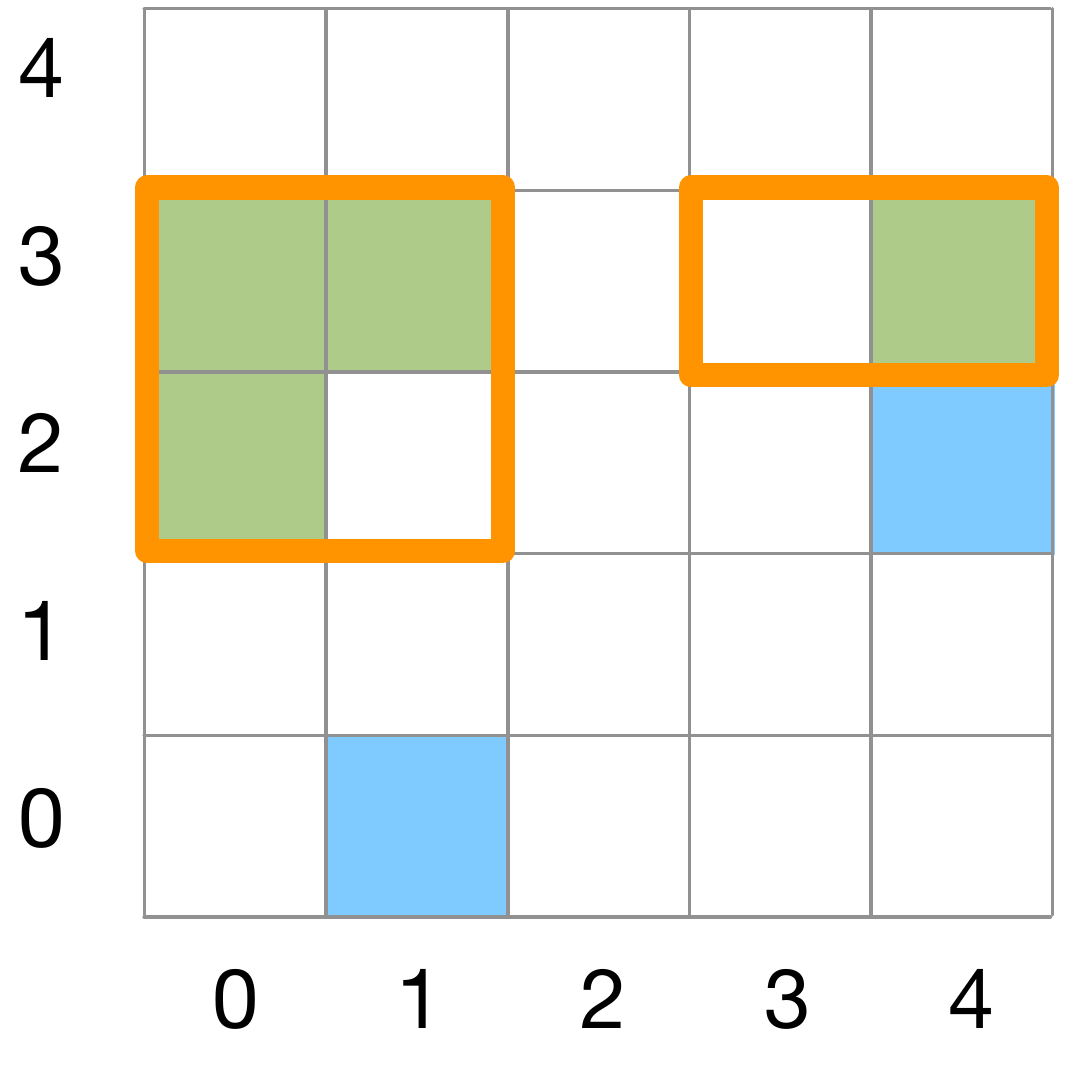}
      \caption{}
      \label{fig:supp.preference.tworec.h2.h2}
    \end{subfigure}
    \\
    \vspace{10pt}
    \begin{subfigure}{.45\textwidth}
      \centering
      \includegraphics[width=.4\textwidth]{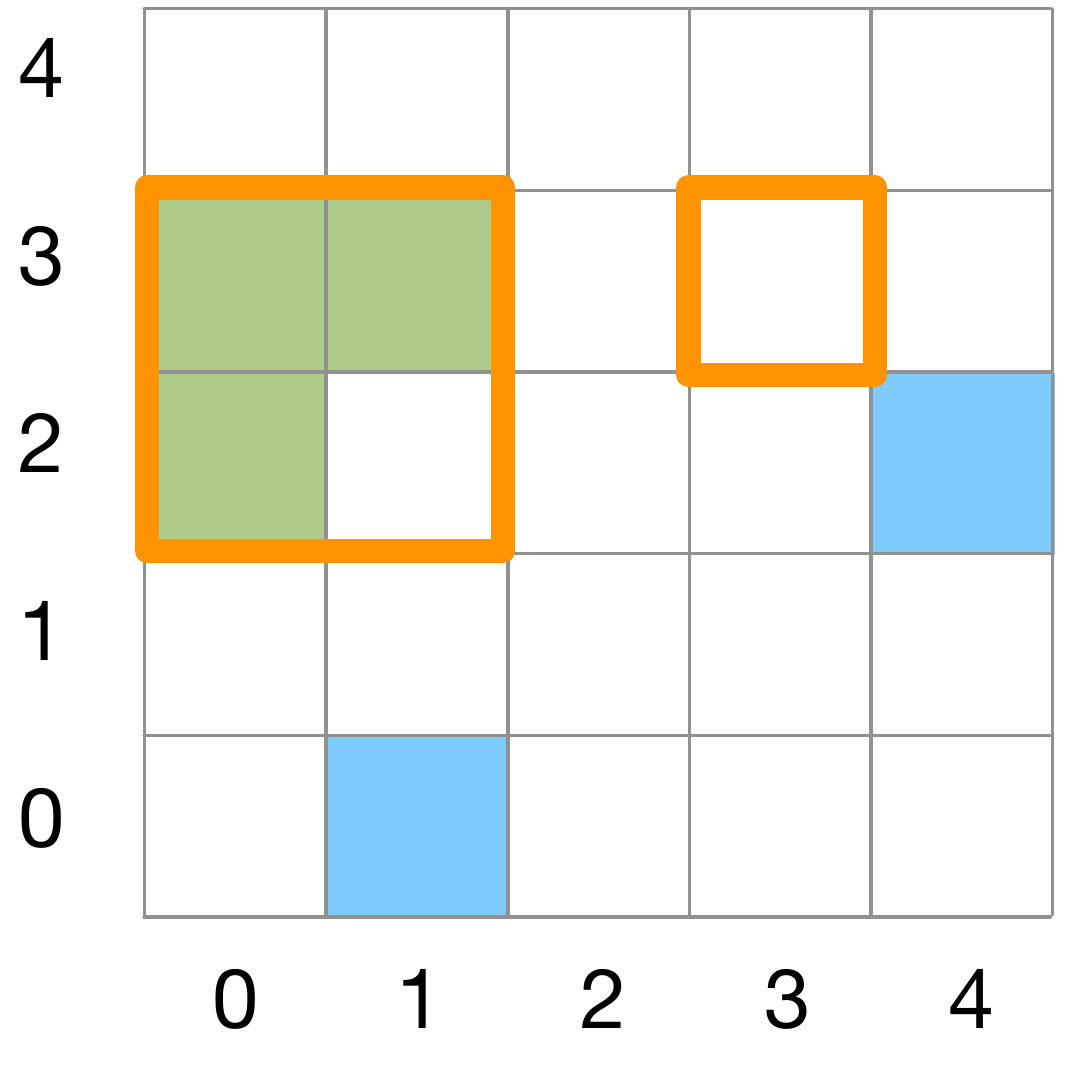}
      \quad
      \includegraphics[width=.4\textwidth]{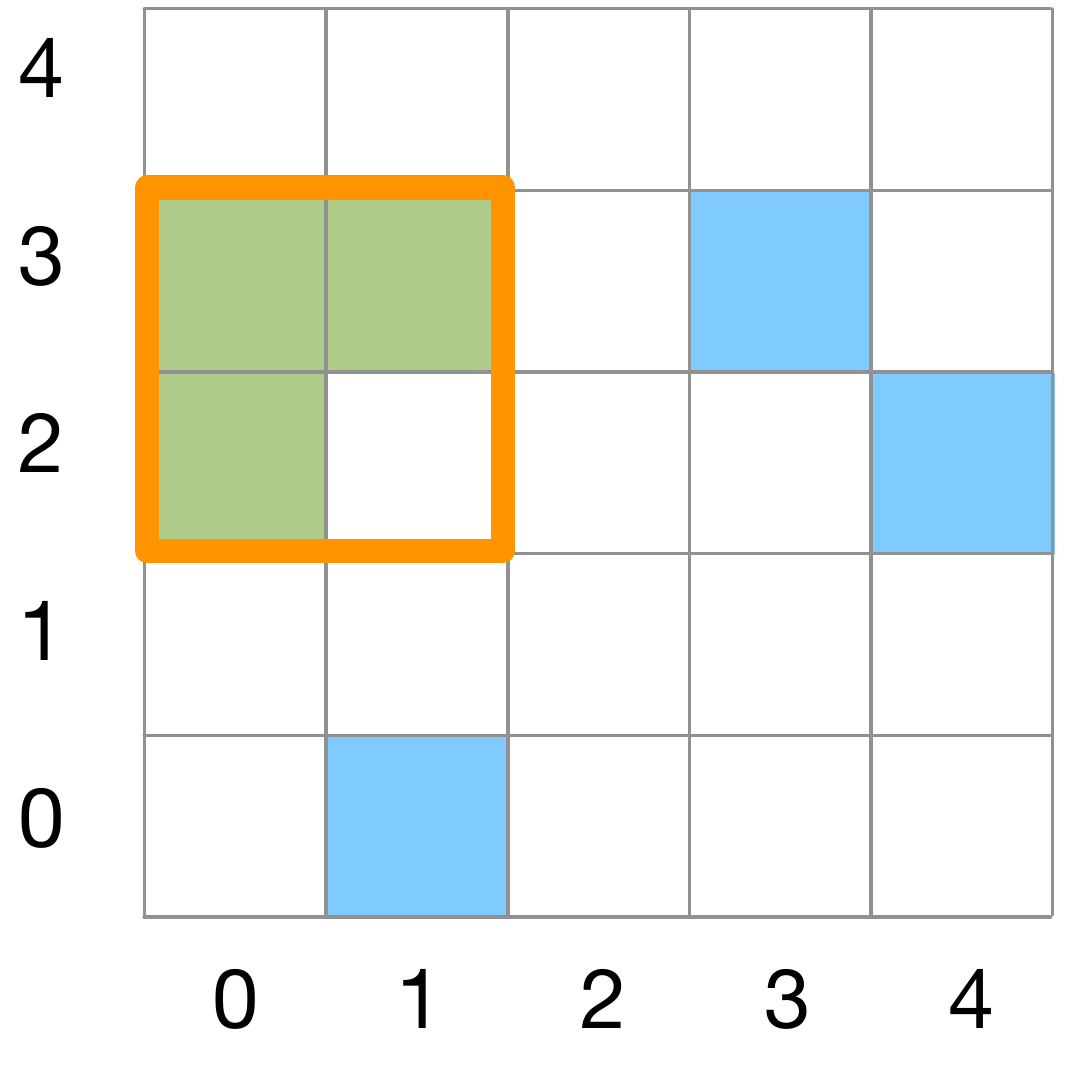}
      \qquad
      \caption{}
      \label{fig:supp.preference.tworec.h2.h1.shortcut.singleton}
    \end{subfigure}
    \\
    \vspace{10pt}
    \begin{subfigure}{.45\textwidth}
      \centering
      \includegraphics[width=.4\textwidth]{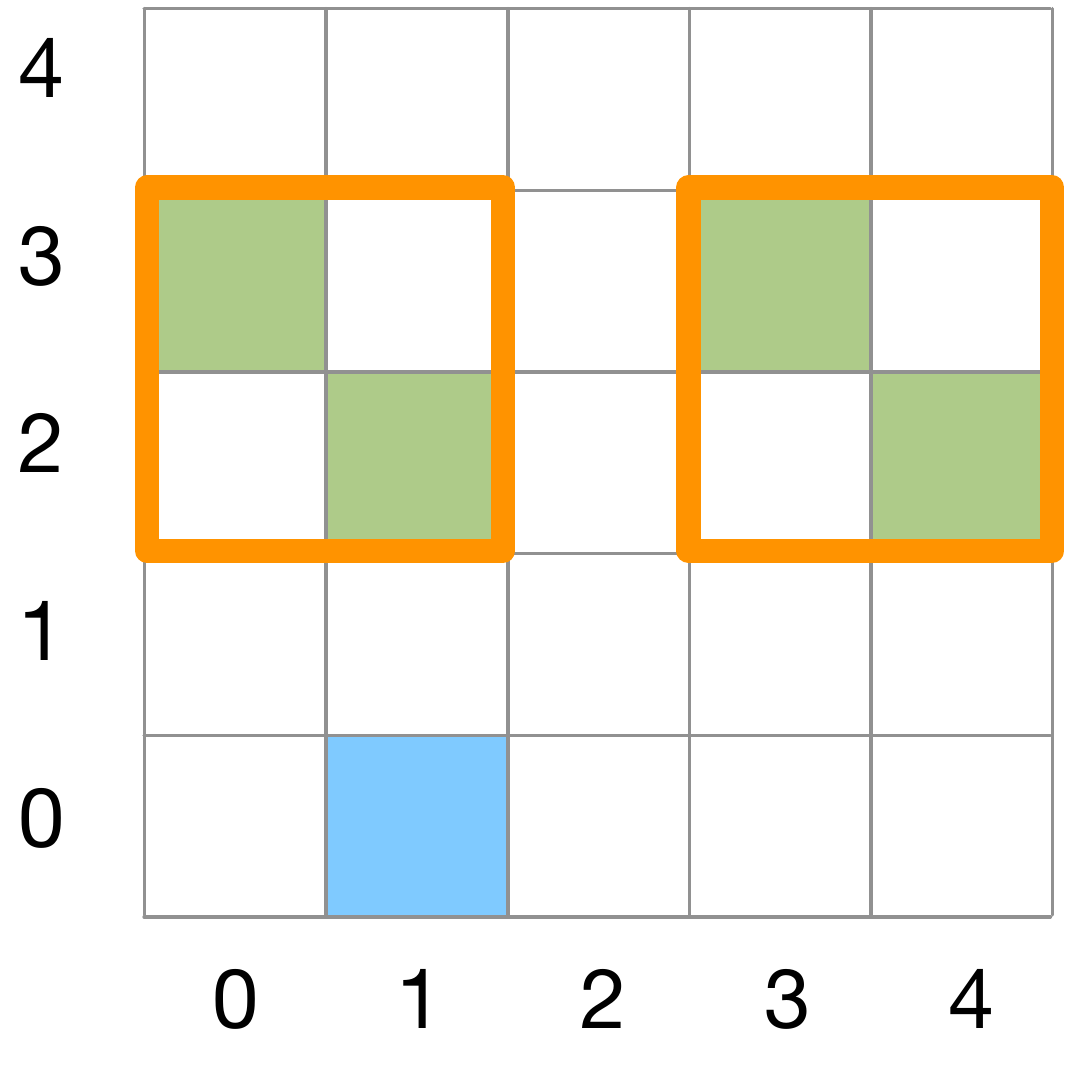}
      \quad
      \includegraphics[width=.4\textwidth]{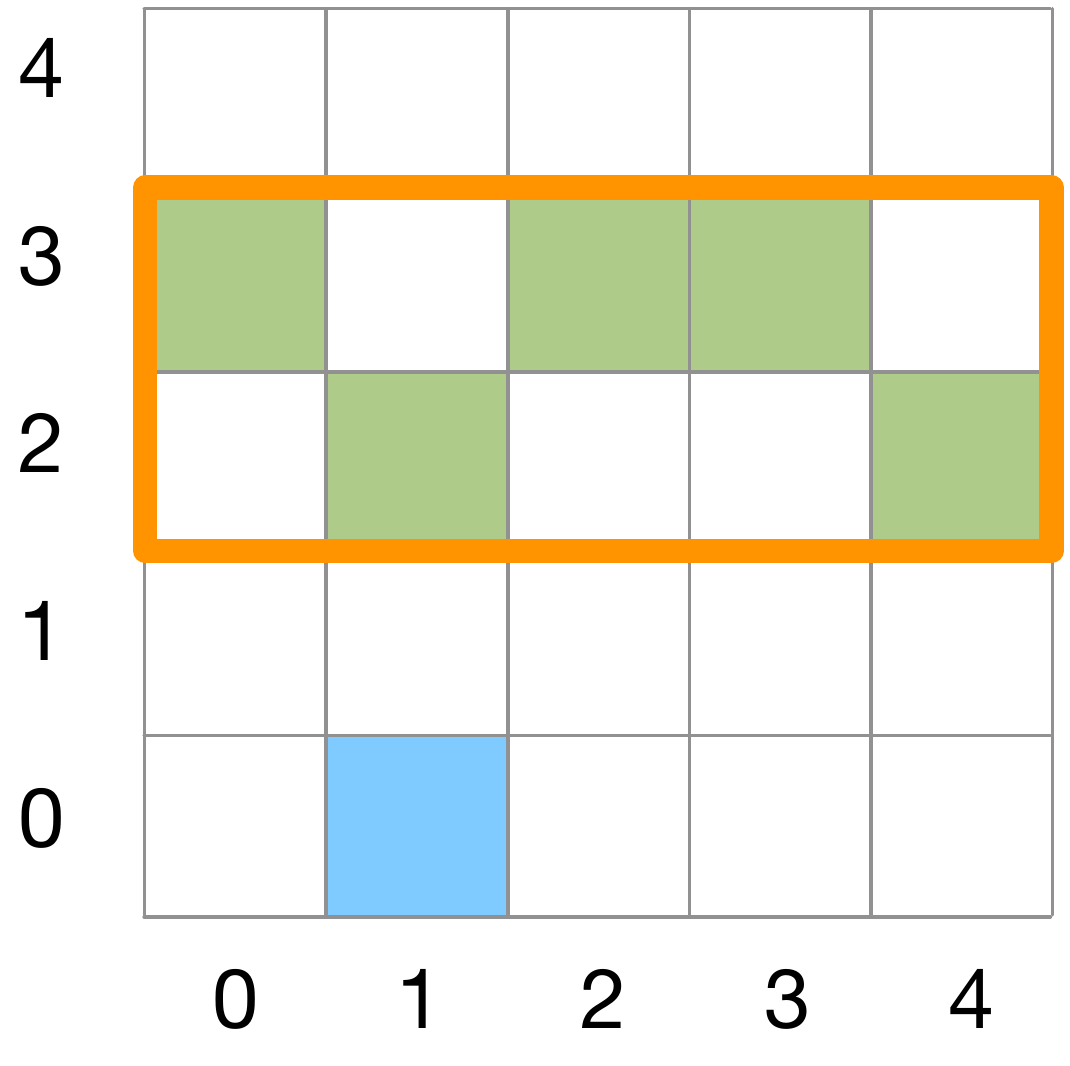}
      \caption{}
      \label{fig:supp.preference.tworec.h2.h1.shortcut.merge}
    \end{subfigure}
    \caption{$h_t \in \Hypotheses^2$, $\Hypotheses_{t+1} \cap \Hypotheses^2 \neq \emptyset$: (a) Move to the closest hypothesis in $\Hypotheses_{t+1} \cap \Hypotheses^2$ (here, distance is defined by the number of edge movements). (b) Eliminate the singleton rectangle (also, see \secref{app:tworec:structure} C-2). (c) Merge two rectangles to a single rectangle (also, see \secref{app:tworec:structure} C-2).
    }
    \label{fig:illustration:h2}
  \end{figure*}
\end{itemize}



\clearpage
\section{Proof of \thmref{thm:adaptivity}}\label{app:proof:adaptivity}

In this section, we provide the proof of \thmref{thm:adaptivity}. In order to prove the theorem, we will construct adaptive teaching sequences for the \tworec and \lattice classes, and provide lower bounds on the optimal non-adaptive teaching algorithms for both classes.

We first prove the following lemma for the \tworec class.

\begin{lemma}\label{lm:tworec}
  Consider the case where $\hstar \in \Hypotheses^1$ represents a single rectangle $r_{1}$, and the learner's initial hypothesis $\hypothesis_0 \in \Hypotheses^2$ consists of two rectangles $(r_{1}, r_{2})$ (cf. \figref{fig:illustration:adaptivetworec:proof:init}). We assume that the learner prefers to move the smallest number of edges when transitioning within a subclass of \tworec. There exists an adaptive teacher that requires at most $\Theta(\log|r_{2}|)$ examples to teach $h^*$, while any non-adaptive teacher requires $\Theta(|r_{2}|)$ examples in the worst case.
\end{lemma}

\begin{figure*}[!h]
  \centering
  \begin{subfigure}{\textwidth}
    \centering
    \includegraphics[width=.25\textwidth]{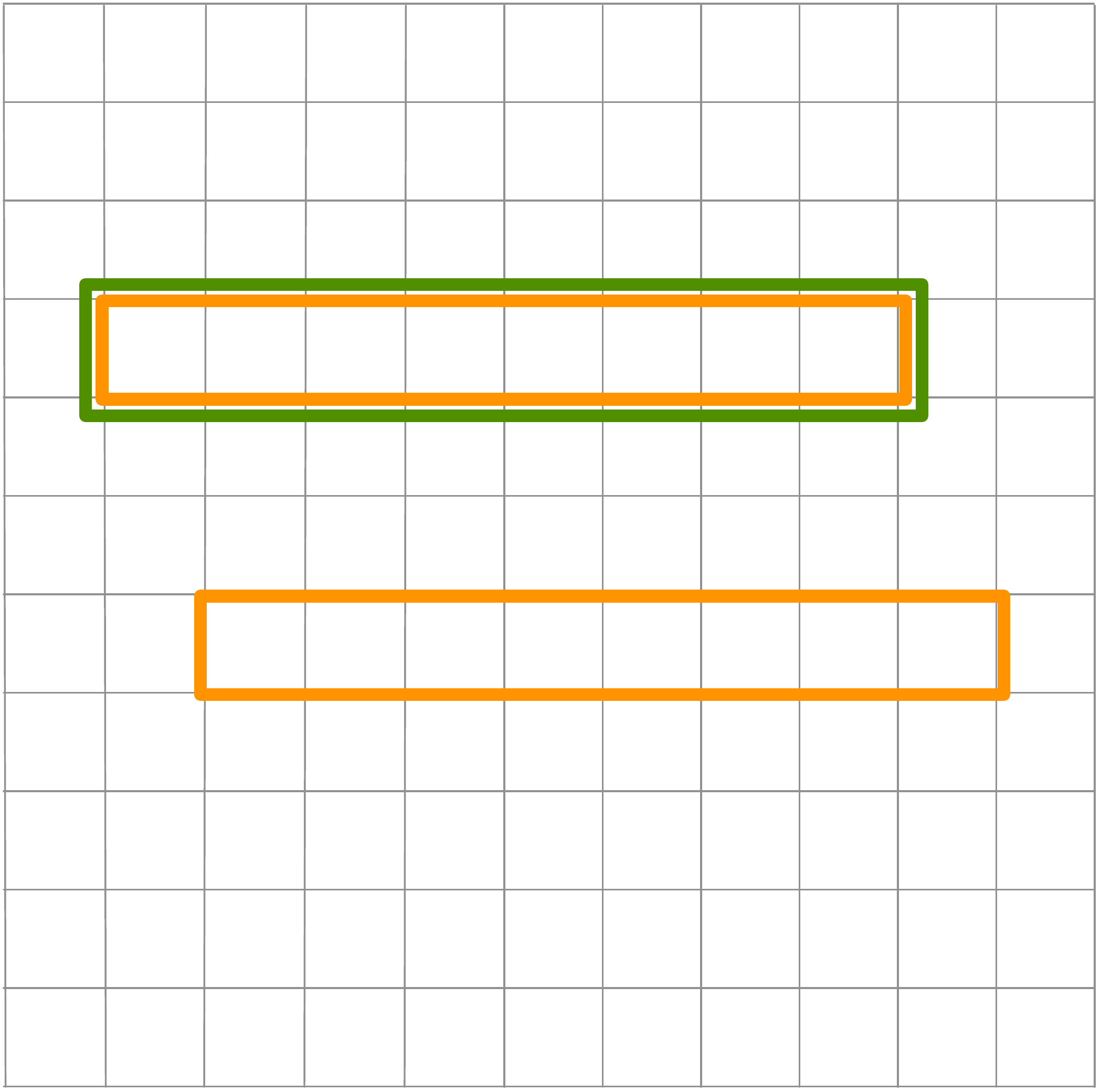}
    \caption{Initial configuration. $\hinit=(r_1, r_2)$ is represented by the orange rectangles, and $\hstar=r_1$ is represented by the green rectangle.}
    \label{fig:illustration:adaptivetworec:proof:init}
  \end{subfigure}
  \\
  \vspace{10pt}
  \begin{subfigure}{\textwidth}
    \centering
    \includegraphics[width=.25\textwidth]{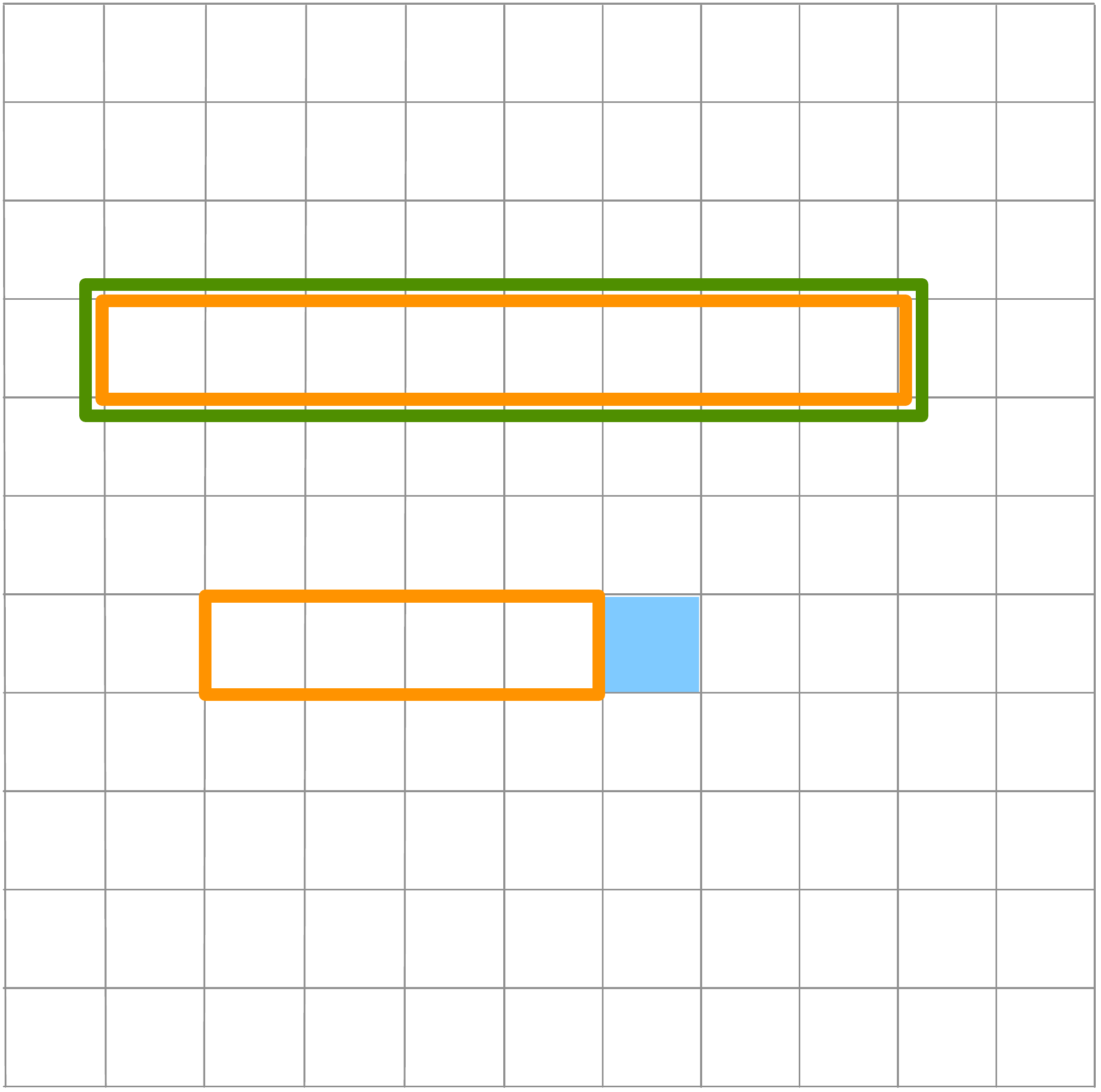}
    \qquad
    \includegraphics[width=.25\textwidth]{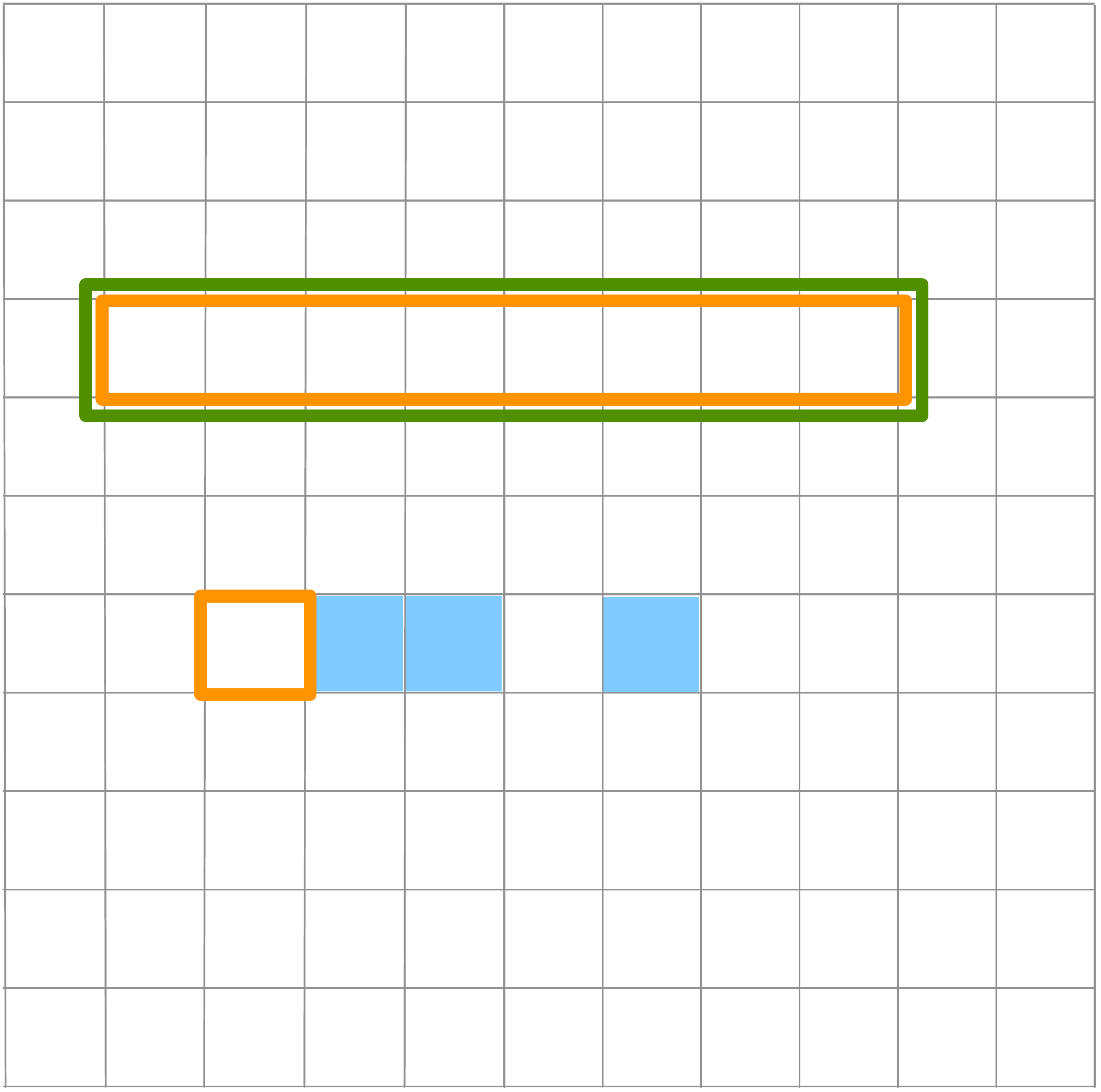}
    \qquad
    \includegraphics[width=.25\textwidth]{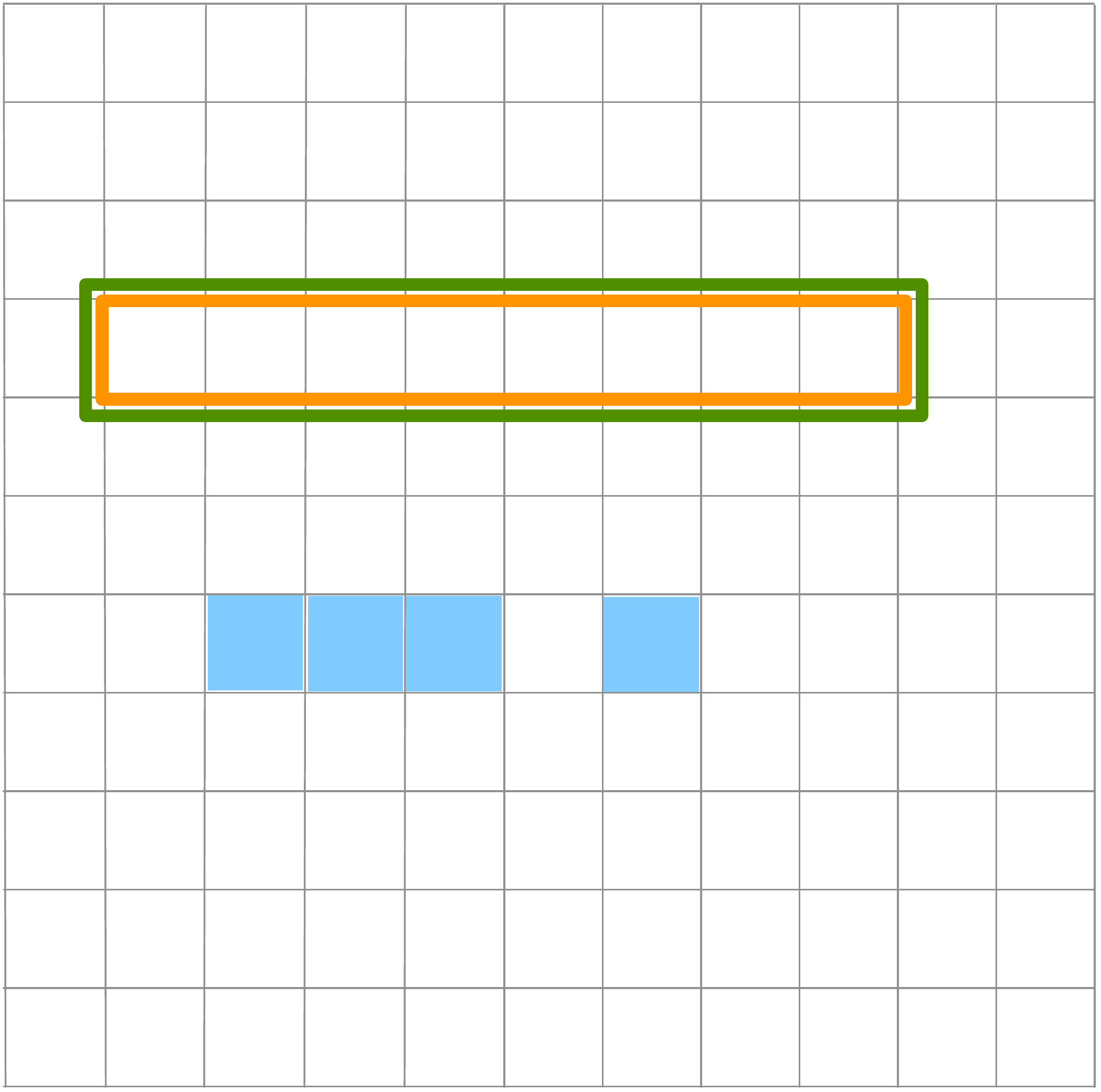}
    \caption{Illustration of the adaptive teaching sequence. It follows a binary search procedure to remove $r_2$.}
    \label{fig:illustration:adar:proof:thm2}
  \end{subfigure}
  \\
  \vspace{10pt}
  \begin{subfigure}{\textwidth}
    \centering
    \includegraphics[width=.25\textwidth]{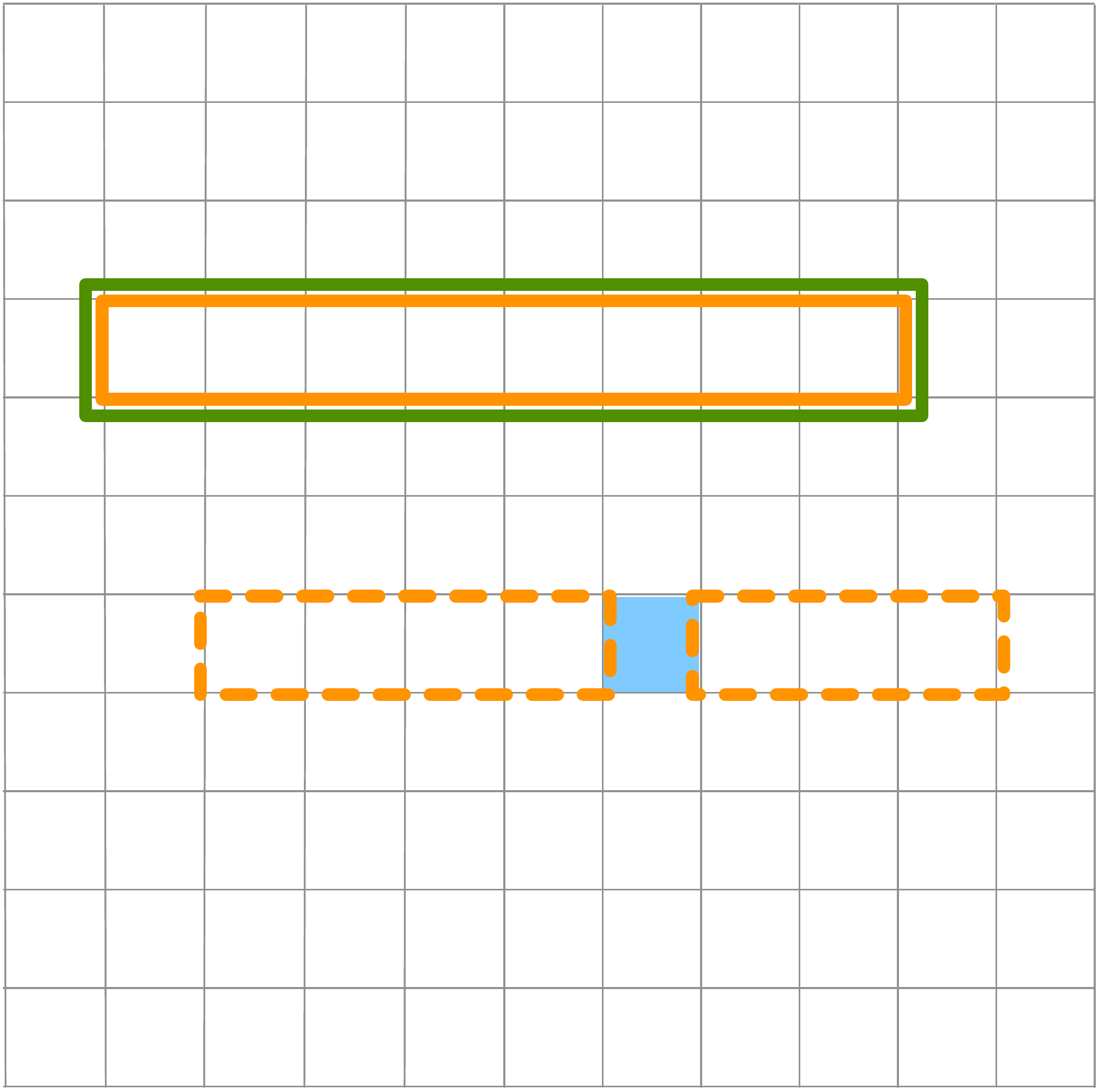}
    \qquad
    \includegraphics[width=.25\textwidth]{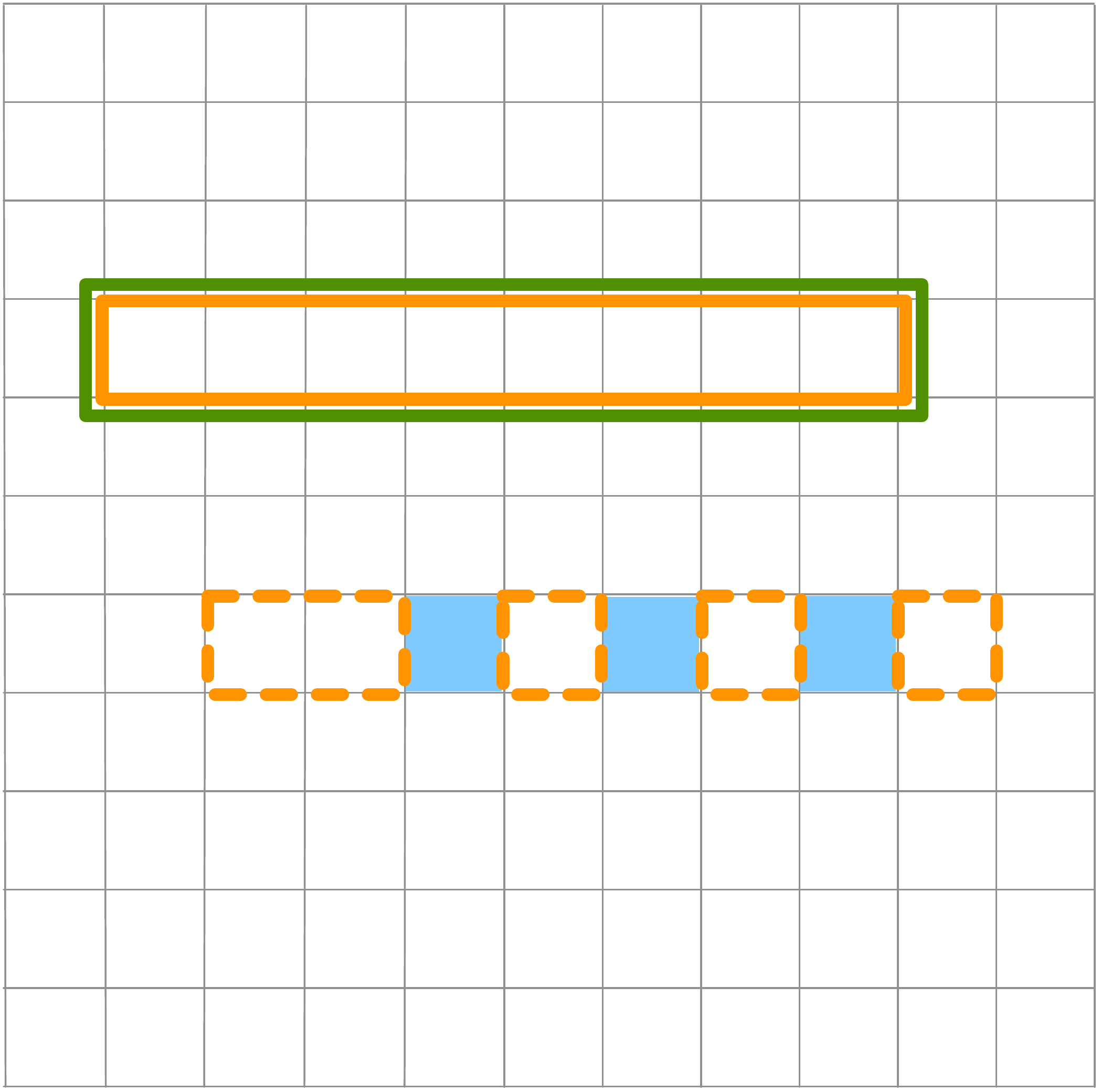}
    \qquad
    \includegraphics[width=.25\textwidth]{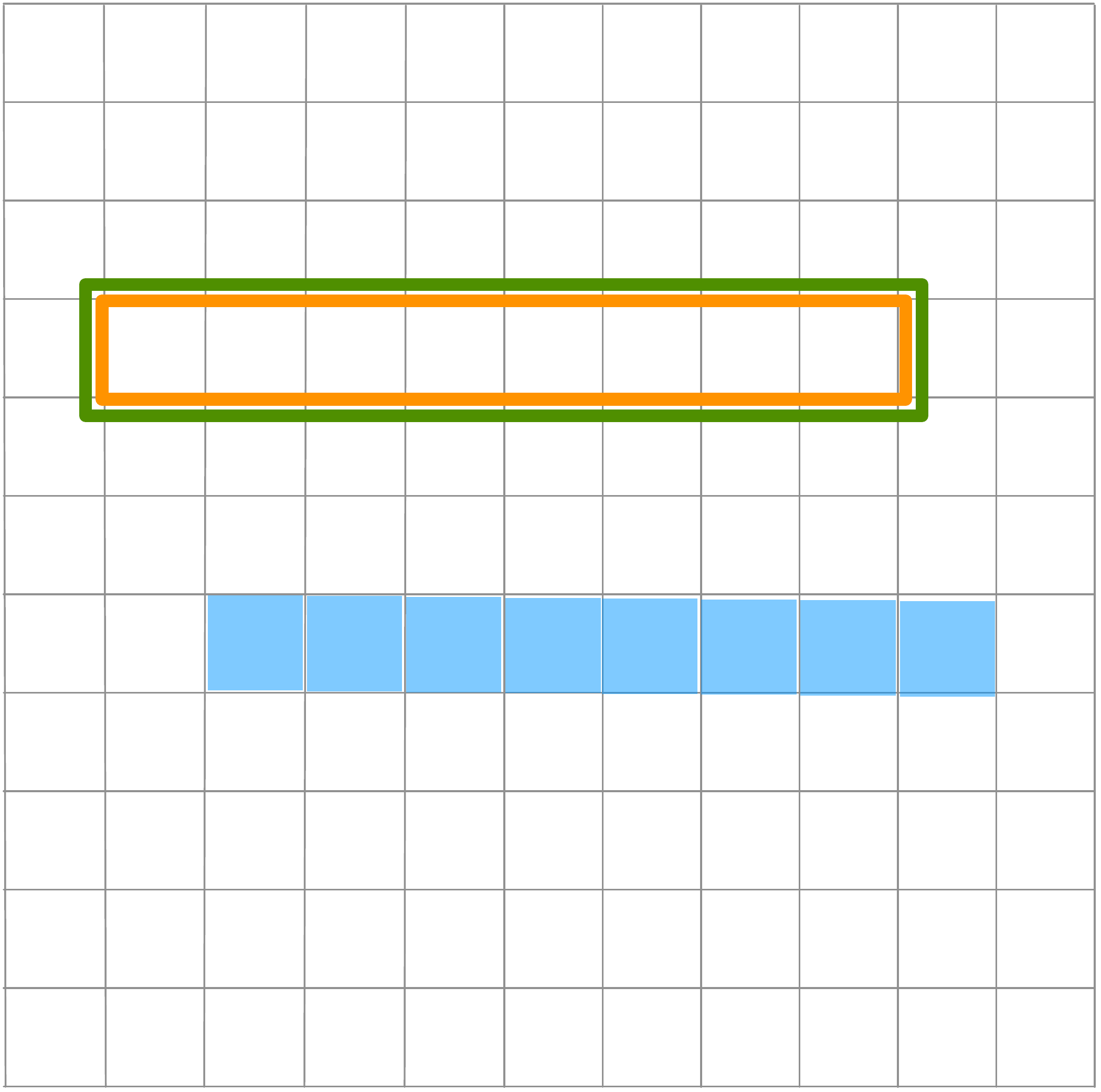}
    \caption{Non-adaptive teaching sequence. The non-adaptive algorithm does not observe the learner's intermediate hypotheses. Here, the dashed orange rectangles represent some of the possible locations of the second rectangles, when the learner is provided with the teaching examples represented by the solid blue grid cells. In the worst case, the teacher has to provide all the negative examples inside $r_2$ to remove it.}
    \label{fig:illustration:nonr:proof:thm2}
  \end{subfigure}
  \caption{Illustration for the \tworec class. 
  }
  \label{fig:illustration:adaptivetworec:proof}
\end{figure*}

\begin{proof}
  Note that $\hinit=(r_1, r_2)$ consists of a rectangle $\hstar=r_1$. In this case, we will focus on designing an adaptive teaching strategy that eliminates the second rectangle $r_2$ from $\hinit$.

  For simplicity let us consider the configuration as illustrated in \figref{fig:illustration:adaptivetworec:proof:init}, where $|r_1|=|r_2|$, and $r_2$ contains grid cells that lie in a line. We consider a binary search procedure under the adaptive setting: at every iteration, the teacher picks the grid cell that is the closest to the center of the second rectangle as the teaching example (cf. \figref{fig:illustration:adar:proof:thm2}). Such procedure requires $\Theta(\log|r_{2}|)$ teaching examples.

  Under the non-adaptive setting, since the teacher does not observe the learner's intermediate hypotheses, it has to perform a linear search (\figref{fig:illustration:nonr:proof:thm2}), which requires $\Theta(|r_{2}|)$ examples.
\end{proof}

We further prove the following lemma for teaching the \lattice class.
\begin{lemma}\label{lm:lattice}
  Let $a, b$ be integers in $(0,n-1)$, and assume $a<b$. Consider a 2-dimensional integer lattice of length $n$, with $\hinit=(a, a)$ and $\hstar=(b, b)$. We assume that the learner prefers close-by hypotheses measured by the Manhattan distance, and in the case of ties the learner prefers hypotheses with larger coordinates (cf. \figref{fig:lattice:illustration:adal:proof}). There exists an adaptive teacher that requires at most $3(b-a)$ examples to teach $\hstar$ from $\hinit$, while any non-adaptive teacher requires at least $4(b-a)$ examples to teach $\hstar$.
\end{lemma}

\begin{figure*}[!h]
  \centering
  \begin{subfigure}{\textwidth}
    \centering
    \includegraphics[width=.25\textwidth]{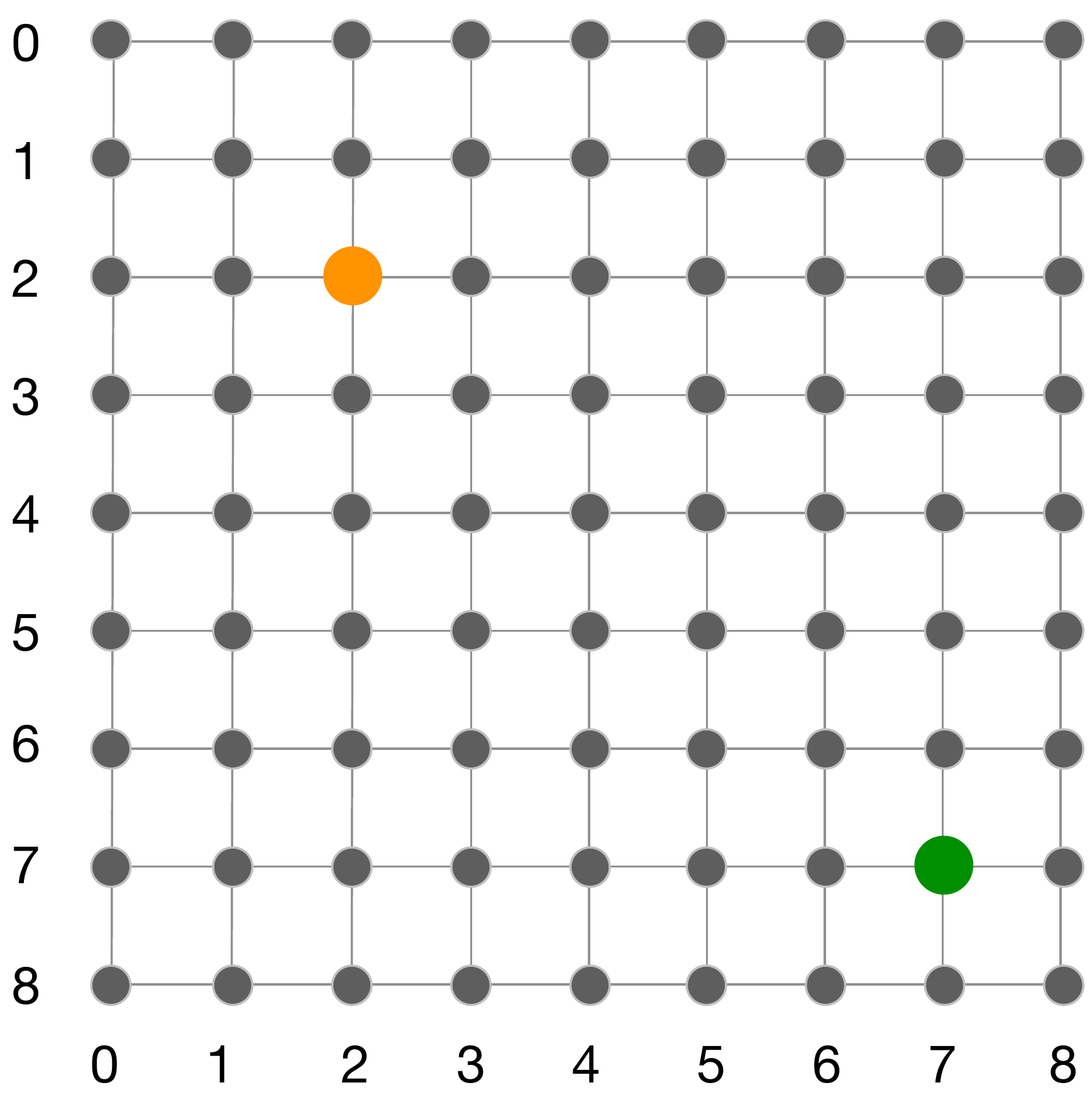}
    \caption{Initial configuration. $\hinit=(2,2)$, $\hstar=(7,7)$, $n=9$.}
    \label{fig:lattice:illustration:adal:init}
  \end{subfigure} \\ \vspace{10pt}
  \begin{subfigure}{\textwidth}
    \centering
    \includegraphics[width=.25\textwidth]{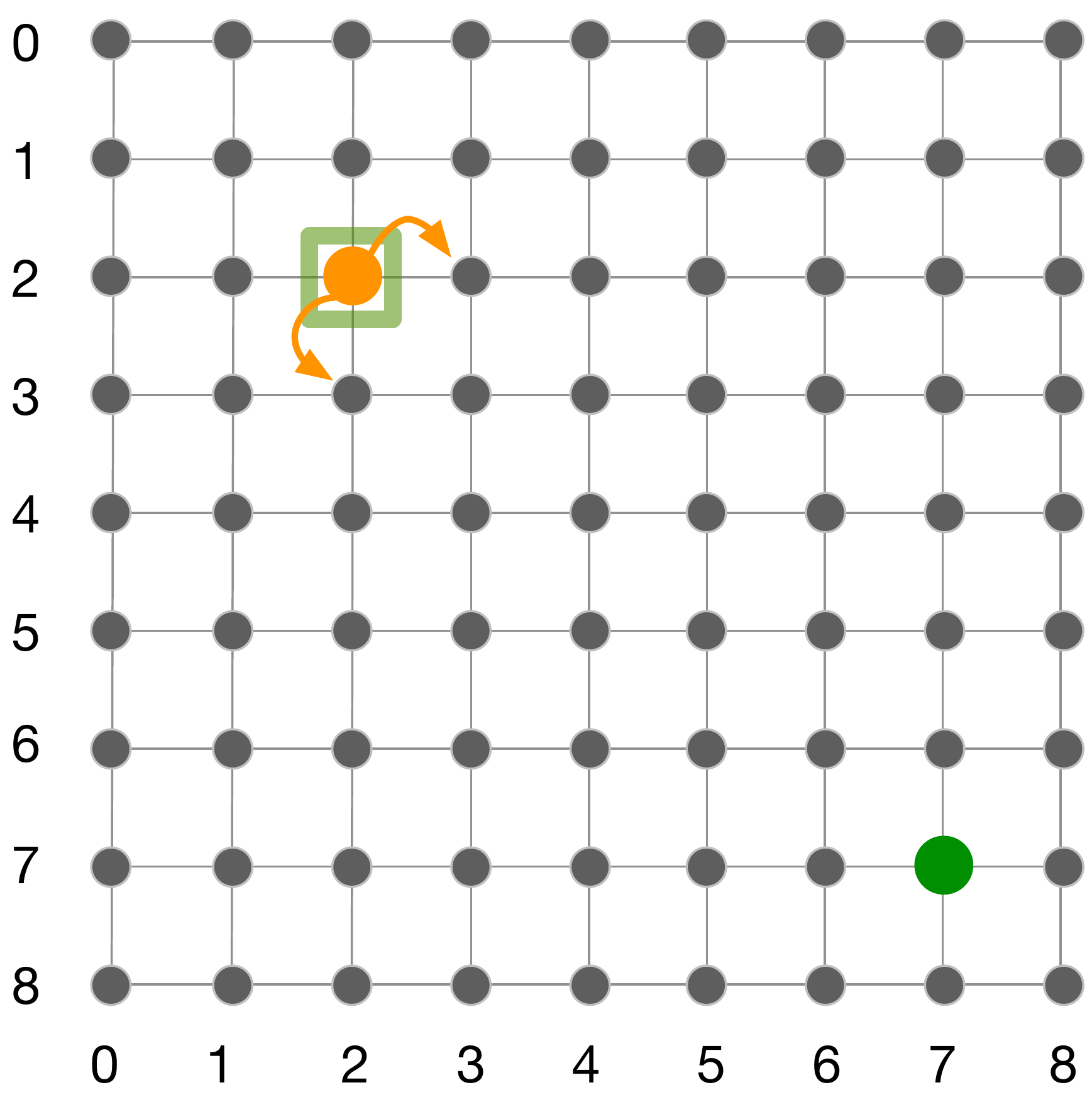}
    \qquad
    \includegraphics[width=.25\textwidth]{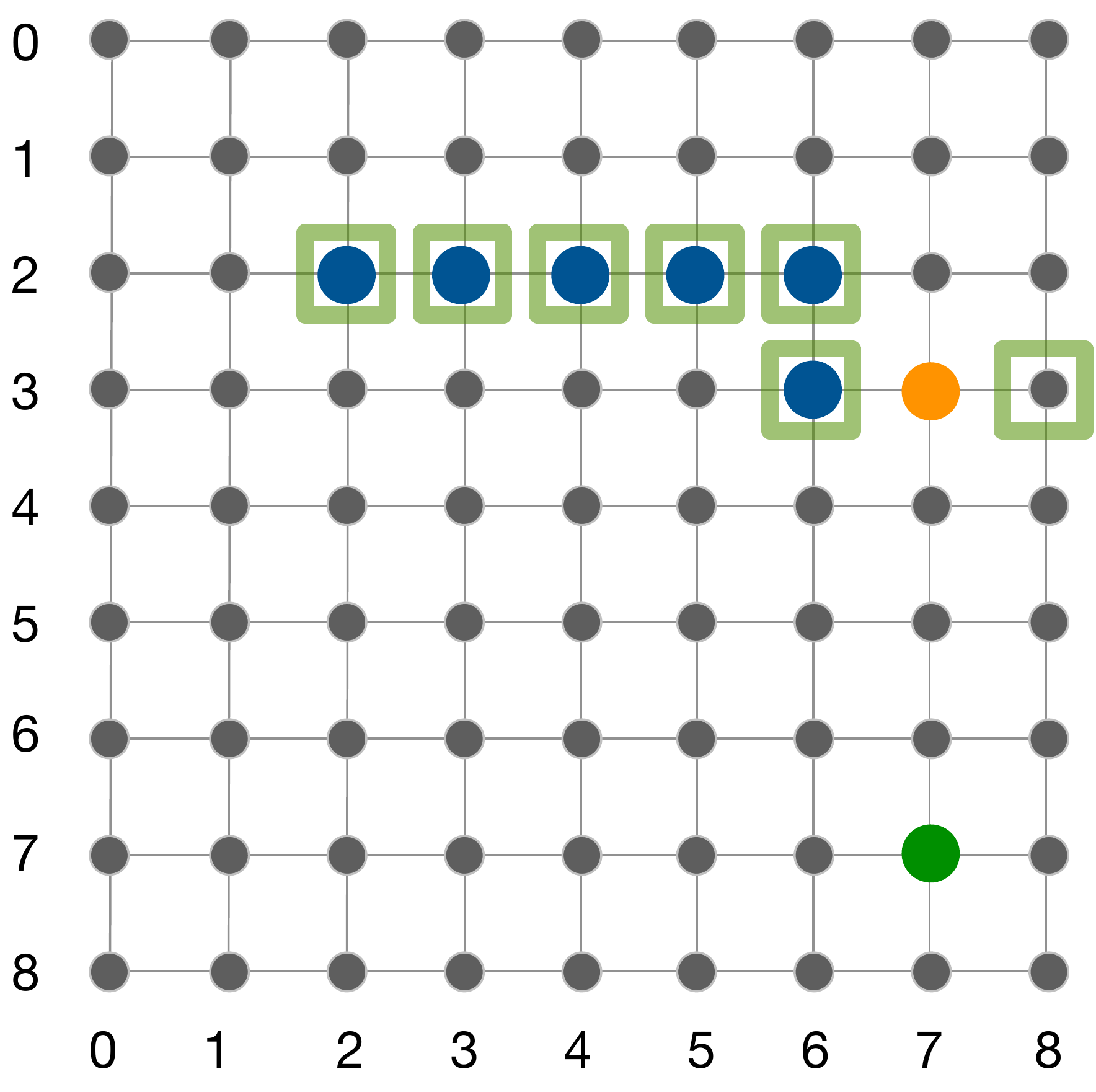}
    \qquad
    \includegraphics[width=.25\textwidth]{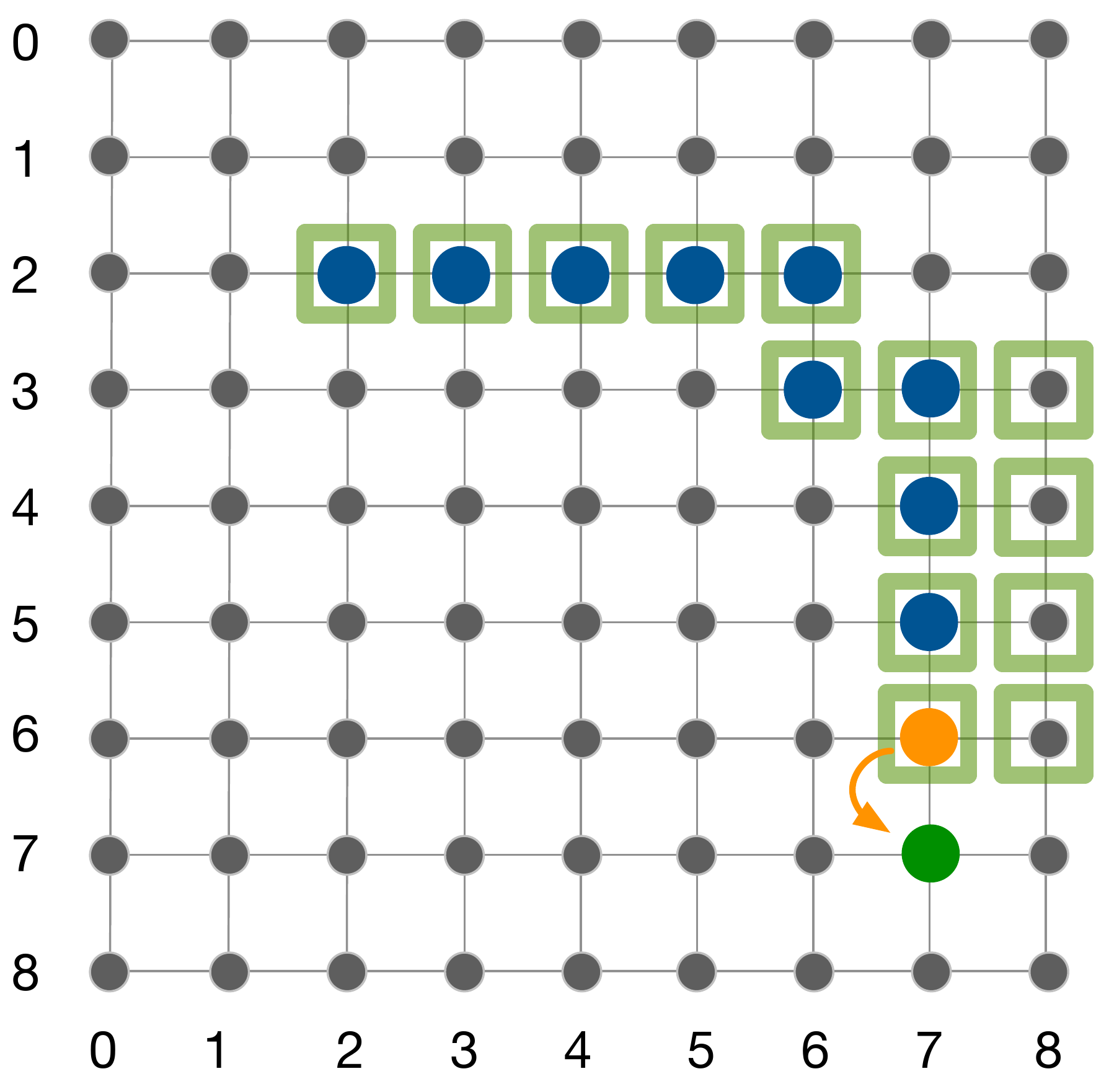}
    \caption{Illustration of the adaptive teaching sequence. In the left figure, when provided with teaching example $(2,2)$ (represented by the solid green square), the learner prefers nodes with larger coordinates. In the middle figure, the learner is at $(7,3)$. The adaptive teacher blocks the lattice node $(8,3)$, so that after providing the teaching example $(7,3)$ the learner will not move further away from the target $\hstar=(7,7)$. In the right figure, we show the set of examples provided by the adaptive teacher. In these figures, the solid blue dots represent the path taken by the learner.}
    \label{fig:lattice:illustration:adal:proof}
  \end{subfigure} \\ \vspace{10pt}
  \begin{subfigure}{\textwidth}
    \centering
    \includegraphics[width=.25\textwidth]{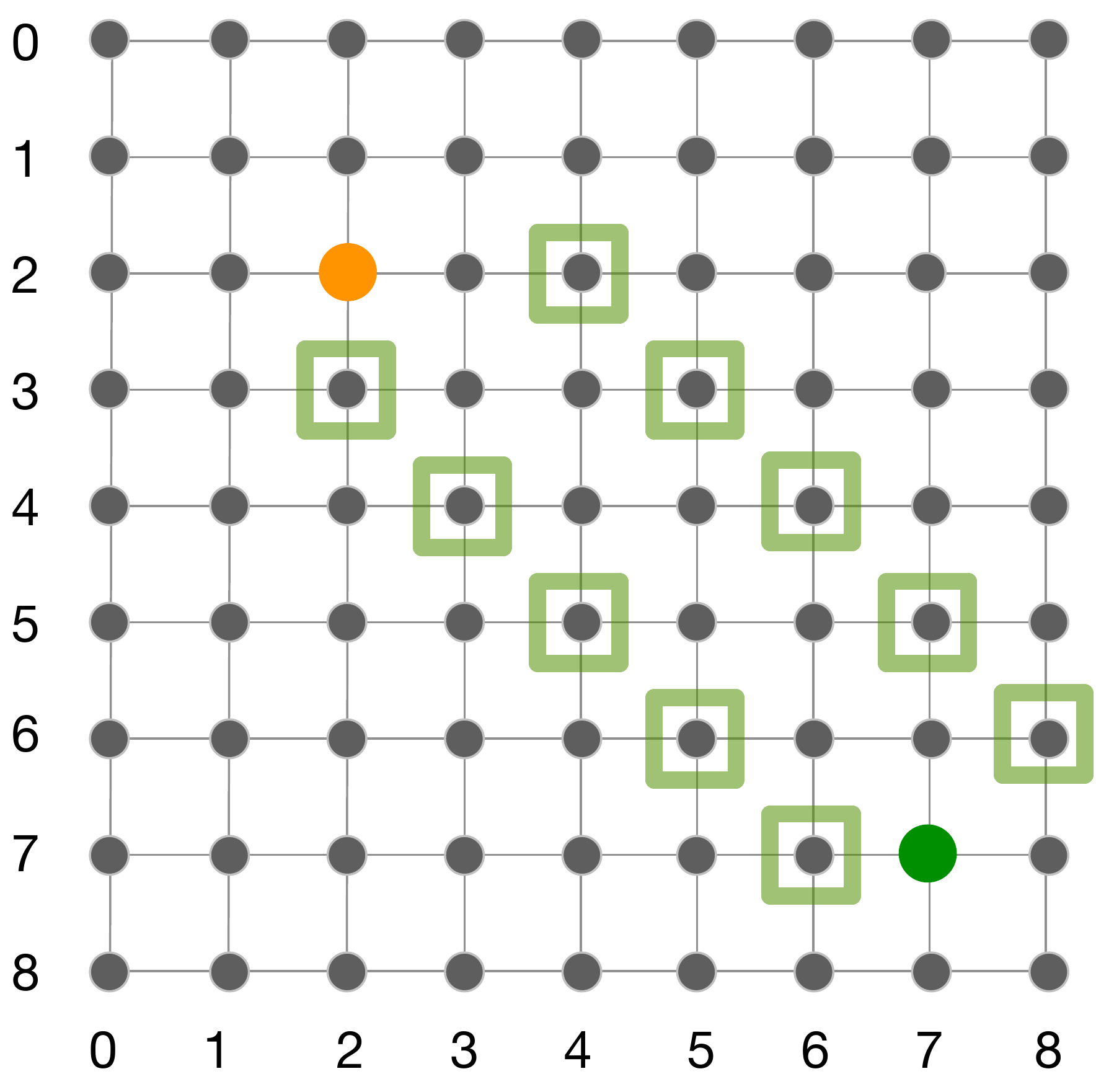}
    \qquad
    \includegraphics[width=.25\textwidth]{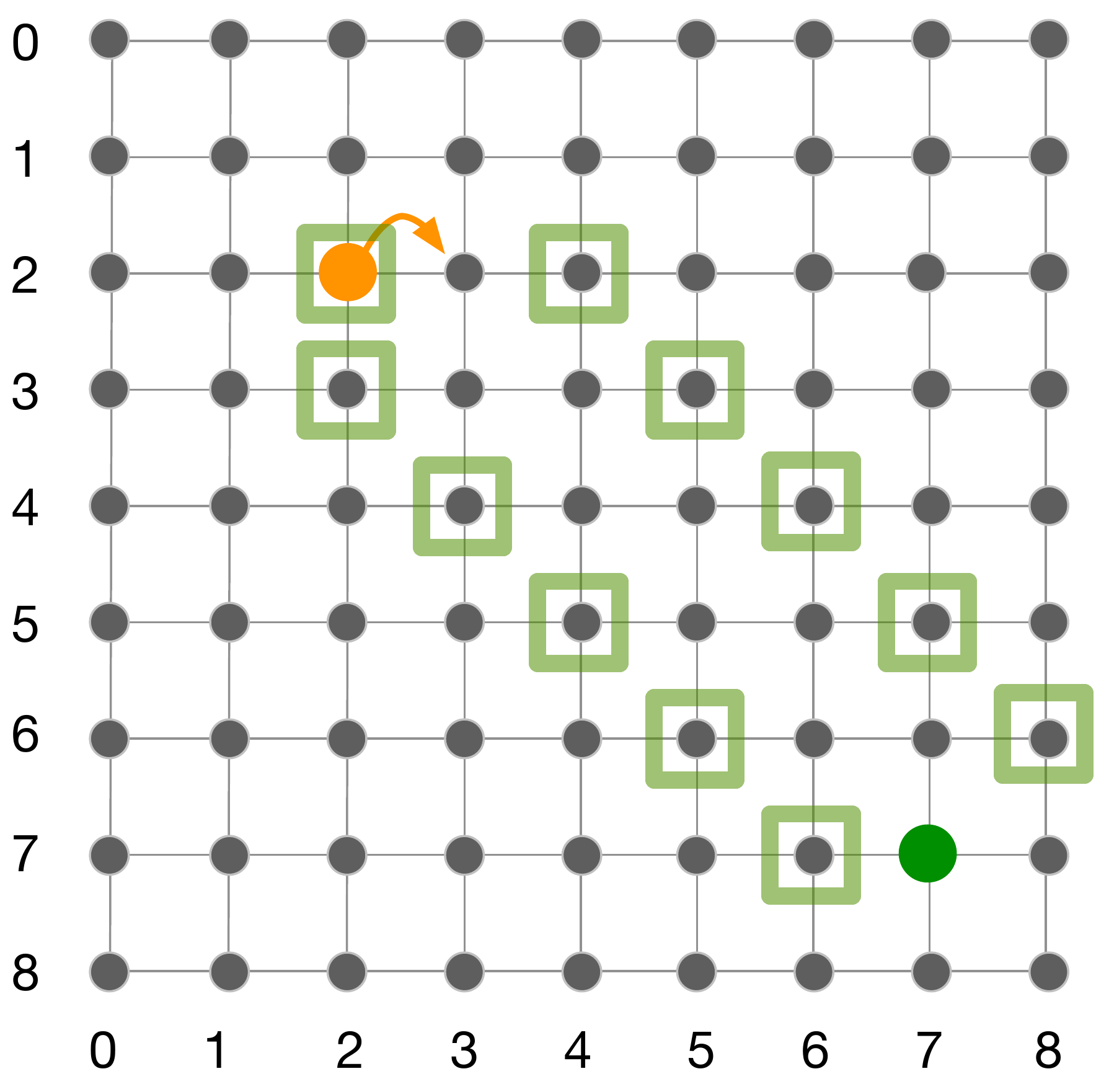}
    \qquad
    \includegraphics[width=.25\textwidth]{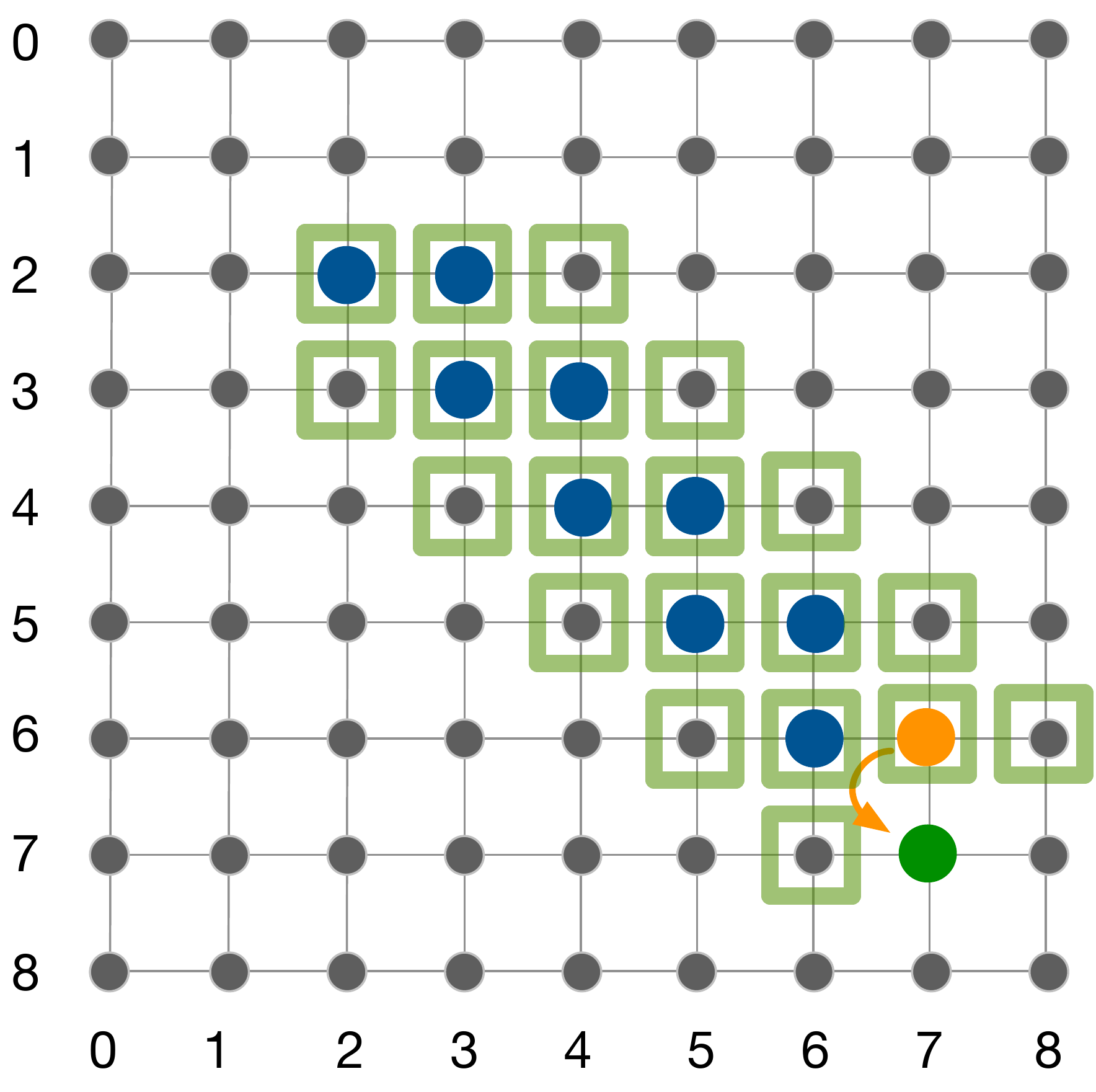}
    \caption{Non-adaptive teaching sequence. The non-adaptive teacher first constructs a deterministic path from $\hinit$ to $\hstar$, and then sequentially steers the learner along the path (represented by the solid blue dots) towards $\hstar$.}
    \label{fig:lattice:illustration:nonl:proof}
  \end{subfigure}
  \caption{Illustration for the \lattice class.}
  \label{fig:illustration:adaptivelattice}
\end{figure*}

\begin{proof}
  We consider the example as illustrated in \figref{fig:lattice:illustration:adal:init}, with $a=2, b=7$ and  $n=9$. That is, $\hinit=(2,2)$ and $\hstar=(7,7)$ on a $9\times 9$ integer lattice.
  
  Under the adaptive setting, we consider the following strategy of constructing a teaching sequence: at each iteration, before removing the learner's current hypothesis, the teacher first tries to ``block'' the next immediate hypotheses (i.e., the most preferred neighboring lattice nodes) which are further away from $\hstar$. Let us refer to such hypotheses as ``bad neighbors''. Since by assumption the learner prefers to move to hypotheses with larger coordinates, there are no bad ``bad neighbors'' until the learner reaches the nodes with coordinates $(7, \cdot)$ or $(\cdot, 7)$. Therefore, the number of teaching examples required by the adaptive teacher is at most
  $$2(b-a) + (b-a) = 3(b-a),$$
  where $2(b-a)$ is the length of the path between $\hinit$ and $\hstar$, and $(b-a)$ is the maximum number of ``bad neighbors'' the adaptive teacher needs to block.

  Under the non-adaptive setting, the teacher has to pre-compute a set of teaching example that ensures the learner taking a deterministic path. Compared to the adaptive strategy which only blocks ``bad'' neighbors, the non-adaptive teacher also has to block additional lattice nodes to form a deterministic path from $\hinit$ to $\hstar$. As illustrated in \figref{fig:lattice:illustration:adal:proof}, the number of teaching examples required by any non-adaptive teacher is at least
  $$2(b-a) + 2(b-a) = 4(b-a),$$
  where the first $2(b-a)$ denotes the length of the path between $\hinit$ and $\hstar$, and the second $2(b-a)$ is the number of lattice nodes the teacher has to block to construct the (deterministic) path. \end{proof}

\begin{proof}[Proof of \thmref{thm:adaptivity}]
  Combining \lemref{lm:tworec} and \lemref{lm:lattice}, we obtain the proof of \thmref{thm:adaptivity}.
\end{proof}

\clearpage
\section{Proof of \thmref{thm:greedy_suff}}\label{app:proof:greedy:suff}
In this section, 
we prove the upper bound of the greedy cost as presented in \thmref{thm:greedy_suff}.

\paragraph{Notations and formal definition of cost} Let us use $\pi$ to denote an adaptive teaching algorithm and $\randomstate$ to denote the internal randomness of the learner. 
We fix $\randomstate$, the preference $\ordering$, the target hypothesis $\hstar$, the learner's initial hypothesis $\hinit$, and the initial version space $\Hypotheses$. 
We denote the learner's hypothesis after running $\pi_t$ (i.e., running $\pi$ for $t$ steps) as $\hypothesis(\pi_t, \randomstate, \ordering, \hstar, \hinit, \Hypotheses)$. To be consistent with the notations used in the main paper, we use $h_t$ as the shorthand notation for $\hypothesis(\pi_t, \randomstate, \ordering, \hstar, \hinit, \Hypotheses)$ whenever it is unambiguous.

Given $\ordering, \hstar, \hinit$ and the version space $\Hypotheses$, the worst-case cost of an algorithm $\pi$ is formally defined as
\begin{align*}
  \cost(\pi \given \ordering, \hstar, \hinit, \Hypotheses) = \max_\randomstate \min_t t, ,~\text{s.t.,}~ \hypothesis(\pi_t, \randomstate, \ordering, \hstar, \hinit, \Hypotheses) = \hstar.
\end{align*}
We use $\optcost(\ordering, \hstar, \hinit, \Hypotheses) := \min_\pi \cost(\pi \given \ordering, \hstar, \hinit, \Hypotheses)$ to denote the cost of an optimal algorithm, and use $\greedycost(\ordering, \hstar, \hinit, \Hypotheses)$ to denote the cost of the greedy algorithm with respect to $\futurecostapprox$ (Eq.~\eqref{eq:objrank}).


\paragraph{Preferred version space} At time step $t$, we define the \emph{preferred version space}
\begin{align}
  \prefset(\ordering, \hstar, \hypothesis_t, \Hypotheses_t) := \{\hypothesis' \in \Hypotheses_t: \orderingof{\hypothesis'}{\hypothesis_t} \leq \orderingof{\hstar}{\hypothesis_t}\}\label{eq:prefset}
\end{align}
to be the set of hypotheses that are preferred over $\hstar$ (according to $\ordering$) from $\hypothesis_t$.

\paragraph{Greedy vs. optimal} We now proceed to the proof of \thmref{thm:greedy_suff}. First, we provide a lower bound on the cost of the optimal algorithm.

\begin{lemma}\label{lm:opt_lowerbound}
  Let $\uordering$ be the uniform preference function. Assume that $\ordering$ satisfies Condition~\ref{thm:suffcond:1} of \thmref{thm:greedy_suff}. Then, the following inequality holds:
  $$\optcost(\ordering, \hstar, \hinit, \Hypotheses) \geq \optcost(\uordering, \hstar, \hinit, \prefset(\ordering, \hstar, \hinit, \Hypotheses)).$$
\end{lemma}

\begin{proof}
  Fix $\randomstate$. Denote the full teaching sequence of the optimal algorithm under preference $\ordering$ as $\examples_m = \{\example_0, \example_1, \dots, \example_{m-1}\}$, and the trace of the learner's hypotheses as $\{\hypothesis_0, \hypothesis_1, \dots, \hypothesis_{m-1}, \hstar\}$. By definition, we have $\Hypotheses(\examples_t) = \Hypotheses_{t}$. 

  Upon receiving teaching example $\example_t$, the set of hypotheses eliminated from the preferred version space at time step $t$ is
  \begin{align}
    & \prefset(\ordering, \hstar, \hypothesis_t, \Hypotheses_{t}) \setminus \prefset(\ordering, \hstar, \hypothesis_t, \Hypotheses(\examples_{t+1})) \notag \\
    \stackrel{}{=}~& \prefset(\ordering, \hstar, \hypothesis_t, \Hypotheses_{t}) \setminus \{\hypothesis' \in \Hypotheses_{t+1}: \orderingof{\hypothesis'}{\hypothesis_t} \leq \orderingof{\hstar}{\hypothesis_t}\} \notag \\
    \stackrel{(a)}{=}~& \prefset(\ordering, \hstar, \hypothesis_t, \Hypotheses_{t}) \setminus \{\hypothesis' \in \Hypotheses_{t+1}: \orderingof{\hypothesis_{t+1}}{\hypothesis_t} \leq \orderingof{\hypothesis'}{\hypothesis_t} \leq \orderingof{\hstar}{\hypothesis_t}\} \notag \\
    \stackrel{(b)}{\supseteq}~& \prefset(\ordering, \hstar, \hypothesis_t, \Hypotheses_{t}) \setminus \{\hypothesis' \in \Hypotheses_{t+1}: \orderingof{\hypothesis'}{\hypothesis_{t+1}} \leq \orderingof{\hstar}{\hypothesis_{t+1}}\} \notag \\
    \stackrel{}{=}~& \prefset(\ordering, \hstar, \hypothesis_t, \Hypotheses_{t}) \setminus \prefset(\ordering, \hstar, \hypothesis_{t+1}, \Hypotheses_{t+1}) \label{eq:prefset:ht}
  \end{align}
  Here, step (a) is due to the fact that $\Hypotheses_{t+1} \cap \prefset(\ordering, \hypothesis_{t+1}, \hypothesis_t, \Hypotheses_{t}) = h_{t+1}$, and step (b) follows from Condition~\ref{thm:suffcond:1} of \thmref{thm:greedy_suff}.

  Further observe that
  \begin{align}
    \prefset(\ordering, \hstar, \hypothesis_0, \Hypotheses)
    &\subseteq \prefset(\ordering, \hstar, \hypothesis_0, \Hypotheses) \setminus \prefset(\ordering, \hstar, \hypothesis_1, \Hypotheses_{1}) \cup \prefset(\ordering, \hstar, \hypothesis_1, \Hypotheses_{1}) \notag\\
    &\subseteq \bigcup_{t=0}^{m-1} \left( \prefset(\ordering, \hstar, \hypothesis_t, \Hypotheses_{t}) \setminus \prefset(\ordering, \hstar, \hypothesis_{t+1}, \Hypotheses_{t+1}) \right) \cup \prefset(\ordering, \hstar, \hstar, \Hypotheses_{m}) \notag\\
    &= \bigcup_{t=0}^{m-1} \left( \prefset(\ordering, \hstar, \hypothesis_t, \Hypotheses_{t}) \setminus \prefset(\ordering, \hstar, \hypothesis_{t+1}, \Hypotheses_{t+1}) \right) \cup \{\hstar\} \label{eq:prefset:h0}
  \end{align}
  Combining Eq.~\eqref{eq:prefset:ht} and \eqref{eq:prefset:h0}, we obtain
  \begin{align*}
    \prefset(\ordering, \hstar, \hypothesis_0, \Hypotheses) \setminus \{\hstar\}
    &\subseteq \bigcup_{t=0}^{m-1} \left( \prefset(\ordering, \hstar, \hypothesis_t, \Hypotheses_{t}) \setminus \prefset(\ordering, \hstar, \hypothesis_{t+1}, \Hypotheses_{t+1}) \right) \\
    & \subseteq \bigcup_{t=0}^{m-1} \left( \prefset(\ordering, \hstar, \hypothesis_t, \Hypotheses_{t}) \setminus \prefset(\ordering, \hstar, \hypothesis_t, \Hypotheses_{t+1}) \right)
  \end{align*}
  That is, providing teaching examples $\{\example_1, \example_2, \dots, \example_{m-1}\}$ is guaranteed to eliminate $\prefset(\ordering, \hstar, \hinit, \Hypotheses) \setminus \{\hstar\}$. By definition, $\optcost(\uordering, \hstar, \hinit, \prefset(\ordering, \hstar, \hinit, \Hypotheses))$ is the minimal number of examples required to eliminate $\prefset(\ordering, \hstar, \hinit, \Hypotheses) \setminus \{\hstar\}$. Since the optimal cost is defined as the worst-case cost for all $\randomstate$, it follows that $\optcost(\uordering, \hstar, \hinit, \prefset(\ordering, \hstar, \hinit, \Hypotheses)) \leq m \leq \optcost(\ordering, \hstar, \hinit, \Hypotheses)$.
\end{proof}

In the following, we will focus on the analysis of the greedy algorithm with preference $\ordering$.

\begin{lemma}\label{lm:greedy_gain}
  Assume that the preference function $\ordering$ and the structure of tests satisfy Condition 1 and 2 from \thmref{thm:greedy_suff}. Suppose we have run the myopic algorithm (w.r.t. Eq.~\ref{eq:objrank}) for $m$ time steps. Let $\example_t$ be the current teaching example, $\examples_t = \{\example_0, \dots, \example_{t-1}\}$ be the set of examples chosen by the greedy teacher up to $t$, and $\hypothesis_t$ be the learner's current hypothesis. For any given example $\example$, let $\hypothesis_z$ be the learner's next hypothesis assuming the teacher provides $\example$. Then,
  \begin{align}
    & |\prefset(\ordering, \hstar, \hypothesis_t,  \Hypotheses(\examples_t))| - |\prefset(\ordering, \hstar, \hypothesis_{t+1},  \Hypotheses(\examples_t \cup \{\example_t\}))| \notag \\
    \geq ~&\frac12 \max_\example \left( |\prefset(\ordering, \hstar, \hinit,  \Hypotheses(\examples_t))|  - |\prefset(\ordering, \hstar, \hinit,  \Hypotheses(\examples_t \cup \{\example\}))| \right)
  \end{align}
\end{lemma}

\begin{proof}
  The myopic algorithm w.r.t. Eq.~\ref{eq:objrank} picks the example which leads to the smallest preferred version space. That is,
  \begin{align*}
    \example_t = \argmin_\example |\prefset(\ordering, \hstar, \hypothesis_z, \Hypotheses(\examples_t \cup \{\example\}))|
  \end{align*}
  Here, $\hypothesis_z$ denotes the hypothesis that the learner takes if she observes example $z$ at time step $t$. If $\example$ is inconsistent
  with $\hypothesis_{t}$, then by Condition \ref{thm:suffcond:2} of
  \thmref{thm:greedy_suff}, there exists an example $\example'$ which is
  consistent with $\hypothesis_{t}$ and only differs from $\example$ at
  $\hypothesis_{t}$, i.e.,
  $\Hypotheses(\{\example'\})\setminus \Hypotheses(\{\example\}) =
  \{\hypothesis_{t}\}$.
  By Condition~\ref{thm:suffcond:1} of \thmref{thm:greedy_suff} (and step (b) of Eq.~\eqref{eq:prefset:ht}), we know $\prefset(\ordering, \hstar, \hypothesis_{t}, \Hypotheses(\examples_t \cup \{\example\})) \subseteq \prefset(\ordering, \hstar, \hypothesis_\example, \Hypotheses(\examples_t \cup \{\example\}))$.
  Therefore
  \begin{align}
    |\prefset(\ordering, \hstar, \hypothesis_{t}, \Hypotheses(\examples_t \cup \{\example'\}))| - 1 =
    |\prefset(\ordering, \hstar, \hypothesis_{t}, \Hypotheses(\examples_t \cup \{\example\}))| \leq |
    \prefset(\ordering, \hstar, \hypothesis_\example, \Hypotheses(\examples_t \cup \{\example\}))|\label{eq:hzvshzp}
  \end{align}
  which gives us $|\prefset(\ordering, \hstar, \hypothesis_{t}, \Hypotheses(\examples_t \cup \{\example'\}))| \leq
  \prefset(\ordering, \hstar, \hypothesis_\example, \Hypotheses(\examples_t \cup \{\example\}))| + 1$.

  We consider the following three cases.
  \begin{enumerate}[C-1]
  \item \label{enum:suff:case1} $  |\prefset(\ordering, \hstar, \hypothesis_{t}, \Hypotheses(\examples_t \cup \{\example'\}))| =
    |\prefset(\ordering, \hstar, \hypothesis_\example, \Hypotheses(\examples_t \cup \{\example\}))| + 1$. Then, by Eq.~\eqref{eq:hzvshzp}, we have
    $$\prefset(\ordering, \hstar, \hypothesis_{t}, \Hypotheses(\examples_t \cup \{\example\})) = \prefset(\ordering, \hstar, \hypothesis_\example, \Hypotheses(\examples_t \cup \{\example\})).$$
    That is, even the example $\example$ can bring the learner to a new hypothesis $\hypothesis_\example$, it does \emph{not} introduce new hypotheses into the preferred version space.

  \item \label{enum:suff:case2} $  |\prefset(\ordering, \hstar, \hypothesis_{t}, \Hypotheses(\examples_t \cup \{\example'\}))| <
    \prefset(\ordering, \hstar, \hypothesis_\example, \Hypotheses(\examples_t \cup \{\example\}))| + 1$. In this case, the greedy teacher will not pick $\example$, because the gain of example $\example'$ is no less than the gain of $\example$ in terms of the greedy heuristic. In the special case where $|\prefset(\ordering, \hstar, \hypothesis_{t}, \Hypotheses(\examples_t \cup \{\example'\}))| =
    \prefset(\ordering, \hstar, \hypothesis_\example, \Hypotheses(\examples_t \cup \{\example\}))|$, according to our tie-breaking rule in the myopic algorithm, the teacher does not pick $\example$, because it makes the learner move away from its current hypothesis and hence is less preferred.
  \end{enumerate}

  For completeness, we also consider the case when the example $\example$ is consistent with $\hypothesis_t$:
  \begin{enumerate}[C-3]
  \item \label{enum:suff:case3} If the teacher picks a consistent example $\example$, then the learner does not move away from her current hypothesis $\hypothesis_t$. As a result, the preference ordering among set $\prefset(\ordering, \hstar, \hypothesis_{t}, \Hypotheses(\examples_t \cup \{\example'\}))$ remains the same.
  \end{enumerate}

  With the above three cases set up, we now reason about the gain of the myopic algorithm. An important observation is that, the greedy teaching examples never add any hypotheses into the preferred version space. Therefore, at time step $t$, for any example $\example$, we have
  \begin{align}
    \label{eq:prefvschange}
    \prefset(\ordering, \hstar, \hypothesis_{t}, \Hypotheses(\examples_t \cup \{\example\}))
    &= \prefset(\ordering, \hstar, \hypothesis_{t-1}, \Hypotheses(\examples_t \cup \{\example\})) \notag \\
    &= \dots \notag \\
    &= \prefset(\ordering, \hstar, \hypothesis_{0}, \Hypotheses(\examples_t \cup \{\example\}))
  \end{align}

  Next, we look into the gain for each of the three cases above.
  \begin{enumerate}[C-1]
  \item Adding $\example_t$ changes the learner's hypothesis, i.e., $\hypothesis_{t+1} \neq \hypothesis_t$, but the resulting preferred version space induced by $\hypothesis_{t+1}$ is the same with that of $\hypothesis_{t}$. In this case,
    \begin{align}
      \label{eq:greedy_case1}
      & |\prefset(\ordering, \hstar, \hypothesis_t,  \Hypotheses(\examples_t))| - |\prefset(\ordering, \hstar, \hypothesis_{t+1},  \Hypotheses(\examples_t \cup \{\example_t\}))| \notag \\
      = ~&|\prefset(\ordering, \hstar, \hinit,  \Hypotheses(\examples_t))|  - \min_\example |\prefset(\ordering, \hstar, \hypothesis_{t+1},  \Hypotheses(\examples_t \cup \{\example\}))| \notag \\
      = ~&|\prefset(\ordering, \hstar, \hinit,  \Hypotheses(\examples_t))|  - \min_\example |\prefset(\ordering, \hstar, \hinit,  \Hypotheses(\examples_t \cup \{\example\}))| \notag \\
      = ~&\max_\example \left( |\prefset(\ordering, \hstar, \hinit,  \Hypotheses(\examples_t))|  - |\prefset(\ordering, \hstar, \hinit,  \Hypotheses(\examples_t \cup \{\example\}))| \right)
    \end{align}
  \item In this case, we have
    \begin{align*}
      |\prefset(\ordering, \hstar, \hypothesis_{t+1}, \Hypotheses(\examples_t \cup \{\example_t\}))|
      &= |\prefset(\ordering, \hstar, \hypothesis_z, \Hypotheses(\examples_t \cup \{\example\}))|\\
      &= |\prefset(\ordering, \hstar, \hypothesis_t, \Hypotheses(\examples_t \cup \{\example'\}))|
    \end{align*}
    and the myopic algorithm picks $\example_t=\example'$ according to the tie-breaking rule. The learner does not move away from her current hypothesis: $\hypothesis_{t+1} = \hypothesis_t$. However, since $\Hypotheses(\{\example'\})\setminus \Hypotheses(\{\example\}) = \{\hypothesis_{t}\}$, we get
    \begin{align}
      |\prefset(\ordering, \hstar, \hypothesis_{t+1},  \Hypotheses(\examples_t \cup \{\example_t\}))|
      &= |\prefset(\ordering, \hstar, \hypothesis_t, \Hypotheses(\examples_t \cup \{\example'\}))| \notag \\
      &= |\prefset(\ordering, \hstar, \hypothesis_t, \Hypotheses(\examples_t \cup \{\example\}))| + 1\notag \\
      &\stackrel{(a)}{=} \min_{\example''} |\prefset(\ordering, \hstar, \hypothesis_t,  \Hypotheses(\examples_t \cup \{\example''\}))| + 1\notag \\
      &\stackrel{\eqref{eq:prefvschange}}{=} \min_{\example''} |\prefset(\ordering, \hstar, \hinit,  \Hypotheses(\examples_t \cup \{\example''\}))| + 1 \notag 
    \end{align}
    where step (a) is due to the greedy choice of the myopic algorithm.
    Further note that before reaching $\hstar$, the gain of a greedy teaching example is positive. Therefore,
    \begin{align}
      \label{eq:greedy_case2}
      & |\prefset(\ordering, \hstar, \hypothesis_t,  \Hypotheses(\examples_t))| - |\prefset(\ordering, \hstar, \hypothesis_{t+1},  \Hypotheses(\examples_t \cup \{\example_t\}))| \notag \\
      \geq ~&\frac12 \left( 1 + |\prefset(\ordering, \hstar, \hypothesis_t,  \Hypotheses(\examples_t))| - |\prefset(\ordering, \hstar, \hypothesis_{t+1},  \Hypotheses(\examples_t \cup \{\example_t\}))| \right) \notag \\
      = ~&\frac12 \max_\example \left( |\prefset(\ordering, \hstar, \hinit,  \Hypotheses(\examples_t))|  - |\prefset(\ordering, \hstar, \hinit,  \Hypotheses(\examples_t \cup \{\example\}))| \right)
    \end{align}
  \item In this case, $\example_t$ is consistent with $h^t$, the greedy gain amounts to the maximal number of hypotheses removed from the preferred version space. Thus we have
    \begin{align}
      \label{eq:greedy_case3}
      & |\prefset(\ordering, \hstar, \hypothesis_t,  \Hypotheses(\examples_t))| - |\prefset(\ordering, \hstar, \hypothesis_{t+1},  \Hypotheses(\examples_t \cup \{\example_t\}))| \notag \\      = ~&\max_\example \left( |\prefset(\ordering, \hstar, \hinit,  \Hypotheses(\examples_t))|  - |\prefset(\ordering, \hstar, \hinit,  \Hypotheses(\examples_t \cup \{\example\}))| \right)
    \end{align}
  \end{enumerate}
  Combining Eq.~\eqref{eq:greedy_case1}, \eqref{eq:greedy_case2}, \eqref{eq:greedy_case3} finishes the proof.
\end{proof}

\paragraph{Proof of \thmref{thm:greedy_suff}}

We are now ready to provide the proof for \thmref{thm:greedy_suff}.

\begin{proof}[Proof of \thmref{thm:greedy_suff}]

  Based on the discussions in \lemref{lm:greedy_gain}, we know that the teaching sequence provided by the myopic algorithm that greedily minimizes \eqref{eq:objrank} never adds new hypotheses into the initial preferred version space $\prefset(\ordering, \hstar, \hinit,  \Hypotheses(\examples_t))$, and neither does it move \emph{consistent} hypotheses out of $\prefset(\ordering, \hstar, \hinit,  \Hypotheses(\examples_t))$. The teaching objective thus reduces to a set cover objective, and the teaching finishes once all hypotheses, except $\hstar$, in the initial preferred version space are covered.

  In \lemref{lm:greedy_gain}, we show that at each time step, the gain of the myopic algorithm is at least $\frac12$ the gain of the greedy set cover algorithm. Therefore, the myopic algorithm is a 2-approximate greedy set cover algorithm \cite{wolsey1982analysis}. The logarithmic approximation result then follows from  \cite{wolsey1982analysis, golovin2011adaptive}:
  \begin{align}
    \greedycost(\ordering, \hstar, \hypothesis, \prefset) \leq 2 \left( \log{\futurecostapprox(\hinit, \Hypotheses)} + 1 \right)\optcost(\uordering, \hstar, \hinit, \prefset(\ordering, \hstar, \hinit, \Hypotheses)).\label{eq:approx_greedy_setcover}
  \end{align}

  Combining Eq.~\ref{eq:approx_greedy_setcover} with \lemref{lm:opt_lowerbound} completes the proof.
\end{proof}

\clearpage
\section{\tworec: Algorithmic Details} \label{app:tworec}
In this section, we provide the detailed specification of the \tworec hypothesis class introduced in \secref{sec:adaptivity}, and present the adaptive algorithm \adar and non-adaptive algorithm \nonadar. 

\subsection{Preference Structure}\label{app:tworec:structure}
As described earlier, \tworec contains two (non-overlapping) subclasses $\Hypotheses^1$ and $\Hypotheses^2$ with different complexity. 
Let $\Hypotheses_t$ be the learner's version space at time step $t$, and $\hypothesis_t$ be the learner's current hypothesis. 

We consider two special subsets of hypotheses of $\Hypotheses^2$.

\paragraph{$\Hypotheses^1$ union singleton:} In the first special subset, each hypothesis can be considered as a $\Hypotheses^1$ hypothesis union a singleton grid cell: Let $r_1: \Hypotheses^2 \rightarrow \Hypotheses^1$ (resp. $r_2$) denote the function that maps a hypothesis $h\in \Hypotheses^2$ to the first (resp. second) rectangle it contains. Then the set of all such ($\Hypotheses^1$ union singleton) hypotheses is
\begin{align}
  \label{eq:singletonM2}
  S_1 = \{h \in \Hypotheses^2: |r_1(h)| = 1 \vee |r_2(h)| = 1\}.
\end{align}



\paragraph{$\Hypotheses^1$ splits:} In the second special subset, each hypothesis can be considered as a \emph{split} from a $\Hypotheses^1$ hypothesis:
Given $h\in \Hypotheses^2$, we call $h$ a $\Hypotheses^1$ \emph{split}, if and only if there exists no other $\Hypotheses^2$ hypothesis in the minimal rectangle that encloses $h$. We denote the set of all $\Hypotheses^1$ split as $S_2$:
\begin{align}
  \label{eq:splitM2}
  S_2 = \{h \in \Hypotheses^2: \text{$h$ is a $\Hypotheses^1$ split }\}.
\end{align}
We consider the subsets $S_1$ and $S_2$ as shortcuts between as $\Hypotheses^1$ and $\Hypotheses^2$. In the following, we will describe our preference model of the learners, based on the subclasses previously defined.


For any pair of hypotheses $h, h'$ from the \emph{same subclass}, define $\edgedist(\hypothesis, \hypothesis')$ to be the minimal number of edge movements required to move from $\hypothesis$ to $\hypothesis'$. For example, $\max_{h,h'\in \Hypotheses^1} \edgedist(h,h') \leq 4$ and $\max_{h,h'\in \Hypotheses^2} \edgedist(h,h') \leq 8$. 

\begin{figure}[!h]
  \centering
  \includegraphics[width=.45\textwidth]{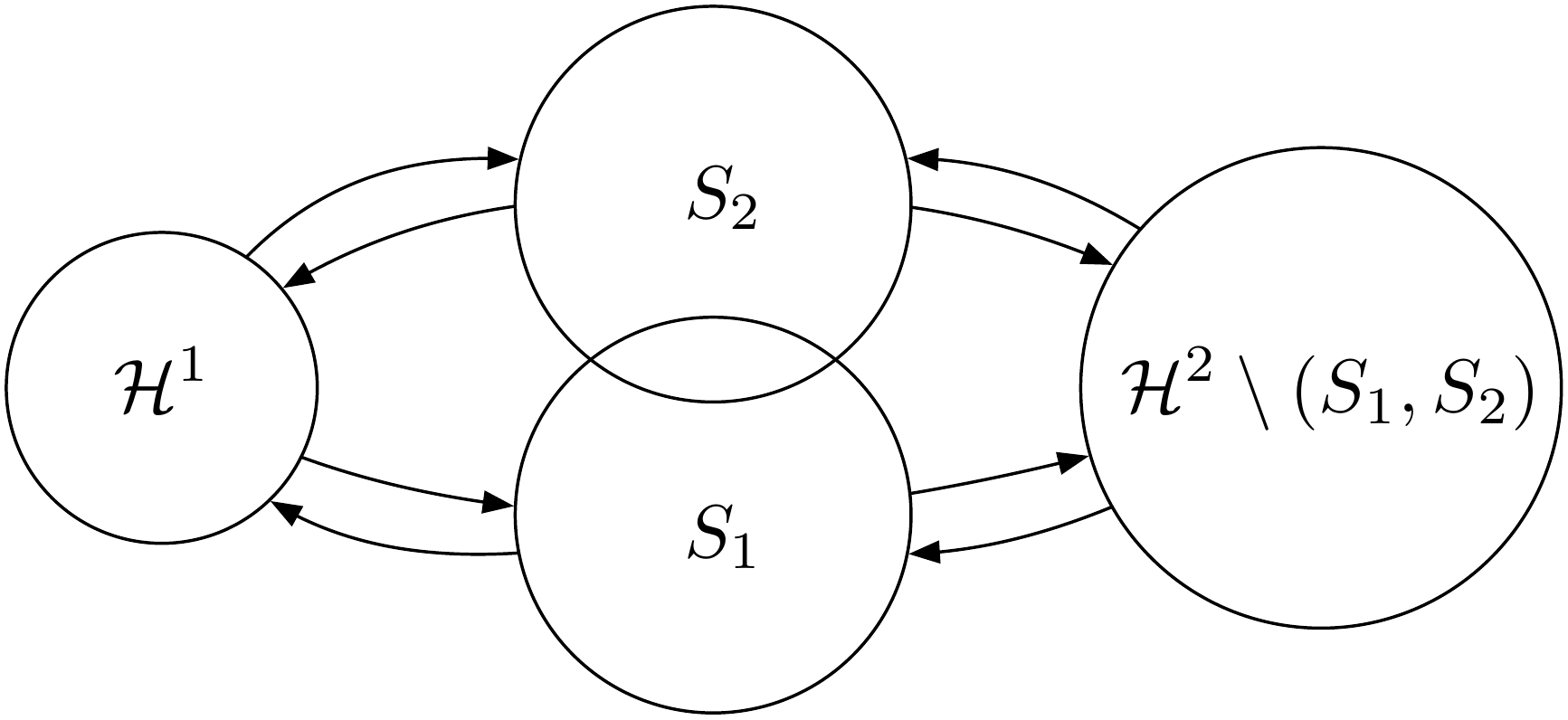}
  \caption{Transitions between subclasses of \tworec upon receiving \emph{one} teaching example}
  \label{fig:tworec_transition}
\end{figure}

Our general assumption about the learner's preference structure $\ordering$ for \tworec is that learners prefer to make small moves among hypotheses. In \figref{fig:tworec_transition}, we depict the valid transitions among \tworec hypotheses upon receiving an example. That includes the following cases:
\begin{enumerate}[C-1]
\item $h_t \in \Hypotheses^1$. In this case, the learner prefers hypotheses in $\Hypotheses^1$, then hypotheses in $S_1$, and lastly hypotheses in $\Hypotheses^2\setminus S_1$. More specifically,
  \begin{enumerate}
  \item Within $\Hypotheses^1$, the learner prefers a hypothesis with smaller distance from $h_t$;
  \item Within $S_1$, the learner has uniform preferences

    ---if learner makes a jump to $S_1$, it corresponds to ``draw new rectangle'' operation.
  \item Within $\Hypotheses^2\setminus S_1$, the learner also has uniform preferences from $h_t$. 

    Note that case (c) is not needed for designing an adaptive teacher, because the learner always needs to move to a hypothesis in $S_1\cup S_2$ from  $h_t$, prior to moving to $\Hypotheses^2\setminus (S_1 \cup S_2)$. 
  \end{enumerate}
\item $h_t \in S_1 \cup S_2$. In this case,
  \begin{enumerate}
  \item The current hypothesis is always the most preferred.
  \item Next,
    \begin{enumerate}
    \item if $h_t\in S_1$, the learner prefers hypotheses in $\Hypotheses^1$ which share a rectangle with $h_t$ ---this corresponds to a ``delete a rectangle'' operation.
    \item if $h_t\in S_2$, the learner prefers the hypothesis in $\Hypotheses^1$ which is the minimal rectangle that encloses $h_t$

      ---this corresponds to a ``merge two existing rectangles'' operation.
    \end{enumerate}

  \item Hypotheses further down in the preference list are $\Hypotheses^2$ hypotheses. These hypotheses are ordered by their distances $\edgedist(\cdot, h_t)$ towards $h_t$.
  \item All other hypotheses in $\Hypotheses^1$ are least preferred.
  \item Within $\Hypotheses^1$, the learner prefers the ones that overlap with one of its rectangles.
  \end{enumerate}
\item $h_t \in \Hypotheses^2\setminus (S_1 \cup S_2)$. In this case, the learner prefers hypotheses in $\Hypotheses^2$ over $\Hypotheses^1$. More specifically,
  \begin{enumerate}
  \item Within $\Hypotheses^2$, the learner prefers a hypothesis with smaller distance.
  \item Within $\Hypotheses^1$, the learner has uniform preferences.

    Note that under such preference, the learner always needs to move to a hypothesis in $S$ prior to moving to $\Hypotheses^1$.
  \end{enumerate}
\end{enumerate}
Intuitively, our oracle generates the intermediate hypotheses from the ``short-cut'' hypothesis set.





\subsection{\adar}
\paragraph{The oracle} An essential component in \algref{alg:non-myopic} is the oracle $\oracle(\hypothesis_t, \Hypotheses_t, \hstar)$, which defines the intermediate target hypothesis
at each time step. We define
$\oracle(\hypothesis_t, \Hypotheses_t, \hstar) = \argmin_{\hypothesis \in \oracle(\hypothesis_t, \Hypotheses_t, \hstar)} \futurecostapprox(\hypothesis, \Hypotheses_t, \hstar)$, where $\oracle(\hypothesis_t, \Hypotheses_t, \hstar)$ contains a set of candidate intermediate targets defined as follows.
For \adar, we employ the following adaptive oracle. Consider the four teaching scenarios:
\begin{itemize}
\item [$\Hypotheses^{1\rightarrow 1}$]: $\hstar \in \Hypotheses^1 \wedge h_t \in \Hypotheses^1$, we have $\oracle(\hypothesis_t, \Hypotheses_t, \hstar) = \{\hstar\}$.
\item [$\Hypotheses^{2\rightarrow 2}$]: $\hstar \in \Hypotheses^2 \wedge h_t \in \Hypotheses^2$, we have $\oracle(\hypothesis_t, \Hypotheses_t, \hstar) = \{\hstar\}$.
\item [$\Hypotheses^{1\rightarrow 2}$]: $\hstar \in \Hypotheses^2 \wedge h_t \in \Hypotheses^1$, in this case,
  \begin{enumerate}
  \item If $h_t = r_1(\hstar) \vee h_t = r_2(\hstar)$, $\oracle(\hypothesis_t, \Hypotheses_t, \hstar) = \{h\in S_1: h_t = r_1(h) \vee h_t = r_2(h)\}$.
  \item Otherwise, $\oracle(\hypothesis_t, \Hypotheses_t, \hstar) = \{r_1(\hstar)\}$, where $r_1(\hstar)$ denotes the first rectangle of $\hstar$.
  \end{enumerate}

\item [$\Hypotheses^{2\rightarrow 1}$]: Now let us consider the case $\hstar \in \Hypotheses^1 \wedge h_t \in \Hypotheses^2$.
  \begin{enumerate}
  \item If both rectangles in $h_t$ overlap with $\hstar$:
    \begin{enumerate}
    \item If $h_t$ is a split of $\hstar$, $\oracle(\hypothesis_t, \Hypotheses_t, \hstar) = \{\hstar\}$.
    \item Otherwise, the oracle returns a subset of hypotheses of $S_2$, where (1) each hypothesis $h$ is a split of $\hstar$, and (2) each of the two rectangles contained in $h$ overlaps with exactly one rectangle in $h_t$.

      For discussion simplicity, let us refer to such subset as the set of valid splits of $\hstar$.
    \end{enumerate}
  \item If at least one of the rectangles in $h_t$ is disjoint with $\hstar$,
    \begin{enumerate}
    \item If $h_t \in S_1$, $\oracle(\hypothesis_t, \Hypotheses_t, \hstar) = \{h\in \Hypotheses^1: h_t \text{ is $h$ union singleton}\}$
    \item Otherwise, the oracle returns a subset of hypothesis of $S_1$, where each hypothesis contains (1) a rectangle that fully aligns with one of the rectangles that are disjoint with $\hstar$, and (2) another singleton rectangle which is in the other rectangle of $h_t$.
    \end{enumerate}
  \end{enumerate}
\end{itemize}

\paragraph{The adaptive teacher}

A useful observation is that for teaching $\Hypotheses^{1\rightarrow 1}, \Hypotheses^{2 \rightarrow 2}$, and $\Hypotheses^{1 \rightarrow 2}$, an optimal teacher needs to provide at most 12 teaching examples:
\begin{itemize}
\item To teach $\Hypotheses^{1\rightarrow 1}$, it is sufficient to provide the two positive corner instances in the diagonal positions (say, the lower left corner and the upper right corner), and the two adjacent negative instances for each of the positive corners---this amounts to 6 examples in total.
\item To teach $\Hypotheses^{2\rightarrow 2}$ and $\Hypotheses^{1\rightarrow 2}$, it is sufficient to provide 6 corner examples for each of the rectangles---this amounts to 12 examples in total.
\end{itemize}

There are two implications from the above observation. First, instances lie on the diagonal corners are useful for teaching targets from the same subclass. Second, even though one can design smart algorithms for teaching the above cases (via adaptivity and exhaustive search), we are not likely to benefit from it by much. Therefore, in such cases, \adarteacher goes through the, at most 12, candidate corner point candidates and proposes an example that brings the learner closer to the target hypothesis.

The more challenging, yet inspiring case, is $\Hypotheses^{2\rightarrow 1}$. To bring the learner to the intermediate targets, \adarteacher runs a greedy heuristic derived from Eq.~\eqref{eq:objrank}: it picks an example $z$ so that after the learner makes a move, the number of hypotheses before reaching the closest $\hstar$ is the minimal:
\begin{align}\label{eq:objrank-tworec}
  \example^* \in \argmin_\example \min_i |\{\hypothesis' \in \Hypotheses_t\cap \Hypotheses(\{\example\}): \orderingof{\hypothesis'}{\hypothesis_z} \leq \orderingof{\hstar_i}{\hypothesis_z}\}|.
\end{align}
Here, $h_z$ denotes the learner's next hypothesis if provided with teaching example $z$.

Now, let us go through each case to analyze the performance of the above greedy heuristic.
\begin{itemize}
\item When the learner's hypothesis is at Scenario [$\Hypotheses^{2\rightarrow 1}$]--1--(a)
  or $\Hypotheses^{2\rightarrow 1}$--2--(a), the learner is ready to make a jump to $\Hypotheses^1$.
  A \emph{single} example suffices to achieve this, and hence the greedy heuristic is
  optimal.

\item When the learner's hypothesis is at Scenario [$\Hypotheses^{2\rightarrow 1}$]--1--(b), the goal of teaching is to reach any of the hypothesis in $\oracle(\hypothesis_t, \Hypotheses_t, \hstar)$---the set of valid splits of $\hstar$. Here, we consider two different cases: 
  \begin{enumerate}
  \item Either of the two rectangles of $\hypothesis_t$ is not aligned with $\hstar$ on exactly 3 edges.

    In this case, \adarteacher picks examples from the corner instances of $\hstar$ to bring the edges of two rectangles to $\hstar$. In the worst case, we need all 12 corner instances of $\hstar$ to ensure that.
  \item Both rectangles of $\hypothesis_t$ are aligned with $\hstar$ on exactly 3 edges.

    In this case, the distance from $\hypothesis_t$ to any valid splits of $\hstar$ is 1. \adarteacher follows the greedy heuristic to pick the next example. Note that before reaching the target, the distance between any hypothesis of the learner its closest target remains to be 1. Therefore, the greedy heuristic (Eq.~\ref{eq:objrank-tworec}) leads to a binary search algorithm. Let the maximal length of $\hstar$ be $\ell$, then the teacher needs $\bigO{\log |\ell| + 1}$ examples in the worst case to eliminates \emph{all} the intermediate targets.
  \end{enumerate}

\item When the learner's hypothesis is at Scenario [$\Hypotheses^{2\rightarrow 1}$]--2--(b), the goal of teaching reduces to reaching any of the hypotheses in $\oracle(\hypothesis_t, \Hypotheses_t, \hstar)$ by providing \emph{negative} examples in the rectangle which contains the singleton intermediate targets. To be consistent with the notation in \lemref{lm:tworec}, we refer to such rectangle by $r_2$. 
  It is not difficult to see that no matter what examples the teacher picks, the distances (defined by $\edgedist$) from the resulting hypothesis of the learner to any of the intermediate target hypothesis are equal. Hence the learner's preference over the intermediate target hypotheses is uniform, and the greedy objective (Eq.~\ref{eq:objrank-tworec}) leads to a binary search algorithm. Therefore, the teacher needs $\bigO{\log |r_2| + 1}$ examples in the worst case to eliminates \emph{all} the intermediate targets.

\end{itemize}

\begin{algorithm}[t]
  \caption{\adar-\teachalg: the adaptive teacher for \tworec (subroutine for \adar/\algref{alg:non-myopic})}\label{alg:adar}
  \begin{algorithmic}
    \STATE {\bf input:}
    $\Hypotheses$, $\ordering$, current $\hypothesis_t$, selected examples $\examples^t$, intermediate target 
    $\oracle(\hypothesis_t, \Hypotheses, \hstar) = \hstar_t$
    \IF{$(\hstar\in \Hypotheses^1 \wedge h_t\in \Hypotheses^2)$}
    \IF{$\hstar_t$ is $\Hypotheses^1$ splits from $\hstar$ and $\dist(\hypothesis_t, \hstar_t) > 1$}
    \STATE $T \leftarrow \text{GenerateAllCorners}(\hstar)$
    \STATE $z \leftarrow \text{Sample}(T\setminus \examples^t)$
    \ELSE
    \STATE $z \leftarrow \argmin_\example |\{\hypothesis' \in \Hypotheses(z): \orderingof{\hypothesis'}{\hypothesis_z} \leq \orderingof{\hstar_t}{\hypothesis_z}\}| $
    \ENDIF
    \ELSE
    \STATE T $\leftarrow \text{GenerateDiagonalCorners}(\hstar_t)$
    \STATE $z \leftarrow \text{Sample}(T\setminus \examples^t)$
    \ENDIF
    \STATE {\bf output:} next teaching example $z$
  \end{algorithmic}
\end{algorithm}

The pseudo code of \adarteacher is given in \algref{alg:adar}.

\subsection{\nonadar}
Next, we present the non-adaptive algorithm, \nonadar, which is used in our simulation.

According to our modeling assumption, other than the learner's initial hypothesis, the non-adaptive teacher does not observe how the learner updates her hypotheses. However, this does not affect teaching the easy scenarios, namely $\Hypotheses^{1\rightarrow 1}$, $\Hypotheses^{2\rightarrow 2}$, and $\Hypotheses^{1\rightarrow 2}$. In such cases, the non-adaptive teacher provides all the diagonal corner examples (including both positive and negative) as described earlier, which needs at most 6 examples for $\Hypotheses^1$ target, and 12 for $\Hypotheses^2$.

When teaching $\Hypotheses^{2\rightarrow 1}$, in particular, for the case of [$\Hypotheses^{2\rightarrow 1}$]--1--(b) and [$\Hypotheses^{2\rightarrow 1}$]--2--(b), it is not possible for the non-adaptive teacher to perform a binary search. The reason is that the learner's behavior is highly non-deterministic at every iteration, and the uncertainty of the learner's hypotheses diffuses at an exponential rate. 
The best thing a non-adaptive teacher can do (in the worst case) is a linear scan over the candidate teaching examples, in which case it requires \bigOmega{|\ell|} examples for [$\Hypotheses^{2\rightarrow 1}$]--1--(b), and \bigOmega{|r_2|} examples for [$\Hypotheses^{2\rightarrow 1}$]--2--(b).

The pseudocode of \nonadar is provided in \algref{alg:nonadar}.

\begin{algorithm}[t]
  \caption{\nonadar: the non-adaptive teaching algorithm for \tworec}\label{alg:nonadar}
  \begin{algorithmic}
    \STATE {\bf input:}
    $\Hypotheses$, $\ordering$, initial $\hypothesis_0$, selected examples $\examples^t$, oracle $\oracle(\hinit, \Hypotheses, \hstar) = \{\hstar_1, \dots, \hstar_k\}$
    \IF{$(\hstar\in \Hypotheses^1 \wedge \hinit\in \Hypotheses^2)$}
    \IF{$\{\hstar_1, \dots, \hstar_k\}$ are $\Hypotheses^1$ splits from $\hstar$}
    \STATE $T_1 \leftarrow \text{GenerateAllCorners}(\hstar)$
    \STATE $T_2 \leftarrow \text{GenerateAllEdgeInstances}(\hstar)$
    \\\COMMENT{\commentfmt{$\drsh$ provide all the (positive) teaching examples on the edges/borders of $\hstar$ to make the learner ``merge'' the two rectangles in $\hinit$.}}
    \STATE $Z \leftarrow (T_1, T_2)$
    \ELSE
    \STATE $T_1 \leftarrow \text{GenerateAllConsistentInstances}(r(\hinit) \text{ which contains the singleton rectangles})$
    \\\COMMENT{\commentfmt{$\drsh$ if one of the rectangles of $\hinit$ is disjoint with $\hstar$, provide all the examples inside this rectangle to make the learner ``delete'' it.}}
    \STATE $T_2 \leftarrow \text{GenerateAllCorners}(\hstar)$
    \STATE $Z \leftarrow (T_1, T_2)$
    \ENDIF
    \ELSE
    \STATE $Z \leftarrow \text{GenerateDiagonalCorners}(\hstar_1)$
    \ENDIF
    \STATE {\bf output:} Sequence of teaching examples $Z$
  \end{algorithmic}
\end{algorithm}


}
{
}

\end{document}